\documentclass[matprg]{svjour3}
\usepackage{graphics,graphicx,epsf,subfigure,epstopdf}
\usepackage{amsmath,amsxtra,amsfonts,amscd,amssymb,bm}
\usepackage{algorithm}
\usepackage{algpseudocode}
\usepackage[top=2.5cm,bottom=2.5cm,right=2.5cm,left=2.5cm]{geometry}
\usepackage[mathscr]{eucal}
\usepackage{color}
\numberwithin{equation}{section}

\usepackage{hyperref}

\DeclareMathOperator*{\argmin}{arg\,min}

\def\eqnok#1{(\ref{#1})}

\newcommand{\tsum}{\textstyle\sum}
\newcommand{\bbe}{\mathbb{E}}
\def\prob{\mathop{\rm Prob}}

\newcommand{\beq}{\begin{equation}}
\newcommand{\eeq}{\end{equation}}
\newcommand{\beqa}{\begin{eqnarray}}
\newcommand{\eeqa}{\end{eqnarray}}
\newcommand{\beqas}{\begin{eqnarray*}}
\newcommand{\eeqas}{\end{eqnarray*}}

\newtheorem{assumption}{Assumption}

\newcommand{\bbr}{\Bbb{R}}
\newcommand{\nn}{\nonumber}

\def\cS{{\cal S}}
\def\cA{{\cal A}}
\def\cP{{\cal P}}
\def\cQ{{\cal Q}}
\def\cD{{\cal D}}
\def\KL{{\rm KL}}
\def\Pr{{\rm Pr}}
\newcommand{\Diag}{\mathrm{Diag}}

\def\vgap{\vspace*{.1in}}

\newcommand{\norm}[1]{\left\lVert#1\right\rVert}
\newcommand{\EE}{\mathbb{E}}
\newcommand{\sbr}[1]{\left[#1\right]}
\newcommand{\abs}[1]{\lvert #1 \rvert}
\newcommand{\RR}{\mathbb{R}}

\title{Policy Mirror Descent for Reinforcement Learning: Linear Convergence, 
New Sampling Complexity, and Generalized Problem Classes
\thanks{This research was partially supported by the NSF grants 1909298 and 1953199 and NIFA grant 2020-67021-31526.
The paper was first released at https://arxiv.org/abs/2102.00135 on 01/30/2021.}}
\author{
     Guanghui Lan 
    \thanks{H. Milton Stewart School of Industrial and Systems
    Engineering, Georgia Institute of Technology, Atlanta, GA, 30332.
    (email: {\tt george.lan@isye.gatech.edu}).}
}

\date{Submitted: Feb 5, 2021; Revised: Oct 26, 2021; Accepted: April 5, 2022.}

\begin{document}

\maketitle

\begin{abstract}
We present new policy mirror descent (PMD) methods
for solving
reinforcement learning (RL) problems
with either strongly convex or general convex regularizers.
By exploring the structural properties
of these overall highly nonconvex problems we show that the PMD methods
exhibit fast linear rate of convergence to the global optimality.
 We develop stochastic
 counterparts of these methods, and establish an ${\cal O}(1/\epsilon)$
 (resp., ${\cal O}(1/\epsilon^2)$) sampling complexity for solving these RL problems with
 strongly (resp., general) convex regularizers using different sampling schemes, where $\epsilon$
 denote the target accuracy. We further 
 show that the complexity for computing the gradients of these regularizers, if necessary,
 can be bounded by ${\cal O}\{(\log_\gamma \epsilon) [(1-\gamma)L/\mu]^{1/2}\log (1/\epsilon)\}$
 (resp., ${\cal O} \{(\log_\gamma \epsilon ) (L/\epsilon)^{1/2}\}$)
 for problems with strongly (resp., general) convex regularizers. Here $\gamma$ denotes
 the discounting factor.
To the best of our knowledge, these complexity bounds,
along with our algorithmic developments,
appear to be new in both optimization and RL literature.
The introduction of these convex regularizers also 
greatly enhances the flexibility and thus expands the applicability of RL models.
\end{abstract}

\section{Introduction} \label{sec_intro}
In this paper, we study a general
class of reinforcement learning (RL)
problems involving either covex or strongly convex
regularizers in their cost functions.
Consider the finite Markov decision process
$M = (\cS, \cA, \cP, c, \gamma)$, where $\cS$
is a finite state space, $\cA$ is a finite action space,
$P: \cS \times \cS \times \cA \to \bbr$ is transition 
model, $c: \cS \times \cA \to \bbr$ is the cost function,
and $\gamma \in (0,1)$ is the discount factor.
A policy $\pi: \cA \times \cS \to \bbr$ determines
the probability of selecting a particular action at a given state.

For a given policy $\pi$,
we measure its performance by the action-value
function ($Q$-function) $Q^\pi: \cS \times \cA \to \bbr$
defined as
\begin{align}
Q^\pi(s,a) &:=
\bbe\left[\tsum_{t=0}^\infty \gamma^t [c(s_t, a_t) + h^\pi(s_t)] \right. \nn \\
&\quad \quad\quad\quad \left. \mid s_0 = s, a_0 = a, a_t \sim \pi(\cdot | s_t), s_{t+1} \sim \cP(\cdot | s_t, a_t)\right]. \label{eq:def_Q_function}
\end{align}
Here $h^\pi$ is a closed convex function w.r.t. the policy $\pi$, i.e., there exist some
$ \mu \ge 0$ s.t.
\beq \label{eq:convex_obj}
h^\pi(s) - [h^{\pi'}(s) + \langle  (h')^{\pi'}(s, \cdot), \pi(\cdot|s ) -  \pi'(\cdot|s )\rangle ] \ge \mu D_{\pi'}^\pi(s),
\eeq
where $\langle \cdot, \cdot \rangle$ denotes
the inner product over the action space $\cA$, $(h')^{\pi'}(s, \cdot)$ 
denotes a subgradient of $h(s)$ at $\pi'$, and $D_{\pi'}^\pi(s)$ is the 
Bregman's distance or Kullback–Leibler (KL) divergence  between $\pi$ and $\pi'$ (see Subsection~\ref{sec:notations_PMD} for
more discussion).

Clearly, if $h^\pi = 0$, then $Q^\pi$ becomes the classic action-value function.
If $h^\pi(s) = \mu D_{\pi_0}^\pi(s)$ for some $\mu > 0$, then $Q^\pi$ reduces to
the so-called entropy regularized action-value function.
The incorporation of a more general convex regularizer $h^\pi$
allows us to not only unify these two cases, but also to greatly enhance
the expression power and thus the applicability of RL. For example, 
by using either the indicator function, quadratic penalty or barrier functions,
$h^\pi$ can model the set of constraints that an optimal policy should satisfy.
It can describe the correlation among different actions for different states.
$h^\pi$ can also model some risk or utility function associated with the policy $\pi$.
Throughout this paper, we say that $h^\pi$ is a strongly convex regularizer if $\mu > 0$.
Otherwise, we call $h^\pi$ a general convex regularizer. Clearly
the latter class of problems covers
the regular case with $h^\pi = 0$.

We define the state-value function
$V^\pi: \cS \to \bbr$ associated with $\pi$
as
\begin{align} 
V^\pi(s) &:=
\bbe
 \left[
 \tsum_{t=0}^\infty \gamma^t [c(s_t, a_t) + h^\pi(s_t)] \right.\nn \\
 & \quad \quad \quad \quad \left. \mid
s_0 = s, a_t \sim \pi(\cdot | s_t),
s_{t+1} \sim \cP(\cdot | s_t, a_t)
\right].\label{eq:def_V_function}
\end{align}
It can be easily seen 
from the definitions of $Q^\pi$ and $V^\pi$ that
\begin{align}
V^\pi(s) &= \tsum_{a \in \cA} \pi( a | s) Q^\pi(s, a)  = \langle Q^\pi(s, \cdot),  \pi( \cdot | s) \rangle, \label{eq:QV1}  \\
Q^\pi(s, a) &= c(s, a) + h^\pi(s) + \gamma \tsum_{s' \in \cS} \cP(s'| s, a) V^\pi(s'). \label{eq:QV2}
\end{align}
The main objective in RL is to find an optimal policy $\pi^*: \cS \times \cA \to \bbr$ s.t.
\beq \label{eq:opt_objective}
V^{\pi^*}(s) \le V^\pi(s), \forall \pi(\cdot | s) \in \Delta_{|\cA|}, \forall s \in \cS.
\eeq 
for any $s \in \cS$. Here $\Delta_{|\cA|}$ denotes the 
simplex constraint given by
\beq \label{eq:def_simplex}
 \Delta_{|\cA|} := \{ p \in \bbr^{|\cA|} | \tsum_{i =1}^{|\cA|} p_i = 1, p_i \ge 0\}, \forall s \in \cS.
\eeq

By examining Bellman's optimality condition
for dynamic programming (\cite{BellmanDreyfus1959} and Chapter 6 of \cite{PutermanBook1994}), we can show the existence
of a policy $\pi^*$ which satisfies \eqnok{eq:opt_objective} simultaneously 
for all $s \in \cS$. Hence, we can formulate \eqnok{eq:opt_objective}
as an optimization problem with a single objective by taking the weighted sum of
$V^\pi$ over $s$ (with weights $\rho_s >0$ and $\tsum_{s \in S} \rho_s = 1$):
\beq \label{eq:opt_objective1}
\begin{array}{ll}
\min_\pi &  \bbe_{s \sim \rho} [V^{\pi}(s)] \\
\mbox{s.t.} & \pi(\cdot | s) \in \Delta_{|\cA|}, \forall s \in \cS.
\end{array}
\eeq
While the weights $\rho$ can be arbitrarily chosen, a reasonable selection
of $\rho$ would be the stationary state distribution
induced by the optimal policy $\pi^*$, denoted by $\nu^* \equiv \nu(\pi^*)$. As such,
problem~\eqnok{eq:opt_objective1} reduces to 
\beq \label{eq:MDP_OPT}
\begin{array}{ll}
\min_\pi & \left\{ f(\pi) := \bbe_{s \sim \nu^*} [V^{\pi}(s)] \right\}\\
\mbox{s.t.} & \pi(\cdot | s) \in \Delta_{|\cA|}, \forall s \in \cS.
\end{array}
\eeq
It has been observed recently (eg., \cite{LiuCaiYangWang2019a})
that one can simplify the analysis of various algorithms by setting $\rho$ to $\nu^*$. 
As we will also see later, 
even though the definition of the objective $f$ in \eqnok{eq:MDP_OPT}
depends on $\nu^*$ and hence the unknown optimal policy $\pi^*$,
the algorithms for solving \eqnok{eq:opt_objective} and 
\eqnok{eq:MDP_OPT} do not really require
the input of $\pi^*$.

Recently, there has been considerable interest in the development of
first-order methods for solving RL problems in \eqnok{eq:opt_objective1}
-\eqnok{eq:MDP_OPT}. While these methods have been derived under various names (e.g., policy gradient, natural policy gradient,
trust region policy optimization), they all utilize the gradient information of $f$ (i.e., $Q$ function)
in some form to guide the search of optimal policy (e.g.,~\cite{SuttonMcAllester1999,KakadeLangford2002,10.2307/40538442,AgarwalKakadeLeeeMhhajan2019,DBLP:conf/aaai/ShaniEM20,2020arXiv200706558C,Wang2020NeuralPG,2020arXiv200506392M}).
As pointed out by a few authors recently, many of these algorithms are intrinsically
connected to the classic mirror descent method originally presented by Nemirovski and Yudin~\cite{nemyud:83,BeckTeb03-1,NJLS09-1},
and some analysis techniques in mirror descent method have thus been adapted to 
reinforcement learning~\cite{DBLP:conf/aaai/ShaniEM20,Wang2020NeuralPG,Tomar2020MirrorDP}. In spite of the popularity of these methods in practice, a few
significant issues remain on their theoretical studies. Firstly, 
most policy gradient methods converge only sublinearly, while many other classic algorithms (e.g.,
policy iteration) can converge at a linear rate due to the contraction properties
of the Bellman operator.
Recently, there are some interesting works relating first-order methods
with the Bellman operator to establish their linear convergence~\cite{2020arXiv200711120B,2020arXiv200706558C}. However,
in a nutshell these developments rely on the contraction of the Bellman operator, and as a consequence,  
they either require unrealistic algorithmic assumptions (e.g.,
exact line search~\cite{2020arXiv200711120B}) or apply only for some restricted problem classes (e.g., 
entropy regularized problems~\cite{2020arXiv200706558C}). 
Secondly, the convergence of stochastic policy gradient methods
has not been well-understood in spite of intensive research effort. 
Due to unavoidable bias, stochastic policy gradient methods exhibit much slower rate of convergence
than related methods, e.g., stochastic Q-learning. 

Our contributions in this paper mainly exist in the following several aspects.
Firstly, we present a policy mirror descent (PMD) method and
show that it can achieve a linear rate of convergence for solving
RL problems with strongly convex regularizers. We then develop
a more general form of PMD, namely approximate policy mirror descent (APMD) method,
obtained
by applying an adaptive perturbation term into PMD,
and show that it can achieve
a linear rate of convergence for solving RL problems
with general convex regularizers. Even though
the overall problem is highly nonconvex, we exploit the
generalized monotonicity~\cite{DangLan12-1,LanBook2020,KotsalisLanLi2020PartI} associated with the variational inequality (VI) reformulation
of \eqnok{eq:opt_objective1}-\eqnok{eq:MDP_OPT} (see~\cite{FacPang03}
for a comprehensive introduction to VI). As a consequence,
our convergence analysis does not rely on the contraction properties  
of the Bellman operator. This fact not only enables us to define
$h^\pi$ as a general (strongly) convex function of $\pi$ and thus
expand the problem classes considered in RL,  but also
facilitates the study of PMD methods
under the stochastic settings.

Secondly, we develop the stochastic policy mirror descent (SPMD)
and stochastic approximate policy mirror descent (SAPMD)
method to handle stochastic first-order information.
One key idea of SPMD and SAPMD is to handle separately
the bias and expected error of the stochastic estimation of the action-value functions in our convergence analysis,
since we can usually reduce the bias term much faster than the total expected error.
We establish general convergence results for both SPMD and SAPMD
applied to solve RL problems with strongly convex and general 
convex regularizers, under different conditions about the bias and expected error
associated with the estimation of value functions. 

Thirdly, we establish the overall sampling complexity of these algorithms
by employing different schemes to estimate the action-value function.
More specifically, we present an ${\cal O}(|\cS| |\cA|/\mu \epsilon)$
and ${\cal O} (|\cS| |\cA|/\epsilon^2)$ sampling complexity
for solving RL problems with strongly convex and general convex
regularizers, when one has access to multiple independent sampling
trajectories. To the best of our knowledge, the former sampling
complexity is new in the RL literature, while
the latter one has not been reported before for policy gradient type methods. We further enhance
a recently developed conditional temporal difference (CTD) method~\cite{KotsalisLanLi2020PartII}
so that it can reduce the bias term faster. We show that with CTD, 
 the aforementioned ${\cal O}(1/\mu \epsilon)$
and ${\cal O} (1/\epsilon^2)$ sampling complexity bounds 
can be achieved in the single trajectory setting with Markovian noise under 
certain regularity assumptions.

Fourthly, observe that unless $h^\pi$ is relatively simple (e.g., $h^\pi$ does not exist or 
it is given as the KL divergence), the subproblems in the SPMD and SAPMD methods
do not have an explicit solution in general and
require an efficient solution procedure to find some approximate solutions. We
establish the general conditions on the accuracy for solving
these subproblems, so that the aforementioned linear rate of convergence and 
new sampling complexity bounds 
can still be maintained. We further show that if $h^\pi$ is a smooth convex function, by employing
an accelerated gradient descent method for solving
these subproblems, the overall gradient computations
for $h^\pi$ can be bounded by ${\cal O}\{(\log_\gamma \epsilon) \sqrt{(1-\gamma)L/\mu} \log (1/\epsilon)\}$
and ${\cal O}\{ (\log_\gamma \epsilon) \sqrt{L/\epsilon}\}$,
respectively, for the case when $h^\pi$ is a strongly convex and general convex function.
To the best of our knowledge, such gradient complexity has not been considered before in the RL 
and optimization literature.

This paper is organized as follows. In Section~\ref{sec_optimality_monotonicity},
we discuss the optimality conditions and generalized monotonicity about RL with convex regularizers.
Sections~\ref{sec:exact_PMD} and \ref{sec_SPMD} are dedicated to the
deterministic and stochastic policy mirror descent methods, respectively.
In Section~\ref{sec:sampling_complexity} we establish the sampling complexity bounds
under different sampling schemes, while the gradient complexity of
computing $\nabla h^\pi$ is shown in Section~\ref{sec:gradient_complexity}.
Some concluding remarks are made in Section~\ref{sec_conclusion}.

\subsection{Notation and terminology} \label{sec:notations_PMD}

For any two points $\pi(\cdot|s), \pi'(\cdot|s) \in \Delta_{|\cA|}$,
we measure their Kullback–Leibler (KL)  divergence by
\[
\KL(\pi(\cdot|s) \parallel \pi'(\cdot|s)) = \tsum_{a \in \cA} \pi(a|s) \log \tfrac{\pi(a|s)}{\pi'(a|s)}.
\]
Observe that the KL divergence can be viewed as
is a special instance of the 
Bregman's distance (or prox-function)
widely used in the optimization literature. 
Let the distance generating function $\omega(\pi(\cdot|s)) := \tsum_{a \in \cA} \pi(a|s) \log \pi(a|s)$~\footnote{It is worth noting that
we do not enforce $\pi(a|s) > 0$ when defining $\omega(\pi(\cdot|s))$ as all the search points generated by our algorithms
will satisfy this assumption.}. The Bregman's distance
associated with $\omega$ is given by
\begin{align}
D_{\pi'}^\pi(s) &:= \omega(\pi(\cdot|s)) - [\omega(\pi'(\cdot|s)) + \langle \nabla \omega(\pi'(\cdot|s)), \pi(\cdot|s) - \pi'(\cdot|s)\rangle ] \nn\\
&= \tsum_{a \in \cA} \left[ \pi(a|s) \log \pi(a|s) - \pi'(a|s) \log \pi'(a|s) \right. \nn\\
& \quad \quad \quad \quad \quad \left. - (1 + \log \pi'(a|s)) (\pi(a|s)- \pi'(a|s))\right] \nn\\
&= \tsum_{a \in \cA} \pi(a|s) \log \tfrac{\pi(a|s)}{\pi'(a|s)},
\end{align}
where the last equation follows from the fact that $\tsum_{a \in \cA}(\pi(a|s)- \pi'(a|s)) = 0$.
Therefore, we will use the KL divergence $\KL(\pi(\cdot|s) \parallel \pi'(\cdot|s))$
and Bregman's distance $D_{\pi'}^\pi(s)$ interchangeably throughout this paper. 
It should be noted that our algorithmic framework allows us to use other distance generating functions, such as
$\|\cdot\|_p^2$ for some $p > 1$, which, different from the KL divergence,
has a bounded prox-function over $\Delta_{|\cA|}$. 

\section{Optimality Conditions and Generalized Monotonicity} \label{sec_optimality_monotonicity}
It is well-known that the value function $V^\pi(s)$ in \eqnok{eq:def_V_function}
is highly nonconvex w.r.t. $\pi$, because the components of $\pi(\cdot | s)$
are multiplied by each other in their definitions (see also Lemma 3 of \cite{AgarwalKakadeLeeeMhhajan2019}
for an instructive counterexample). However, we will show in this subsection
that problem \eqnok{eq:MDP_OPT} 
can be formulated as a variational inequality (VI) which satisfies  
certain generalized monotonicity properties (see \cite{DangLan12-1}, Section 3.8.2 of~\cite{LanBook2020}
and \cite{KotsalisLanLi2020PartI}).

Let us first compute the gradient of the value function $V^\pi(s)$ in \eqnok{eq:def_V_function}. 
For simplicity, we assume for now that $h^\pi$ is differentiable and will relax this assumption later. 
For a given policy $\pi$, we define the discounted state visitation distribution by
\beq \label{eq:def_visitation}
d_{s_0}^\pi(s) := (1 - \gamma) \tsum_{t=0}^\infty \gamma^t \Pr^\pi(s_t = s | s_0),
\eeq
where $\Pr^\pi(s_t = s | s_0)$ denotes the state visitation probability of $s_t =s$ after we follow
the policy $\pi$ starting at state $s_0$. Let $\cP^\pi$ denote the transition probability
matrix associated with policy $\pi$, i.e., $\cP^\pi(i,j) = \tsum_{a \in \cA} \pi(a | i) \cP(j| i, a)$,
and  $e_i$ be the $i$-th unit vector.
Then $\Pr^\pi(s_t = s | s_0) = e_{s_0}^T(\cP^\pi)^t e_s$ and
\beq \label{eq:def_visitation1}
d_{s_0}^\pi(s) = (1 - \gamma) \tsum_{t=0}^\infty \gamma^t e_{s_0}^T(\cP^\pi)^t e_s.
\eeq

\vgap

\begin{lemma} \label{lemma:gradient}
For any $(s_0, s, a) \in \cS \times \cS \times \cA$, we have
\[
\tfrac{\partial V^\pi(s_0)}{\partial \pi(a|s)}
= \tfrac{1}{1-\gamma} d_{s_0}^\pi(s)
\left[ 
Q^\pi(s,a) +  \nabla h^\pi(s, a) 
\right],
\]
where $\nabla h^\pi(s,\cdot)$ denotes the gradient of $h^\pi(s)$ w.r.t. $\pi$.
\end{lemma}

\begin{proof}
It follows from \eqnok{eq:QV1} that
\begin{align*}
\tfrac{\partial V^\pi(s_0)}{\partial \pi(a|s)}
&= \tfrac{\partial}{\partial \pi(a'|s)} \tsum_{a'\in \cA} \pi(a'|s_0) Q^\pi(s_0, a')\\
&=  \tsum_{a' \in \cA} 
\left[
\tfrac{\partial \pi(a'|s_0)}{\partial \pi(a|s)} Q^\pi(s_0, a')
+ \pi(a'|s_0) 
\tfrac{\partial Q^\pi(s_0, a')}{\partial \pi(a|s)}
\right].
\end{align*}
Also the relation in \eqnok{eq:QV2}
implies that
\begin{align*}
\tfrac{\partial Q^\pi(s_0, a')}{\partial \pi(a|s)}
= \tfrac{  \partial h^\pi(s_0)}{\partial \pi (a|s)} +
\gamma \tsum_{s' \in \cS} 
\cP(s'| s_0, a') \tfrac{\partial V^\pi(s')}{\partial \pi(a|s)}.
\end{align*}
Combining the above two relations, we obtain
\begin{align*}
\tfrac{\partial V^\pi(s_0)}{\partial \pi(a|s)}
&= \tsum_{a' \in \cA} \left[ \tfrac{\partial \pi(a'|s_0)}{\partial \pi(a|s)} Q^\pi(s_0, a')
+ \pi(a'|s_0)  \tfrac{  \partial h^\pi(s_0)}{\partial \pi (a|s)}\right]  \\
& \quad + \gamma \tsum_{a' \in \cA} \pi(a'|s_0) \tsum_{s' \in \cS} 
\cP(s'| s_0, a') \tfrac{\partial V^\pi(s')}{\partial \pi(a|s)} \\
&=\tsum_{x\in \cS} \tsum_{t=0}^\infty \gamma^t \Pr^\pi(s_t = x | s_0) \\
&\quad \quad  \tsum_{a' \in \cA} \left[ \tfrac{\partial \pi(a'|x)}{\partial \pi(a|s)} Q^\pi(x, a')
+ \pi(a'|x)  \tfrac{   \partial h^\pi(x)}{\partial \pi (a|s)}\right]\\
&=\tfrac{1}{1-\gamma} \tsum_{x\in \cS} d_{s_0}^\pi(x) \left\{
\tsum_{a' \in \cA} \left[ \tfrac{\partial \pi(a'|x)}{\partial \pi(a|s)} Q^\pi(x, a')\right]
+ \tfrac{   \partial h^\pi(x)}{\partial \pi (a|s)}\right\}\\
&= \tfrac{1}{1-\gamma} d_{s_0}^\pi(s) \left[ Q^\pi(s, a) + \tfrac{   \partial h^\pi(s)}{\partial \pi (a|s)}\right],
\end{align*}
where the second equality follows by
expanding $\tfrac{\partial V^\pi(s')}{\partial \pi(a|s)}$ recursively,
and the third equality follows from the definition of $d_{s_0}^\pi(s)$
in \eqnok{eq:def_visitation}, and the last identity follows from $\tfrac{\partial \pi(a'|x)}{\partial \pi(a|s)} =0$
for $x \neq s$ or $a' \neq a$, and $\tfrac{\partial h^\pi(x)}{\partial \pi (a|s)} =0$
for $x \neq s$.
\end{proof}

\vgap

In view of Lemma~\ref{lemma:gradient}, the gradient of the objective
 $f(\pi)$ in \eqnok{eq:MDP_OPT} at the optimal policy $\pi^*$  is given by
\begin{align}
\tfrac{\partial f(\pi^*) }{\partial \pi(a|s)}
&=  \bbe_{s_0 \sim \nu^*}\left[ \tfrac{\partial V^{\pi^*}(s_0)}{\partial \pi(a|s)}\right]
=  \tfrac{1}{1-\gamma}  \bbe_{s_0 \sim \nu^*}\left[ d_{s_0}^{\pi^*}(s)
[Q^{\pi^*}(s,a) + \nabla h^{\pi^*}(s, a)] \right] \nn \\
&= \tsum_{t=0}^\infty \gamma^t (\nu^*)^T (\cP^{\pi^*})^t e_s \, [Q^{\pi^*}(s,a) + \nabla h^{\pi^*}(s, a)] \nn \\
&= \tfrac{1}{1-\gamma} (\nu^*)^T e_s \, [Q^{\pi^*}(s,a) + \nabla h^{\pi^*}(s, a)] \nn \\
&= \tfrac{1}{1-\gamma} \nu^*(s) \, [Q^{\pi^*}(s,a) + \nabla h^{\pi^*}(s, a)], \label{eq:stationary_relation}
\end{align}
where the third identity follows from \eqnok{eq:def_visitation1} and
the last one follows from 
the fact that $(\nu^*)^T (\cP^{\pi^*})^t = (\nu^*)^T$ for any $t \ge 0$ 
 since $\nu^*$
is the steady state distribution of $\pi^*$. 
Therefore, the optimality condition of \eqnok{eq:MDP_OPT} suggests us to solve the following
variational inequality
\beq\label{eq:VI_MDP0}
\bbe_{s \sim \nu^*} \left[ \langle Q^{\pi^*}(s,\cdot) +  \nabla h^{\pi^*}(s, \cdot), \pi(\cdot |s) - \pi^*(\cdot |s)\rangle\right]  \ge 0.
\eeq
However, the above VI requires $h^\pi$ to be differentiable. 
In order to handle the possible non-smoothness of $h^\pi$, we instead solve the following problem
\beq\label{eq:VI_MDP}
\bbe_{s \sim \nu^*} \left[ \langle Q^{\pi^*}(s,\cdot), \pi(\cdot |s) - \pi^*(\cdot |s)\rangle + h^\pi(s) - h^{\pi^*}(s) \right]  \ge 0.
\eeq
It turns out this variational inequality satisfies certain generalized 
monotonicity properties thanks to the following performance difference lemma
obtained by generalizing some previous results (e.g., Lemma~6.1 of \cite{KakadeLangford2002}).

\begin{lemma} \label{lemma_per_diff}
For any two feasible policies $\pi$ and $\pi'$, we have
\[
V^{\pi'}(s) - V^{\pi}(s) =
\tfrac{1}{1-\gamma} \bbe_{s' \sim d_s^{\pi'}}
\left[
\langle A^{\pi}(s', \cdot), \pi'(\cdot | s')\rangle
+ h^{\pi'}(s') - h^{\pi}(s')
\right],
\]
where
\beq \label{eq:def_advantage}
A^{\pi}(s', a) := Q^{\pi}(s', a)-V^{\pi}(s').
\eeq
\end{lemma}

\begin{proof}
For simplicity, let us denote $\xi^{\pi'}(s_0)$ the random process $(s_t, a_t, s_{t+1})$, $t \ge 0$,
generated by following the policy $\pi'$ starting with the initial state $s_0$.
It then follows from the definition of $V^{\pi'}$ that
\begin{align*}
&V^{\pi'}(s) - V^{\pi}(s) \\
&= \bbe_{\xi^{\pi'}(s)}
 \left[
 \tsum_{t=0}^\infty \gamma^t [c(s_t, a_t) + h^{\pi'}(s_t)] \right]- V^{\pi}(s)\\
&=\bbe_{\xi^{\pi'}(s)}
 \left[
 \tsum_{t=0}^\infty \gamma^t [ c(s_t, a_t) + h^{\pi'}(s_t) + V^{\pi}(s_t)-V^{\pi}(s_t) ] \right]- V^{\pi}(s)\\
 &\stackrel{(a)}{=} \bbe_{\xi^{\pi'}(s)}
 \left[
 \tsum_{t=0}^\infty \gamma^t [ c(s_t, a_t) + h^{\pi'}(s_t) + \gamma V^{\pi}(s_{t+1})-V^{\pi}(s_t) ] \right] \\
&\quad \quad \quad \quad \quad+ \bbe_{\xi^{\pi'}(s)} [V^{\pi}(s_0)] - V^{\pi}(s)\\
&\stackrel{(b)}{=} \bbe_{\xi^{\pi'}(s)}
 \left[
 \tsum_{t=0}^\infty \gamma^t [ c(s_t, a_t) + h^{\pi'}(s_t) + \gamma V^{\pi}(s_{t+1})-V^{\pi}(s_t) ]\right] \\
 &= \bbe_{\xi^{\pi'}(s)}
 \left[
 \tsum_{t=0}^\infty \gamma^t [ c(s_t, a_t) + h^{\pi}(s_t) + \gamma V^{\pi}(s_{t+1})-V^{\pi}(s_t) \right. \\
&\quad \quad \quad \quad \quad \quad  \quad \quad  \left. + h^{\pi'}(s_t) -h^{\pi}(s_t) ] \right]\\
 &\stackrel{(c)}{=}\bbe_{\xi^{\pi'}(s)}
 \left[
 \tsum_{t=0}^\infty \gamma^t \left[  Q^{\pi}(s_t, a_t)-V^{\pi}(s_t)
 +h^{\pi'}(s_t) -h^{\pi}(s_t) \right] \right],
\end{align*}
where (a) follows by taking the term $V^{\pi}(s_0)$ outside the summation,
(b) follows from the fact that $\bbe_{\xi^{\pi'}(s)} [V^{\pi}(s_0)] = V^{\pi}(s)$ 
since the random process starts with $s_0 = s$, and (c) follows from \eqnok{eq:QV2}.
The previous conclusion, together with \eqnok{eq:def_advantage} and the definition $d_s^{\pi'}$ in \eqnok{eq:def_visitation}, then
imply that
\begin{align*}
&V^{\pi'}(s) - V^{\pi}(s) \\
&= \tfrac{1}{1-\gamma}\tsum_{s' \in \cS} \tsum_{a' \in \cA}
d_s^{\pi'}(s') \pi'(a' | s') \left[ A^{\pi}(s', a') +  h^{\pi'}(s') - h^{\pi}(s') \right]\\
&= \tfrac{1}{1-\gamma} \tsum_{s' \in \cS} d_s^{\pi'}(s')
\left[
\langle A^{\pi}(s', \cdot), \pi'(\cdot | s')\rangle
+ h^{\pi'}(s') - h^{\pi}(s')
\right], 
\end{align*}
which immediately implies the result.
\end{proof}

\vgap

We are now ready to prove the generalized monotonicity for the variational inequality in \eqnok{eq:VI_MDP}.

\begin{lemma} \label{prop_strong_monotone}
The VI problem in \eqnok{eq:VI_MDP} satisfies
\begin{align}
&\bbe_{s\sim \nu^*}
\left[
 \langle Q^{\pi}(s, \cdot),  \pi(\cdot|s) - \pi^*(\cdot | s)\rangle + h^\pi(s) - h^{\pi^*}(s)
\right] \nn \\
& \quad \quad =
 \bbe_{s\sim \nu^*}[(1- \gamma)(V^{\pi}(s) - V^{\pi^*}(s) ) ].
\end{align}
\end{lemma}

\begin{proof}
It follows from Lemma~\ref{lemma_per_diff} (with $\pi' = \pi^*$) that
\begin{align*}
(1- \gamma) [V^{\pi^*}(s) - V^{\pi}(s)] 
=\bbe_{s' \sim d_s^{\pi^*}}
\left[
\langle A^{\pi}(s', \cdot), \pi^*(\cdot | s')\rangle
+ h^{\pi^*}(s') - h^{\pi}(s')
\right].
\end{align*}
Let $e$ denote the vector of all $1$'s. Then, we have 
\begin{align}
\langle A^{\pi}(s', \cdot), \pi^*(\cdot | s')\rangle
&= \langle Q^{\pi}(s', \cdot) - V^\pi(s') e, \pi^*(\cdot | s')\rangle \nn\\
&= \langle Q^{\pi}(s', \cdot), \pi^*(\cdot | s')\rangle - V^\pi(s') \nn \\
&=  \langle Q^{\pi}(s', \cdot), \pi^*(\cdot | s')\rangle - \langle Q^{\pi}(s', \cdot), \pi(\cdot|s') \nn \\
&= \langle Q^{\pi}(s', \cdot), \pi^*(\cdot | s') - \pi(\cdot|s')\rangle, \label{eq:inner_product_equivalence}
\end{align}
where the first identity follows from the definition of $A^{\pi}(s', \cdot)$ in \eqnok{eq:def_advantage},
the second equality follows from the fact that $\langle e,  \pi^*(\cdot | s')\rangle = 1$,
and the third equality follows from the definition of $V^\pi$ in \eqnok{eq:def_V_function}.
Combining the above two relations and taking expectation w.r.t. $\nu^*$, we obtain
\begin{align*}
&(1- \gamma) \bbe_{s\sim \nu^*}[V^{\pi^*}(s) - V^{\pi}(s)] \\
&= \bbe_{s\sim \nu^*, s' \sim d_s^{\pi^*}}
\left[
 \langle Q^{\pi}(s', \cdot), \pi^*(\cdot | s') - \pi(\cdot|s')\rangle
 +   h^{\pi^*}(s') - h^{\pi}(s')
\right]\\
&= \bbe_{s\sim \nu^*}
\left[
 \langle Q^{\pi}(s, \cdot), \pi^*(\cdot | s) - \pi(\cdot|s)\rangle
 +   h^{\pi^*}(s) - h^{\pi}(s)
\right],
\end{align*}
where the second identity follows similarly to \eqnok{eq:stationary_relation}
since $\nu^*$ is the steady state distribution induced by $\pi^*$.
The result then follows by rearranging the terms.
\end{proof}

\vgap

Since $V^{\pi}(s) - V^{\pi^*}(s) \ge 0$ for any feasible policity $\pi$,
we conclude from Lemma~\ref{prop_strong_monotone} that 
\[
\bbe_{s\sim \nu^*}
\left[
 \langle Q^{\pi}(s, \cdot),  \pi(\cdot|s) - \pi^*(\cdot | s)\rangle + h^{\pi}(s) - h^{\pi^*}(s)
\right] \ge 0.
\]
Therefore, the VI in \eqnok{eq:VI_MDP} 
satisfies the generalized monotonicity.
In the next few sections, we will exploit the generalized monotonicity and
some other structural properties to design efficient algorithms
for solving the RL problem.

\section{Deterministic Policy Mirror Descent} \label{sec:exact_PMD}
In this section, we present the basic schemes of policy mirror descent (PMD) and establish
their convergence properties.

\subsection{Prox-mapping} \label{subsec_prox}
In the proposed PMD methods, we will update 
a given policy $\pi$ to $\pi^+$ through the following proximal mapping:
\beq \label{eq:prox-mapping}
\pi^+(\cdot|s) = \argmin_{p(\cdot|s) \in \Delta_{|\cA|}} \eta [ \langle G^\pi(s, \cdot), p(\cdot|s) \rangle + h^p(s)] + D_{\pi}^p(s).
\eeq
Here $\eta> 0$ denotes a certain stepsize (or learning rate), and $G^\pi$ can be the operator for the VI formulation, e.g., 
$G^\pi(s, \cdot) = Q^{\pi}(s,\cdot)$ or its approximation.

It is well-known that
one can solve \eqnok{eq:prox-mapping} explicitly for some interesting special cases, e.g.,
when $h^p(s) = 0$ or $h^p(s) = \tau D_{\pi_0}^p(s)$ for some $\tau > 0$ and given $\pi_0$. 
For both these cases, the solution of \eqnok{eq:prox-mapping} boils down to
solving a problem of the form
\[
p^* := \argmin_{p(\cdot|s) \in \Delta_{|\cA|}} \tsum_{i= 1}^{|\cA|} \left(g_i p_i + p_i \log p_i \right)
\]
for some $g \in \bbr^{|\cA|}$. It can be easily checked from 
the Karush-Kuhn-Tucker conditions that its optimal
solution is given by 
\beq \label{eq:prox-mapping-generic}
p^*_i = \exp(-g_i) / [\tsum_{i=1}^{|\cA|} \exp(-g_i)].
\eeq
For more general convex functions $h^p$, 
problem \eqnok{eq:prox-mapping} usually does not have an explicit solution, and
one can only solve it approximately. In fact, we will show 
in Section~\ref{sec:gradient_complexity} that by applying the accelerated gradient descent method,
we only need to compute a small number of updates in the form of \eqnok{eq:prox-mapping-generic}
in order to approximately solve \eqnok{eq:prox-mapping}
without slowing down the efficiency of the overall PMD algorithms.

\subsection{Basic PMD method}
As shown in Algorithm~\ref{basic_pmd}, each iteration of the
PMD method applies the prox-mapping step discussed in Subsection~\ref{subsec_prox}
to update the policy $\pi_k$.
It involves the stepsize parameter $\eta_k$ and requires the selection of an initial point $\pi_0$.
 For the sake of simplicity, we will assume throughout the paper that
 \beq \label{eq:initial_p0}
 \pi_0(a | s) = 1/|\cA|, \ \forall a \in \cA, \forall s \in \cS.
 \eeq
In this case, we have 
\beq \label{eq:boundnessofdomain}
D_{\pi_0}^\pi(s) = \tsum_{a \in \cA} \pi(a | s) \log \pi(a| s) + \log |\cA| \le \log |\cA|, \ \forall \pi(\cdot|s) \in \Delta_{|\cA|}.
\eeq
Observe also that we can replace $Q^{\pi_k}(s, \cdot)$ in \eqnok{eq:PMD_step}
with  $A^{\pi_{k}}(s, a)$ defined in \eqnok{eq:def_advantage} without 
impacting the updating of $\pi_{k+1}(s, \cdot)$, since this only introduces
an extra constant into the objective function of \eqnok{eq:PMD_step}. 

\begin{algorithm}[H]
\caption{The policy mirror descent (PMD) method}
\begin{algorithmic}
\State {\bf Input:} initial points $\pi_0$ and stepsizes $\eta_k\ge 0$.
\For {$k =0,1,\ldots,$}
\beq \label{eq:PMD_step}
\pi_{k+1}(\cdot|s) = \argmin_{p(\cdot|s) \in \Delta_{|\cA|}} \left\{ \eta_k[ \langle Q^{\pi_{k}}(s, \cdot), p(\cdot|s) \rangle + h^{p}(s)] + D_{\pi_k}^p(s)\right\},
\forall s \in \cS.
\eeq
\EndFor
\end{algorithmic} \label{basic_pmd}
\end{algorithm}

Below we establish some general convergence properties 
about the PMD method. 
Different from the classic policy iteration or value iteration method used in Markov Decision Processes,
our analysis does not rely on the contraction properties of the Bellman's operator, but on the
so-called three-point lemma associated with the optimality condition of problem~\eqref{eq:PMD_step} (see Lemma~\ref{lemma:prox_optimality}).
Our analysis also significantly differs from the one for the classic mirror descent method in convex optimization (see, e.g., Chapter 3 of \cite{LanBook2020}). First, the classic mirror descent method requires
the convexity of the objective function, while the analysis of PMD utilizes the generalized monotonicity in Lemma~\ref{prop_strong_monotone}. 
Second, the classic mirror descent utilizes the Lipschitz or smoothness properties
of the objective function, while in the PMD method, we
 show the progress made
in each iteration of this algorithm (see Lemma~\ref{prop:PMD_generic}) by using 
the performance difference lemma (c.f., Lemma~\ref{lemma_per_diff}) and the three-point lemma (c.f., Lemma~\ref{lemma:prox_optimality}). As a result, we make no assumptions about the smoothness properties
of the objective function at all.

The following result characterizes the optimality condition of problem~\eqnok{eq:PMD_step} (see Lemma 3.5
of \cite{LanBook2020}). We add a proof for the sake of completeness.

\begin{lemma} \label{lemma:prox_optimality}
For any $p(\cdot|s) \in \Delta_{|\cA|}$,
we have
\begin{align*}
&\eta_k[ \langle Q^{\pi_k}(s, \cdot), \pi_{k+1}(\cdot|s) - p(\cdot|s) \rangle + h^{\pi_{k+1}}(s) - h^{p}(s)]
+ D_{\pi_k}^{\pi_{k+1}}(s)\\
& \le D_{\pi_k}^p(s) -  (1 + \eta_k \mu) D_{\pi_{k+1}}^p(s).
\end{align*}
\end{lemma}

\begin{proof}
By the optimality condition of \eqnok{eq:PMD_step},
\[
\langle \eta_k [Q^{\pi_{k}}(s, \cdot) + (h')^{\pi_{k+1}}(s, \cdot)] + \nabla D_{\pi_k}^{\pi_{k+1}}(s, \cdot),  p(\cdot|s) - \pi_{k+1}(\cdot|s)  \rangle \ge 0,
\ \ \forall p(\cdot|s) \in \Delta_{|\cA|},
\]
where $(h')^{\pi_{k+1}}$ denotes the subgradient of $h$ at $\pi_{k+1}$ and
$\nabla D_{\pi_k}^{\pi_{k+1}}(s, \cdot)$ denotes the gradient of $D_{\pi_k}^{\pi_{k+1}}(s)$ at $\pi_{k+1}$.
Using the definition of Bregman's  distance, it is easy
to verify that
\beq \label{prox_identity}
D_{\pi_k}^p(s) = D_{\pi_k}^{\pi_{k+1}}(s) + \langle \nabla D_{\pi_k}^{\pi_{k+1}}(s, \cdot), p(\cdot|s) -\pi_{k+1}(\cdot|s) \rangle + D_{\pi_{k+1}}^p(s).
\eeq
The result then immediately follows by combining the above two relations together with \eqnok{eq:convex_obj}.
\end{proof}

\begin{lemma} \label{prop:PMD_generic}
For any $s \in \cS$, we have
\begin{align}
V^{\pi_{k+1}}(s) &\le V^{\pi_k}(s), \label{eq:pmd_decrease1}\\
\langle Q^{\pi_k}(s, \cdot),\, \pi_{k+1}(\cdot | s) - \pi_k(\cdot|s) \rangle + h^{\pi_{k+1}}(s) - h^{\pi_{k}}(s)  
&\ge V^{\pi_{k+1}}(s) - V^{\pi_k}(s). \label{eq:pmd_decrease2}
\end{align}
\end{lemma}

\begin{proof}
It follows from Lemma~\ref{lemma_per_diff} (with $\pi' = \pi_{k+1}$, $\pi = \pi_k$ and $\tau = \tau_k$) that
\begin{align}
&V^{\pi_{k+1}}(s) - V^{\pi_k}(s) \nn\\
&=
\tfrac{1}{1-\gamma} \bbe_{s' \sim d_s^{\pi_{k+1}}}
\left[
\langle A^{\pi_k}(s', \cdot), \pi_{k+1}(\cdot | s')\rangle
+ h^{\pi_{k+1}}(s') - h^{\pi_k}(s')
\right]. \label{eq:perf_diff_pmd}
\end{align}
Similarly to \eqnok{eq:inner_product_equivalence}, we can show that
\begin{align*}
\langle A^{\pi_k}(s', \cdot), \pi_{k+1}(\cdot | s')\rangle
&= \langle Q^{\pi_k}(s', \cdot) - V^{\pi_k}(s') e, \pi_{k+1}(\cdot | s')\rangle\\
&= \langle Q^{\pi_k}(s', \cdot), \pi_{k+1}(\cdot | s')\rangle - V^{\pi_k}_{\tau_k}(s')\\
&= \langle Q^{\pi_k}(s', \cdot), \pi_{k+1}(\cdot | s') - \pi_k(\cdot| s')\rangle.
\end{align*}
Combining the above two identities, we then obtain
\begin{align}
V^{\pi_{k+1}}(s) - V^{\pi_k}(s) &=
\tfrac{1}{1-\gamma} \bbe_{s' \sim d_s^{\pi_{k+1}}}
\left[
\langle Q^{\pi_k}(s', \cdot), \pi_{k+1}(\cdot | s') - \pi_k(\cdot|s') \rangle \right.\nn \\
& \quad \quad \quad \quad \quad \quad \quad \quad
\left. + h^{\pi_{k+1}}(s') - h^{\pi_k}(s')\right]. \label{eq:improvement_in_one_step_temp}
\end{align}
Now we conclude from Lemma~\ref{lemma:prox_optimality} applied to \eqnok{eq:PMD_step}
with $p(\cdot|s') = \pi_k(\cdot|s')$ that
\begin{align}
& \langle Q^{\pi_{k}}(s', \cdot), 
\pi_{k+1}(\cdot|s') - \pi_k(\cdot|s') \rangle + h^{\pi_{k+1}}(s') - h^{\pi_k}(s') \nn \\
&\le  -\tfrac{1}{\eta_k} [(1+\eta_k \mu) D_{\pi_{k+1}}^{\pi_k}(s') + D_{\pi_k}^{\pi_{k+1}}(s')].\label{eq:negative_term_pmd}
\end{align}
The previous two conclusions
then clearly imply the result in \eqnok{eq:pmd_decrease1}.
It also follows from \eqnok{eq:negative_term_pmd} that
\begin{align}
&\bbe_{s' \sim d_s^{\pi_{k+1}}}
\left[
\langle Q^{\pi_k}(s', \cdot), \pi_{k+1}(\cdot | s') - \pi_k(\cdot|s') \rangle 
 + h^{\pi_{k+1}}(s') - h^{\pi_k}(s') \right]\nn \\
& \le d_s^{\pi_{k+1}}(s) \left[
\langle Q^{\pi_k}(s, \cdot), \pi_{k+1}(\cdot | s) - \pi_k(\cdot|s) \rangle 
 + h^{\pi_{k+1}}(s) - h^{\pi_k}(s) \right]\nn \\
 &\le (1-\gamma) \left[
\langle Q^{\pi_k}(s, \cdot), \pi_{k+1}(\cdot | s) - \pi_k(\cdot|s) \rangle 
 + h^{\pi_{k+1}}(s) - h^{\pi_k}(s) \right], \label{eq:negative_term_pmd_prob_switch}
\end{align}
where the last inequality follows from the fact that $d_s^{\pi_{k+1}}(s) \ge (1-\gamma)$ due to the definition of $d_s^{\pi_{k+1}}$
in \eqnok{eq:def_visitation}. The result in \eqnok{eq:pmd_decrease1} 
then follows immediately from \eqnok{eq:improvement_in_one_step_temp} and the above inequality.
\end{proof}

Now we show that with a constant stepsize rule, the PMD method can achieve
a linear rate of convergence for solving RL problems with strongly convex regularizers (i.e., $\mu > 0$).

\begin{theorem} \label{theorem:PMD_reg}
Suppose that  $\eta_k = \eta$ for any $k \ge 0$ in the PMD method with
\beq \label{eq:def_eta_PMD_reg}
1 + \eta \mu \ge \tfrac{1}{\gamma}.
\eeq
Then we have
\begin{align*}
f(\pi_{k}) - f(\pi^*) + \tfrac{\mu}{1-\gamma}\cD (\pi_k, \pi^*)
 &\le \gamma^k [f(\pi_{0}) - f(\pi_\tau^*) + \tfrac{\mu}{1-\gamma}\log|\cA|]
 \end{align*}
for any $k \ge 0$,
where 
\beq \label{eq:def_distance_measure}
\cD (\pi_k, \pi^*) := \bbe_{s\sim \nu^*} [D_{\pi_{k}}^{\pi^*}(s)].
\eeq
\end{theorem}

\begin{proof}
By Lemma~\ref{lemma:prox_optimality} applied to \eqnok{eq:PMD_step}
(with $\eta_k = \eta$ and $p = \pi^*$), we have
\begin{align*}
&\eta [\langle Q^{\pi_{k}}(s, \cdot), 
\pi_{k+1}(\cdot|s) - \pi^*(\cdot|s) \rangle + h^{\pi_{k+1}}(s) - h^{\pi^*}(s)] + D_{\pi_k}^{\pi_{k+1}}(s)\\
&\le D_{\pi_k}^{\pi^*}(s) - (1+ \eta \mu) D_{\pi_{k+1}}^{\pi^*}(s),
\end{align*}
which, in view of \eqnok{eq:pmd_decrease2}, then implies that
\begin{align*}
&\eta [\langle Q^{\pi_{k}}(s, \cdot), 
\pi_{k}(\cdot|s) - \pi^*(\cdot|s) \rangle + h^{\pi_{k}}(s) - h^{\pi^*}(s)] \\
&+ \eta[V^{\pi_{k+1}}(s) - V^{\pi_k}(s) ] 
+ D_{\pi_k}^{\pi_{k+1}}(s)
\le D_{\pi_k}^{\pi^*}(s) - (1+ \eta \mu)D_{\pi_{k+1}}^{\pi^*}(s).
\end{align*}
Taking expectation w.r.t. $\nu^*$ on both sides of the above inequality
and using Lemma~\ref{prop_strong_monotone}, we arrive at
\begin{align*}
&\bbe_{s\sim \nu^*}[\eta(1- \gamma)(V^{\pi_k}(s) - V^{\pi_\tau^*}(s) )]
+ \eta \bbe_{s\sim \nu^*}[V^{\pi_{k+1}}(s) - V^{\pi_k}(s) ] 
+ \bbe_{s\sim \nu^*}[ D_{\pi_k}^{\pi_{k+1}}(s)]\\
&\le \bbe_{s\sim \nu^*}[D_{\pi_k}^{\pi^*}(s) - (1+ \eta \mu)D_{\pi_{k+1}}^{\pi^*}(s)].
\end{align*}
Noting $V^{\pi_{k+1}}(s) -  V^{\pi_k}(s)= 
V^{\pi_{k+1}}(s) - V^{\pi^*}(s) - [V^{\pi_k}(s) - V^{\pi^*}(s) ] $
and rearranging the terms in the above inequality, we have
\begin{align}
& \bbe_{s\sim \nu^*}[\eta(V^{\pi_{k+1}}(s) - V^{\pi^*}(s)) + (1+ \eta \mu) D_{\pi_{k+1}}^{\pi^*}(s)] +  \bbe_{s\sim \nu^*}[ D_{\pi_k}^{\pi_{k+1}}(s)] \nn\\
&\le \gamma \bbe_{s\sim \nu^*}[\eta (V^{\pi_k}(s) - V^{\pi^*}(s)) + \tfrac{1}{\gamma} D_{\pi_k}^{\pi^*}(s) ], \label{eq:general_deter_PMD}
\end{align}
which, in view of the assumption \eqnok{eq:def_eta_PMD_reg} and 
the definition of $f$ in \eqnok{eq:MDP_OPT}
\begin{align*}
&  f(\pi_{k+1}) - f(\pi^*) + \tfrac{\mu}{1-\gamma} \bbe_{s\sim \nu^*}[ D_{\pi_{k+1}}^{\pi^*}(s)]  \\
&\le  \gamma \left[  (f(\pi_k) - f(\pi^*) ) +   \tfrac{\mu}{1-\gamma} \bbe_{s\sim \nu^*}[ D_{\pi_{k}}^{\pi^*}(s)] \right ].
\end{align*}
Applying this relation recursively and using  the bound in \eqnok{eq:boundnessofdomain} we then conclude
the result.
\end{proof}

According to Theorem~\ref{theorem:PMD_reg}, the PMD method converges linearly in terms of both
function value and the distance to the optimal solution for solving
RL problems with strongly convex regularizers. 
Now we show that a direct application of 
the PMD method
only achieves a sublinear rate of convergence for the case when $\mu = 0$.

\begin{theorem} \label{theorem:PMD_non_reg}
Suppose that $\eta_k = \eta$ in the PMD method. Then we have
\begin{align*}
 f(\pi_{k+1}) -f(\pi^*)
\le  \tfrac{\eta \gamma[f(\pi_0) - f(\pi^*)] + \log |\cA| }{\eta (1 - \gamma)  (k+1)  }
\end{align*}
for any $k \ge 0$.
\end{theorem}

\begin{proof}
It follows from \eqnok{eq:general_deter_PMD} with $\mu = 0$ that
\begin{align*}
& \bbe_{s\sim \nu^*}[\eta(V^{\pi_{k+1}}(s) - V^{\pi^*}(s)) + D_{\pi_{k+1}}^{\pi^*}(s)] 
+  \bbe_{s\sim \nu^*}[ D_{\pi_k}^{\pi_{k+1}}(s)] \\
&\le \eta \gamma \bbe_{s\sim \nu^*}[V^{\pi_k}(s) - V^{\pi^*}(s)] + \bbe_{s\sim \nu^*}[D_{\pi_k}^{\pi^*}(s) ].
\end{align*}
Taking the telescopic sum of the above inequalities and using the fact that $V^{\pi_{k+1}}(s) 
\le V^{\pi_{k}}(s)$ due to \eqnok{eq:pmd_decrease1} , we obtain
\[
(k+1) \eta (1 - \gamma)   \bbe_{s\sim \nu^*}[V^{\pi_{k+1}}(s) - V^{\pi^*}(s)]
\le  \bbe_{s\sim \nu^*}[\eta \gamma(V^{\pi_0}(s) - V^{\pi^*}(s)) + D_{\pi_0}^{\pi^*}(s)],
\]
which clearly implies the result in view of 
the definition of $f$ in \eqnok{eq:MDP_OPT} and the bound on $D_{\pi_0}^{\pi^*}$ in
\eqnok{eq:boundnessofdomain}.
\end{proof}

\vgap

The result in Theorem~\ref{theorem:PMD_non_reg}
shows that the PMD method requires ${\cal O}(1/(1-\gamma) \epsilon)$
iterations to find an $\epsilon$-solution for general RL problems.
This bound already matches, in terms of its dependence on $(1-\gamma)$ and $\epsilon$,
 the previously best-known complexity  
for natural policy gradient methods~\cite{AgarwalKakadeLeeeMhhajan2019}. 
We will further enhance the PMD method
so that it can achieve a linear rate of convergence for the case when $\mu = 0$
in next subsection.

\subsection{Approximate policy mirror descent method} \label{sec_adPMD}
In this subsection, we propose to enhance the basic PMD method
by adding adaptively a perturbation term into 
the definition of the value functions or the proximal-mapping.

For some $\tau \ge 0$ and a given initial policy $\pi_0(a | s)  > 0$, $\forall s \in \cS, a \in \cA$, 
we define the perturbed action-value and state-value functions, respectively, by
\begin{align}
Q_\tau^\pi(s,a) &:=
\bbe\left[\tsum_{t=0}^\infty \gamma^t \left[c(s_t, a_t) + h^\pi(s_t) + \tau D_{\pi_0}^\pi(s_t)\right] \right.\nn\\
& \quad \quad \quad \quad \quad \left.  \mid
s_0 = s, a_0 = a, a_t \sim \pi(\cdot | s_t),
s_{t+1} \sim \cP(\cdot | s_t, a_t)\right], \label{eq:def_Q_tau} \\
V_\tau^\pi(s) &:= \langle Q_\tau^\pi(s,\cdot), \pi(\cdot|s) \rangle. \label{eq:def_V_tau}
\end{align}
Clearly, if $\tau = 0$, then the perturbed value functions reduce to the usual value functions, i.e.,
\[
Q_0^\pi(s,a) = Q^\pi(s,a) \ \ \mbox{and} \ \ V_0^\pi(s)  = V^\pi(s).
\]
The following result relates the value functions
with different $\tau$.

\begin{lemma}
For any given $\tau, \tau' \ge 0$, we have
\beq \label{eq:closeness_perturbation}
V_{\tau}^\pi(s) - V_{\tau'}^\pi(s) = \tfrac{\tau - \tau'}{1-\gamma} \bbe_{s' \sim d_s^\pi}[D_{\pi_0}^\pi(s')].
 \eeq
 As a consequence, if $\tau \ge \tau' \ge 0$ then
 \beq \label{eq:closeness_perturbation_direct}
 V_{\tau'}^\pi(s) \le V_{\tau}^\pi(s) \le V_{\tau'}^\pi(s) + \tfrac{\tau - \tau'}{1-\gamma} \log|\cA|.
 \eeq
\end{lemma}

\begin{proof}
By the definitions of $V_\tau^\pi$ and $d_s^\pi$, we have
\begin{align*}
&V_\tau^\pi(s) \\
&=
\bbe
 \left[
 \tsum_{t=0}^\infty \gamma^t [c(s_t, a_t)+h^\pi(s_t)+ \tau D_{\pi_0}^\pi(s)] \mid
s_0 = s, a_t \sim \pi(\cdot | s_t),
s_{t+1} \sim \cP(\cdot | s_t, a_t)
\right]\\
&= \bbe
 \left[
 \tsum_{t=0}^\infty \gamma^t [c(s_t, a_t)+h^\pi(s_t)+ \tau' D_{\pi_0}^\pi(s)] \mid
s_0 = s, a_t \sim \pi(\cdot | s_t),
s_{t+1} \sim \cP(\cdot | s_t, a_t)
\right]\\
&\quad+ \bbe
 \left[
 \tsum_{t=0}^\infty \gamma^t (\tau - \tau') D_{\pi_0}^\pi(s)] \mid
s_0 = s, a_t \sim \pi(\cdot | s_t),
s_{t+1} \sim \cP(\cdot | s_t, a_t)
\right]\\
&= V_{\tau'}^\pi(s) + \tfrac{\tau - \tau'}{1-\gamma} \bbe_{s' \sim d_s^\pi}[D_{\pi_0}^\pi(s')],
\end{align*}
which together with the bound on $D_{\pi_0}^{\pi}$ in
\eqnok{eq:boundnessofdomain} then imply \eqnok{eq:closeness_perturbation_direct}.
\end{proof}

\vgap

As shown in Algorithm~\ref{basic_adapmd}, the approximate policy mirror descent (APMD) method
is obtained by replacing $Q^{\pi_{k}}(s, \cdot)$ with its approximation $Q_{\tau_k}^{\pi_{k}}(s, \cdot)$
and adding the perturbation $\tau_k D_{\pi_0}^\pi(s_t)$ for the updating of $\pi_{k+1}$ in
the basic PMD method. As discussed in Subsection~\ref{subsec_prox},
the incorporation of the perturbation term does not impact the
difficulty of solving the subproblem in \eqnok{eq:PMD_step_ada}.
In fact, the APMD method can be viewed as a general form of the PMD method
since it reduces to the PMD method when $\tau_k = 0$.
In fact, the perturbation parameter $\tau_k$ used to define the action-value function
$Q_{\tau_k}^{\pi_{k}}(s, \cdot)$ is not necessarily the same as the one used
in the regularization term  $\tau_k D_{\pi_0}^p(s_t)$, yielding more flexibility
to the design and analysis for this class of algorithms.

\begin{algorithm}[H]
\caption{The approximate policy mirror descent (APMD) method}
\begin{algorithmic}
\State {\bf Input:} initial points $\pi_0$, stepsizes $\eta_k\ge 0$ and perturbation $\tau_k \ge 0$.
\For {$k =0,1,\ldots,$}
\beq \label{eq:PMD_step_ada}
\pi_{k+1}(\cdot|s) = \argmin_{p(\cdot|s) \in \Delta_{|\cA|}} \left\{ \eta_k[ \langle Q_{\tau_k}^{\pi_{k}}(s, \cdot), p(\cdot|s) \rangle + h^{p}(s)
+ \tau_k D_{\pi_0}^p(s_t)] + D_{\pi_k}^p(s)\right\},
\forall s \in \cS.
\eeq
\EndFor
\end{algorithmic} \label{basic_adapmd}
\end{algorithm}

Our goal in the remaining part of this subsection is to show that the APMD method,
when employed with proper selection of $\tau_k$,
can achieve a linear rate of convergence for solving general RL problems.
Note that in the classic mirror descent method, adding a perturbation term into the objective function
usually would not improve its rate of convergence from sublinear to linear.
However, the linear rate of convergence in PMD depends on the
discount factor rather than the strongly convex modulus of the regularization term,
which makes it possible for us to show a linear rate of convergence for the APMD method.

First we observe that Lemma~\ref{prop_strong_monotone}  can still be applied to
the perturbed value functions. The difference between the following 
result and Lemma~\ref{prop_strong_monotone} exists
in that the RHS of \eqnok{eq:strong_monotone_APMD} is no longer nonnegative, i.e.,
$V_\tau^{\pi}(s) - V_\tau^{\pi^*}(s) \ngeqslant 0$.
However, this relation will be approximately satisfied if $\tau$ is small enough.

\begin{lemma} \label{prop_strong_monotone_APMD}
The VI problem in \eqnok{eq:VI_MDP} satisfies
\begin{align}
&\bbe_{s\sim \nu^*}
\left[
 \langle Q_\tau^{\pi}(s, \cdot),  \pi(\cdot|s) - \pi^*(\cdot | s)\rangle + h^\pi(s) - h^{\pi^*}(s) + \tau [D_{\pi_0}^\pi(s) - D_{\pi_0}^{\pi^*}(s)]
\right] \nn \\
& \quad \quad =
 \bbe_{s\sim \nu^*}[(1- \gamma)(V_\tau^{\pi}(s) - V_\tau^{\pi^*}(s) ) ]. \label{eq:strong_monotone_APMD}
\end{align}
\end{lemma}

\begin{proof}
The proof is the same as that for Lemma~\ref{prop_strong_monotone} except that
we will apply the performance difference lemma (i.e., Lemma~\ref{lemma_per_diff})
to the perturbed value function $V_\tau^\pi$.
\end{proof}

Next we establish some general convergence properties 
about the APMD method.
Lemma~\ref{lemma:prox_optimality_ada} below characterizes the optimal solution of \eqnok{eq:PMD_step_ada} (see, e.g., Lemma 3.5 of \cite{LanBook2020}).
\begin{lemma} \label{lemma:prox_optimality_ada}
Let $\pi_{k+1}(\cdot|s)$ be defined in \eqnok{eq:PMD_step_ada}. For any $p(\cdot|s) \in \Delta_{|\cA|}$,
we have
\begin{align*}
&\eta_k[ \langle Q_{\tau_k}^{\pi_k}(s, \cdot), \pi_{k+1}(\cdot|s) - p(\cdot|s) \rangle + h^{\pi^+}(s) - h^{p}(s)] \\
&+ \eta_k \tau_k [D_{\pi_0}^{\pi_{k+1}}(s_t) - D_{\pi_0}^p(s_t)] 
+ D_{\pi_k}^{\pi_{k+1}}(s) \le D_{\pi_k}^p(s) -  (1 + \eta \tau_k) D_{\pi_{k+1}}^p(s).
\end{align*}
\end{lemma}

Lemma~\ref{prop:APMD_generic} below is similar to Lemma~\ref{prop:PMD_generic}
for the PMD method.

\begin{lemma} \label{prop:APMD_generic}
For any $s \in \cS$, we have
\begin{align}
&\langle Q_{\tau_k}^{\pi_k}(s, \cdot),\, \pi_{k+1}(\cdot | s) - \pi_k(\cdot|s) \rangle + h^{\pi_{k+1}}(s) - h^{\pi_{k}}(s) \nn \\
&+\tau_k[D_{\pi_0}^{\pi_{k+1}}(s) - D_{\pi_0}^{\pi_k}(s)] \ge V_{\tau_k}^{\pi_{k+1}}(s) - V_{\tau_k}^{\pi_k}(s). \label{eq:apmd_decrease2}
\end{align}
\end{lemma}

\begin{proof}
By applying Lemma~\ref{lemma_per_diff} to the perturbed value function $V_\tau^\pi$
and using an argument similar to \eqnok{eq:improvement_in_one_step_temp},
we can show that
\begin{align}
V_{\tau_k}^{\pi_{k+1}}(s) - V_{\tau_k}^{\pi_k}(s) &=
\tfrac{1}{1-\gamma} \bbe_{s' \sim d_s^{\pi_{k+1}}}
\left[
\langle Q_{\tau_k}^{\pi_k}(s', \cdot), \pi_{k+1}(\cdot | s') - \pi_k(\cdot|s') \rangle \right.\nn \\
& \quad
\left. + h^{\pi_{k+1}}(s') - h^{\pi_k}(s') +\tau_k[D_{\pi_0}^{\pi_{k+1}}(s) - D_{\pi_0}^{\pi_k}(s)] \right]. \label{eq:improvement_in_one_step_temp_apmd}
\end{align}
Now we conclude from Lemma~\ref{lemma:prox_optimality_ada}
with $p(\cdot|s') = \pi_k(\cdot|s')$ that
\begin{align}
& \langle Q_{\tau_k}^{\pi_{k}}(s', \cdot), 
\pi_{k+1}(\cdot|s') - \pi_k(\cdot|s') \rangle + h^{\pi_{k+1}}(s') - h^{\pi_k}(s')  \nn \\
&+ \tau_k [D_{\pi_0}^{\pi_{k+1}}(s') - D_{\pi_0}^{\pi_k}(s')] \le  -\tfrac{1}{\eta_k} [(1+\eta_k \tau_k) D_{\pi_{k+1}}^{\pi_k}(s') + D_{\pi_k}^{\pi_{k+1}}(s')],\label{eq:negative_term_apmd}
\end{align}
which implies that
\begin{align}
&\bbe_{s' \sim d_s^{\pi_{k+1}}}
\left[
\langle Q_{\tau_k}^{\pi_k}(s', \cdot), \pi_{k+1}(\cdot | s') - \pi_k(\cdot|s') \rangle 
 + h^{\pi_{k+1}}(s') - h^{\pi_k}(s') \right. \nn \\
&\quad \left. + \tau_k [D_{\pi_0}^{\pi_{k+1}}(s') - D_{\pi_0}^{\pi_k}(s')]\right]\nn \\
& \le d_s^{\pi_{k+1}}(s) \left[
\langle Q_{\tau_k}^{\pi_k}(s, \cdot), \pi_{k+1}(\cdot | s) - \pi_k(\cdot|s) \rangle 
 + h^{\pi_{k+1}}(s) - h^{\pi_k}(s) \right. \nn\\
 & \quad \left. + \tau_k [D_{\pi_0}^{\pi_{k+1}}(s) - D_{\pi_0}^{\pi_k}(s)] \right]\nn \\
 &\le (1-\gamma) \left[
\langle Q_{\tau_k}^{\pi_k}(s, \cdot), \pi_{k+1}(\cdot | s) - \pi_k(\cdot|s) \rangle 
 + h^{\pi_{k+1}}(s) - h^{\pi_k}(s) \right. \nn\\
 & \quad \left. + \tau_k [D_{\pi_0}^{\pi_{k+1}}(s) - D_{\pi_0}^{\pi_k}(s)] \right], \label{eq:negative_term_apmd_prob_switch}
\end{align}
where the last inequality follows from the fact that $d_s^{\pi_{k+1}}(s) \ge (1-\gamma)$ due to the definition of $d_s^{\pi_{k+1}}$
in \eqnok{eq:def_visitation}. The result in \eqnok{eq:apmd_decrease2} 
then follows immediately from \eqnok{eq:improvement_in_one_step_temp_apmd} and the above inequality.
\end{proof}

\vgap

The following general result holds for different stepsize rules for APMD.
\begin{lemma} \label{prop:PMD_un_reg_semi_final}
Suppose $1+ \eta_k \tau_k = 1/\gamma$ and $\tau_k \ge \tau_{k+1}$ in the APMD method. Then for any
$k \ge 0$, we have
\begin{align}
& \bbe_{s\sim \nu^*}[V_{\tau_{k+1}}^{\pi_{k+1}}(s) - V_{\tau_{k+1}}^{\pi^*}(s)  + \tfrac{\tau_{k+1}}{1 - \gamma }  D_{\pi_{k+1}}^{\pi^*}(s)]\nn\\
&\le   \bbe_{s\sim \nu^*}[ \gamma[ V_{\tau_k}^{\pi_k}(s) - V_{\tau_k}^{\pi^*}(s) +
\tfrac{\tau_k}{1 - \gamma}D_{\pi_k}^{\pi^*}(s)] + \tfrac{\tau_{k}-\tau_{k+1}}{1-\gamma}  \log|\cA|.   \label{eq:PMD_un_reg_semi_final}
\end{align}
\end{lemma}

\begin{proof}
By Lemma~\ref{lemma:prox_optimality_ada} with $p = \pi^*$, we have
\begin{align*}
&\eta_k \left[\langle Q_{\tau_k}^{\pi_{k}}(s, \cdot), 
\pi_{k+1}(\cdot|s) - \pi^*(\cdot|s) \rangle + h^{\pi_{k+1}}(s) - h^{\pi^*}(s)\right]\\
&+\eta_k \tau_k [ D_{\pi_0}^{\pi_{k+1}}(s_t) - D_{\pi_0}^{\pi^{*}}(s_t)] 
+ D_{\pi_k}^{\pi_{k+1}}(s) \\
&\le D_{\pi_k}^{\pi^*}(s) - (1+ \eta_k \tau_k ) D_{\pi_{k+1}}^{\pi^*}(s).
\end{align*}
Moreover, by Lemma~\ref{prop:APMD_generic},
\begin{align*}
\langle Q_{\tau_k}^{\pi_k}(s, \cdot),\, &\pi_{k+1}(\cdot | s) - \pi_k(\cdot|s) \rangle 
 + h^{\pi_{k+1}}(s) - h^{\pi_k}(s) +  \tau_k [ D_{\pi_0}^{\pi_{k+1}}(s_t) - D_{\pi_0}^{\pi_{k}}(s_t)]
\nn\\
&\ge V_{\tau_k}^{\pi_{k+1}}(s) - V_{\tau_k}^{\pi_k}(s).
\end{align*}
Combining the above two relations, we obtain
\begin{align*}
&\eta_k \left[\langle Q_{\tau_k}^{\pi_{k}}(s, \cdot) , 
\pi_{k}(\cdot|s) - \pi^*(\cdot|s) \rangle + h^{\pi_{k}}(s) - h^{\pi^*}(s)\right] + \eta_k \tau_k [ D_{\pi_0}^{\pi_{k}}(s_t) - D_{\pi_0}^{\pi^{*}}(s_t)] \\
&+ \eta_k [V_{\tau_k}^{\pi_{k+1}}(s) - V_{\tau_k}^{\pi_k}(s) ] 
+  D_{\pi_k}^{\pi_{k+1}}(s)
\le D_{\pi_k}^{\pi^*}(s) - (1+ \eta_k \tau_k ) D_{\pi_{k+1}}^{\pi^*}(s).
\end{align*}
Taking expectation w.r.t. $\nu^*$ on both sides of the above inequality
and using Lemma~\ref{prop_strong_monotone_APMD}, we arrive at
\begin{align*}
&\bbe_{s\sim \nu^*}[\eta_k(1- \gamma)(V_{\tau_k}^{\pi_k}(s) - V_{\tau_k}^{\pi^*}(s) )]
+ \eta_k \bbe_{s\sim \nu^*}[V_{\tau_k}^{\pi_{k+1}}(s) - V_{\tau_k}^{\pi_k}(s) ] \\
&+  \bbe_{s\sim \nu^*}[ D_{\pi_k}^{\pi_{k+1}}(s)]
\le \bbe_{s\sim \nu^*}[D_{\pi_k}^{\pi^*}(s) - (1+ \eta_k \tau_k)D_{\pi_{k+1}}^{\pi^*}(s)].
\end{align*}
Noting $V_{\tau_k}^{\pi_{k+1}}(s) -  V_{\tau_k}^{\pi_k}(s)= 
V_{\tau_k}^{\pi_{k+1}}(s) - V_{\tau_k}^{\pi^*}(s) - [V_{\tau_k}^{\pi_k}(s) - V_{\tau_k}^{\pi^*}(s) ] $
and rearranging the terms in the above inequality, we have
\begin{align}
& \bbe_{s\sim \nu^*}[\eta_k(V_{\tau_k}^{\pi_{k+1}}(s) - V_{\tau_k}^{\pi^*}(s)) + 
(1+ \eta_k \tau_k)D_{\pi_{k+1}}^{\pi^*}(s)+ D_{\pi_k}^{\pi_{k+1}}(s)]  \nn\\
&\le \eta_k \gamma \bbe_{s\sim \nu^*}[V_{\tau_k}^{\pi_k}(s) - V_{\tau_k}^{\pi^*}(s)] + \bbe_{s\sim \nu^*}[D_{\pi_k}^{\pi^*}(s) ].\label{eq:PMD_un_reg_temp}
\end{align}
Using the above inequality, the assumption $\tau_k \ge \tau_{k+1}$ and \eqnok{eq:closeness_perturbation_direct}, we have
\begin{align}
& \bbe_{s\sim \nu^*}[\eta_k(V_{\tau_{k+1}}^{\pi_{k+1}}(s) - V_{\tau_{k+1}}^{\pi^*}(s) ) + (1+ \eta_k \tau_k)  D_{\pi_{k+1}}^{\pi^*}(s)
+ D_{\pi_k}^{\pi_{k+1}}(s)] \nn\\
&\le   \bbe_{s\sim \nu^*}[\eta_k \gamma(V_{\tau_k}^{\pi_k}(s) - V_{\tau_k}^{\pi^*}(s)) +
D_{\pi_k}^{\pi^*}(s)] + \tfrac{\eta_k(\tau_{k}-\tau_{k+1})}{1-\gamma}  \log|\cA|,   \label{eq:PMD_un_reg_semi_final}
\end{align}
which implies the result by the assumption $1+ \eta_k \tau_k = 1/\gamma$.
\end{proof}

\vgap

We are now ready to establish the rate of convergence of
the APMD method with dynamic stepsize rules to select $\eta_k$ and $\tau_k$ for solving general RL problems.
\begin{theorem} \label{theorem_pmd_un_reg_optimal0}
Suppose that $\tau_k = \tau_0 \gamma^k$ for some $\tau_0 \ge 0$ and that
$1+ \eta_k \tau_k = 1/\gamma$ for any $k \ge 0$ in the APMD
method. Then for any $k \ge 0$, we have
\beq \label{eq:un_reg_MDP_linear_special_simple}
f(\pi_{k}) - f(\pi^*) 
\le \gamma^{k}  \left[f(\pi_0) - f(\pi^*) +
\tau_0 \left( \tfrac{2}{1-\gamma} + \tfrac{k}{\gamma }\right) \log |\cA|\right].
\eeq
\end{theorem}

\begin{proof}
Applying the result in Lemma~\ref{prop:PMD_un_reg_semi_final} 
recursively, we have
\begin{align*}
 \bbe_{s\sim \nu^*}[V_{\tau_{k}}^{\pi_{k}}(s) - V_{\tau_{k}}^{\pi^*}(s) ]
&\le \gamma^{k}  \bbe_{s\sim \nu^*}[V_{\tau_0}^{\pi_0}(s) - V_{\tau_0}^{\pi^*}(s)+ \tfrac{\tau_0}{1 - \gamma}D_{\pi_0}^{\pi^*}(s)] \\
&\quad + \tsum_{i=1}^k\tfrac{(\tau_{i-1} -\tau_i) \gamma^{k-i}}{1-\gamma }\log |\cA|.
\end{align*}
Noting that  $V_{\tau_{k}}^{\pi_{k}}(s) \ge V^{\pi_{k}}(s)$,
$V_{\tau_{k}}^{\pi^*}(s) \le V^{\pi^*}(s) + \tfrac{\tau_k}{1-\gamma}\log |\cA|$, and
$V_{\tau_0}^{\pi^*}(s) \ge V^{\pi^*}(s)$ due to  \eqnok{eq:closeness_perturbation}, and that
$V_{\tau_0}^{\pi_0}(s) = V^{\pi_0}(s)$ due to $D_{\pi_0}^{\pi_0}(s) = 0$, we conclude from the previous inequality that
\begin{align}
 \bbe_{s\sim \nu^*}[V^{\pi_{k}}(s) - V^{\pi^*}(s) ] 
&\le \gamma^{k}  \bbe_{s\sim \nu^*}[V^{\pi_0}(s) - V^{\pi^*}(s) + \tfrac{\tau_0}{1 - \gamma}D_{\pi_0}^{\pi^*}(s)]\nn \\
&\quad + \left[\tfrac{\tau_k}{1-\gamma} + 
 \tsum_{i=1}^k\tfrac{(\tau_{i-1} -\tau_i) \gamma^{k-i}}{1-\gamma } \right]\log |\cA|.\label{eq:RPMD_result_temp}
\end{align}
The result in \eqnok{eq:un_reg_MDP_linear_special_simple}
immediately follows from the above relation, the definition of $f$ in \eqnok{eq:MDP_OPT}, and the selection 
of $\tau_k$.
\end{proof}

\vgap

According to \eqnok{eq:un_reg_MDP_linear_special_simple}, if $\tau_0$ is a constant,
then the rate of convergence of the APMD method is ${\cal O} (k \gamma^k )$. 
If the total number of iterations $k$ is given a priori, we can improve the rate of
convergence to ${\cal O} (\gamma^k )$ by setting $\tau_0 = 1/k$. Below we propose
a different way to specify $\tau_k$ for the APMD method so that it can achieve
this  ${\cal O} (\gamma^k )$ rate of convergence without fixing $k$ a priori.

We first establish a technical result that will also be used later for the analysis of stochastic PMD methods.

\begin{lemma} \label{lemma:induction_tech_sequence}
Assume that the nonnegative sequences $\{X_k\}_{k \ge 0}, \{Y_k\}_{k \ge 0}$ and $\{Z_k\}_{k \ge 0}$ satisfy
\beq \label{eq:induction_tech_sequence}
X_{k+1} \le \gamma X_{k} + (Y_k - Y_{k+1}) + Z_k. 
\eeq
Let us denote $l = \left \lceil  \log_\gamma \tfrac{1}{4} \right \rceil$.
If $Y_k = Y \cdot 2^{-(\lfloor k/l \rfloor+1)}$ and $Z_k = Z \cdot 2^{-(\lfloor k/l \rfloor+2)}$ for some $Y\ge 0$ and $Z \ge 0$, 
then 
\beq  \label{eq:induction_tech_sequence_result}
X_k \le 2^{-\lfloor k/l \rfloor} (X_0 + Y + \tfrac{5Z}{4(1-\gamma)}).
\eeq
\end{lemma}

\begin{proof}
Let us group the indices $\{0, \ldots, k\}$  into $\bar p \equiv \lfloor k/l \rfloor+1$ epochs
with each of the first $\bar p - 1$ epochs consisting of $l$ iterations. Let $p=0, \ldots, \bar p$ be the epoch indices.
We first show that for any $p = 0, \ldots, \bar p - 1$,
\beq \label{eq:induction_tech_result}
X_{pl} \le 2^{-p} (X_0 + Y + \tfrac{Z}{1-\gamma} ).
\eeq
This relation holds obviously for $p = 0$. Let us assume that \eqnok{eq:induction_tech_result} holds at the beginning
of epoch $p$ ad examine the progress made in epoch $p$.
Note that for any indices  $k = p l, \ldots, (p+1)l-1$ in epoch $p$,
we have $Y_{k} = Y \cdot 2^{-(p+1)}$ and  $Z = Z \cdot 2^{-(p+2)}$.
By applying \eqnok{eq:induction_tech_sequence} recursively, we have
\begin{align*}
X_{(p+1) l} &\le \gamma^l X_{pl} + Y_{pl}- Y_{(p+1)l} + Z_{pl}  \tsum_{i=0}^{l-1} \gamma^{i} \\
&= \gamma^l X_{pl} + Y_{(p+1)l} + Z_{pl} \tfrac{1 - \gamma^l}{1-\gamma}\\
&\le \gamma^l X_{pl} + Y \cdot 2^{-(p+2)} +  \tfrac{Z \cdot 2^{-(p+2)}}{1-\gamma}\\
&\le \tfrac{1}{4}X_{pl} + Y \cdot 2^{-(p+2)} +  \tfrac{Z \cdot 2^{-(p+2)}}{1-\gamma}\\
&\le \tfrac{1}{4} 2^{-p} (X_0 +  Y + \tfrac{Z}{1-\gamma} ) + Y \cdot 2^{-(p+2)} +  \tfrac{Z \cdot 2^{-(p+2)}}{1-\gamma}\\
&\le 2^{-(p+1)} (X_0 + Y + \tfrac{Z}{1-\gamma} ),
\end{align*}
where the second inequality follows from the definition of $Z_{pl}$ and $\gamma^l \ge 0$,
the third one follows from $\gamma^l \le 1/4$, the fourth one follows by induction hypothesis, and
the last one follows by regrouping the terms.
Since $k = (\bar p -1) l + k\pmod l$, we have
\begin{align*}
X_k &\le \gamma^{{k\pmod l}} X_{(\bar p -1) l} + Z_{(\bar p -1)l} \tsum_{i=0}^{k\pmod l - 1} \gamma^i\\
&\le 2^{-(\bar p -1)} (X_0 + Y + \tfrac{Z}{1-\gamma} ) + \tfrac{Z \cdot 2^{-(\bar p +1)}}{1-\gamma}\\
&= 2^{-(\bar p -1)} (X_0 + Y + \tfrac{5 Z}{4(1-\gamma)} ),
\end{align*}
which implies the result.
\end{proof}

We are now ready to present a more convenient selection of $\tau_k$ and $\eta_k$ for the APMD method.
\begin{theorem} \label{theorem_pmd_un_reg_optimal}
Let us denote $
l :=\left \lceil  \log_\gamma \tfrac{1}{4} \right \rceil.
$
If
$
\tau_k = 2^{-(\lfloor k/l \rfloor+1)}$ 
and $1+\eta_k \tau_k = 1/\gamma$,
then
\[
f(\pi_{k}) - f(\pi^*) \le 2^{-\lfloor k/l \rfloor} [f(\pi_0) - f(\pi^*) + \tfrac{2  \log |\cA|}{1-\gamma}].
\]
\end{theorem}

\begin{proof}
By using Lemma~\ref{prop:PMD_un_reg_semi_final}
and Lemma~\ref{lemma:induction_tech_sequence} (with $X_k =\bbe_{s\sim \nu^*}[ V_{\tau_k}^{\pi_k}(s) - V_{\tau_k}^{\pi^*}(s) +
\tfrac{\tau_k}{1 - \gamma}D_{\pi_k}^{\pi^*}(s)]$ and $Y_k = \tfrac{\tau_{k}}{1-\gamma} \log|\cA|$), we have
\begin{align*}
&\bbe_{s\sim \nu^*}[ V_{\tau_k}^{\pi_k}(s) - V_{\tau_k}^{\pi^*}(s) +
\tfrac{\tau_k}{1 - \gamma}D_{\pi_k}^{\pi^*}(s)] \\
& \le 2^{-\lfloor k/l \rfloor} \left\{\bbe_{s\sim \nu^*}[ V_{\tau_0}^{\pi_k}(s) - V_{\tau_0}^{\pi^*}(s) +
\tfrac{\tau_0}{1 - \gamma}D_{\pi_0}^{\pi^*}(s)]  + \tfrac{\log|\cA|}{1-\gamma}  \right\}.
\end{align*}
Noting that  $V_{\tau_{k}}^{\pi_{k}}(s) \ge V^{\pi_{k}}(s)$,
$V_{\tau_{k}}^{\pi^*}(s) \le V^{\pi^*}(s) + \tfrac{\tau_k}{1-\gamma}\log |\cA|$, 
$V_{\tau_0}^{\pi^*}(s) \ge V^{\pi^*}(s)$ due to  \eqnok{eq:closeness_perturbation}, and that
$V_{\tau_0}^{\pi_0}(s) = V^{\pi_0}(s)$ due to $D_{\pi_0}^{\pi_0}(s) = 0$, we conclude from the previous inequality
and the definition of $\tau_k$ that
\begin{align*}
&\bbe_{s\sim \nu^*}[ V^{\pi_k}(s) - V^{\pi^*}(s) +
\tfrac{\tau_k}{1 - \gamma}D_{\pi_k}^{\pi^*}(s)] \\
& \le 2^{-\lfloor k/l \rfloor} \left\{\bbe_{s\sim \nu^*}[ V^{\pi_k}(s) - V_{\tau_0}^{\pi^*}(s) +
\tfrac{\tau_0}{1 - \gamma}D_{\pi_0}^{\pi^*}(s)]  + \tfrac{  \log|\cA|}{1-\gamma} \right\} + \tfrac{\tau_k \log |\cA|}{1-\gamma}\\
&\le 2^{-\lfloor k/l \rfloor} \left\{\bbe_{s\sim \nu^*}[ V^{\pi_0}(s) - V_{\tau_0}^{\pi^*}(s) + \tfrac{2  \log|\cA|}{1-\gamma} \right\}.
\end{align*}
\end{proof}

\vgap

In view of Theorem~\ref{theorem_pmd_un_reg_optimal},  a policy $\bar \pi$ s.t. $f(\bar \pi) - f(\pi^*) \le \epsilon$
will be found in at most ${\cal O} (\log (1/\epsilon))$ epochs and hence at most ${\cal O} (l \log (1/\epsilon)) =
 {\cal O}(\log_\gamma (\epsilon))$
iterations, which matches the one for solving RL problems with strongly convex regularizers. 
However, for general RL problems, we cannot guarantee the linear
convergence of $D_{\pi_{k+1}}^{\pi^*}(s)$ since its coefficient $\tau_k$ will become very
small eventually. 
By using the continuity of the objective function and the compactness of the feasible set,
we can possibly show that the solution sequence converges to the true optimal policy asymptotically as the number of iterations increases. 
On the other hand, the rate of convergence associated with
the solution sequence of the PMD method for general RL problems cannot be established
 unless more structural properties of the RL problems can be further explored.

\section{Stochastic Policy Mirror Descent} \label{sec_SPMD}
The policy mirror descent methods described in the previous section
require the input of the exact action-value functions $Q^{\pi_k}$.
This requirement can hardly be satisfied in practice even for the case when $\cP$ is given explicitly,
since $Q^{\pi_k}$ is defined as an infinite sum. In addition, in RL
one does not know the transition dynamics $\cP$ and thus only stochastic estimators
of action-value functions are available.
In this section, we propose  stochastic versions
for the PMD and APMD methods to address these issues.

\subsection{Basic stochastic policy mirror descent}
In this subsection, we assume that for a given policy $\pi_k$,
there exists a stochastic estimator $\cQ^{\pi_k,\xi_k}$ s.t.
\begin{align}
 \bbe_{\xi_k} [\cQ^{\pi_k,\xi_k}] &= \bar \cQ^{\pi_k}, \label{eq:expectation_ass} \\
\bbe_{\xi_k}[\|\cQ^{\pi_k,\xi_k} - Q^{\pi_k} \|_\infty^2] &\le \sigma_k^2, \label{eq:variance_ass}\\
\|\bar \cQ^{\pi_k} - Q^{\pi_k}\|_\infty &\le \varsigma_k, \label{eq:bias_ass}
\end{align}
for some $\sigma_k\ge$ and $\varsigma_k \ge 0$,
where $\xi_k$ denotes the random vector used to generate the stochastic
estimator $\cQ^{\pi_k,\xi_k}$.  Clearly, if $\sigma_k = 0$, then
we have exact information about $Q^{\pi_k}$. 
One key insight we have for the stochastic PMD methods
is to handle separately the bias term $\varsigma_k$ from the overall expected error term $\sigma_k$,
because one can reduce the bias term much faster than
the total error.
This makes the analysis of the stochastic PMD method considerably different from
that of the classic stochastic mirror descent method.
While in this section we focus on the convergence analysis of
the algorithms, we will show in next section that such separate treatment of bias and total error enables us to substantially 
improve the sampling complexity for solving RL problems by using policy gradient type methods.

The stochastic policy mirror descent (SPMD)
is obtained by replacing $Q^{\pi_k}$ in \eqnok{eq:PMD_step} 
with its stochastic estimator $\cQ^{\pi_k,\xi_k}$, i.e.,
\beq \label{eq:SPMD_step}
\pi_{k+1}(\cdot|s) = \argmin_{p(\cdot|s) \in \Delta_{|\cA|}} \left\{ \Phi_k(p):=\eta_k [\langle \cQ^{\pi_{k},\xi_k}(s, \cdot), p(\cdot|s) \rangle + h^p(s)]+ D_{\pi_k}^p(s)\right\}.
\eeq
In the sequel, we denote $\xi_{\lceil k \rceil}$ the sequence of random vectors $\xi_0, \ldots, \xi_k$ and define
 \beq \label{eq:def_SPMD_delta}
\delta_k := \cQ^{\pi_k,\xi_k} - Q^{\pi_k}.
\eeq
By using the assumptions in \eqnok{eq:expectation_ass} and \eqnok{eq:bias_ass} and
the decomposition 
\begin{align*}
 \langle \cQ^{\pi_{k},\xi_k}(s, \cdot), \pi_{k}(\cdot|s) - \pi^*(\cdot|s) \rangle
 &= \langle Q^{\pi_k}(s, \cdot),  \pi_{k}(\cdot|s) - \pi^*(\cdot|s) \rangle \\
 & \quad + \langle \bar \cQ^{\pi_k}(s, \cdot) - Q^{\pi_k}(s, \cdot),  \pi_{k}(\cdot|s) - \pi^*(\cdot|s) \rangle \\
 & \quad + \langle  \cQ^{\pi_k,\xi_k}(s, \cdot) - \bar \cQ^{\pi_k}(s, \cdot),  \pi_{k}(\cdot|s) - \pi^*(\cdot|s) \rangle,
\end{align*}
we can see that 
\begin{align}
&\bbe_{\xi_k}[ \langle \cQ^{\pi_{k},\xi_k}(s, \cdot), \pi_{k}(\cdot|s) - \pi^*(\cdot|s) \rangle \mid \xi_{\lceil \xi_{k-1} \rceil}] \nn \\
&\ge  \langle Q^{\pi_k}(s, \cdot),  \pi_{k}(\cdot|s) - \pi^*(\cdot|s) \rangle - 2 \varsigma_k. \label{eq:bnd_bias_SPMD}
\end{align}

Similar to Lemma~\ref{prop:PMD_generic},
below we show some general convergence properties 
about the SPMD method. Unlike PMD, SPMD does
not guarantee the non-increasing property of $V^{\pi_k}(s)$  anymore.

\begin{lemma} \label{prop:SPMD_generic}
For any $s \in \cS$, we have
\begin{align}
V^{\pi_{k+1}}(s) - V^{\pi_k}(s) &\le
\langle \cQ^{\pi_k,\xi_k}(s, \cdot) , \pi_{k+1}(\cdot | s) - \pi_k(\cdot|s) \rangle + h^{\pi_{k+1}}(s) - h^{\pi_{k}}(s) \nn\\
& \quad   +
\tfrac{1}{\eta_k} D_{\pi_k}^{\pi_{k+1}}(s)+ \tfrac{\eta_k \|\delta_k\|_\infty^2}{2(1-\gamma)}.
\label{eq:spmd_decrease2}
\end{align}
\end{lemma}

\begin{proof}
Observe that \eqnok{eq:improvement_in_one_step_temp}
still holds, and hence that
\begin{align}
&V^{\pi_{k+1}}(s) - V^{\pi_k}(s)\nn \\
&=
\tfrac{1}{1-\gamma} \bbe_{s' \sim d_s^{\pi_{k+1}}}
\left[
\langle Q^{\pi_k}(s', \cdot), \pi_{k+1}(\cdot | s') - \pi_k(\cdot|s') \rangle + h^{\pi_{k+1}}(s') - h^{\pi_{k}}(s') \right] \nn \\
&=
\tfrac{1}{1-\gamma} \bbe_{s' \sim d_s^{\pi_{k+1}}}
\left[
\langle \cQ^{\pi_k,\xi_k}(s', \cdot), \pi_{k+1}(\cdot | s') - \pi_k(\cdot|s')\rangle + h^{\pi_{k+1}}(s') - h^{\pi_{k}}(s') \right.\nn \\
& \quad 
\left.  - \langle \delta_k, \pi_{k+1}(\cdot | s') - \pi_k(\cdot|s') \rangle \right] \nn\\
&\le \tfrac{1}{1-\gamma} \bbe_{s' \sim d_s^{\pi_{k+1}}}
\left[
\langle \cQ^{\pi_k,\xi_k}(s', \cdot), \pi_{k+1}(\cdot | s') - \pi_k(\cdot|s') \rangle + h^{\pi_{k+1}}(s') - h^{\pi_{k}}(s') \right.\nn \\
& \quad
\left. +
 \tfrac{1 }{2\eta_k}\|\pi_{k+1}(\cdot | s') - \pi_k(\cdot|s') \|_1^2 + \tfrac{\eta_k \|\delta_k\|_\infty^2}{2}  \right] \nn\\
 &\le \tfrac{1}{1-\gamma} \bbe_{s' \sim d_s^{\pi_{k+1}}}
\left[
\langle \cQ^{\pi_k,\xi_k}(s', \cdot), \pi_{k+1}(\cdot | s') - \pi_k(\cdot|s') \rangle + h^{\pi_{k+1}}(s') - h^{\pi_{k}}(s')\right.\nn \\
& \quad
\left. + \tfrac{1}{\eta_k} D_{\pi_{k}}^{\pi_{k+1}}(s') + \tfrac{\eta_k \|\delta_k\|_\infty^2}{2} \right], \label{eq:improvement_in_one_step_temp_stoch}
\end{align}
where the first inequality follows from Young's inequality and
the second one follows from the strong convexity of $D_{\pi_{k+1}}^{\pi_{k}}$ w.r.t. to $\|\cdot\|_1$.
Moreover, we conclude from Lemma~\ref{lemma:prox_optimality} 
applied to \eqnok{eq:SPMD_step}
with $Q^{\pi_{k}}$ replaced by $\cQ^{\pi_{k},\xi_k}$ and $p(\cdot|s') = \pi_k(\cdot|s')$ that
\begin{align*}
& \langle \cQ^{\pi_{k},\xi_k}(s', \cdot), 
\pi_{k+1}(\cdot|s') - \pi_k(\cdot|s') \rangle + h^{\pi_{k+1}}(s') - h^{\pi_{k}}(s') + \tfrac{1}{\eta_k} D_{\pi_k}^{\pi_{k+1}}(s')\\
&\le  -\tfrac{1}{\eta_k} [(1+ \eta_k \mu) D_{\pi_{k+1}}^{\pi_k}(s') ] \le 0,
\end{align*}
which implies that
\begin{align*}
&\bbe_{s' \sim d_s^{\pi_{k+1}}}
\left[
\langle \cQ^{\pi_k,\xi_k}(s', \cdot), \pi_{k+1}(\cdot | s') - \pi_k(\cdot|s') \rangle 
+ h^{\pi_{k+1}}(s') - h^{\pi_{k}}(s') +  \tfrac{1}{\eta_k} D_{\pi_{k+1}}^{\pi_{k}}(s') \right]\nn \\
& \le d_s^{\pi_{k+1}}(s) \left[
\langle \cQ^{\pi_k,\xi_k}(s, \cdot), \pi_{k+1}(\cdot | s) - \pi_k(\cdot|s) \rangle + h^{\pi_{k+1}}(s) - h^{\pi_{k}}(s) 
 + \tfrac{1}{\eta_k} D_{\pi_k}^{\pi_{k+1}}(s) \right]\nn \\
 &\le (1-\gamma) \left[
\langle \cQ^{\pi_k,\xi_k}(s, \cdot), \pi_{k+1}(\cdot | s) - \pi_k(\cdot|s) \rangle + h^{\pi_{k+1}}(s) - h^{\pi_{k}}(s) 
 + \tfrac{1}{\eta_k} D_{\pi_k}^{\pi_{k+1}}(s) \right], 
 \end{align*}
where the last inequality follows from the fact that $d_s^{\pi_{k+1}}(s) \ge 1-\gamma$ due to the definition of $d_s^{\pi_{k+1}}$
in \eqnok{eq:def_visitation}. The result in \eqnok{eq:spmd_decrease2} 
then follows immediately from \eqnok{eq:improvement_in_one_step_temp_stoch} and the above inequality.
\end{proof}

\vgap

We now establish an important recursion about the SPMD method.

\begin{lemma} \label{prop:SPMD_regularized_MDPs}
For any $k \ge 0$, we have
\begin{align*}
&  \bbe_{\xi_{\lceil k \rceil}} [f(\pi_{k+1}) - f(\pi^*) + (\tfrac{1}{\eta_k}+  \mu) \cD(\pi_{k+1}, \pi^*)] \\
&\le   \bbe_{\xi_{\lceil k-1 \rceil}} [\gamma (f(\pi_k) - f(\pi^*))+ \tfrac{1}{\eta_k} \cD(\pi_k,\pi^*) ]
+ 2 \varsigma_k + \tfrac{ \eta_k \sigma_k^2}{2(1-\gamma)}.
\end{align*}
\end{lemma}

\begin{proof}
By applying Lemma~\ref{lemma:prox_optimality} to \eqnok{eq:PMD_step}
(with $Q^{\pi_{k}}$ replaced by $\cQ^{\pi_{k},\xi_k}$ and $p = \pi^*$), we have
\begin{align*}
&\eta_k[ \langle \cQ^{\pi_{k},\xi_k}(s, \cdot), 
\pi_{k+1}(\cdot|s) - \pi^*(\cdot|s) \rangle 
+ h^{\pi_{k+1}}(s) -  h^{\pi^*}(s) ]
+ D_{\pi_k}^{\pi_{k+1}}(s) \\
&\le D_{\pi_k}^{\pi^*}(s) - (1+\eta_k \mu) D_{\pi_{k+1}}^{\pi^*}(s),
\end{align*}
which, in view of \eqnok{eq:spmd_decrease2}, then implies that
\begin{align*}
&\langle \cQ^{\pi_{k},\xi_k}(s, \cdot) , 
\pi_{k}(\cdot|s) - \pi^*(\cdot|s) \rangle + h^{\pi_{k}}(s) -  h^{\pi^*}(s) +  V^{\pi_{k+1}}(s) - V^{\pi_k}(s) 
 \\
& 
\le \tfrac{1}{\eta_k} D_{\pi_k}^{\pi^*}(s) - (\tfrac{1}{\eta_k} + \mu) D_{\pi_{k+1}}^{\pi^*}(s) + \tfrac{\eta_k \|\delta_k\|_\infty^2}{2(1-\gamma)}.
\end{align*}
Taking expectation w.r.t. $\xi_{\lceil k \rceil}$ and $\nu^*$ on both sides of the above inequality,
and using Lemma~\ref{prop_strong_monotone} and the relation in \eqnok{eq:bnd_bias_SPMD}, 
we arrive at
\begin{align*}
&\bbe_{s\sim \nu^*, \xi_{\lceil k \rceil}}\left[(1- \gamma)(V^{\pi_k}(s) - V^{\pi_\tau^*}(s) ) 
+  V^{\pi_{k+1}}(s) - V^{\pi_k}(s) \right] \\
&\le \bbe_{s\sim \nu^*,\xi_{\lceil k \rceil}}[\tfrac{1}{\eta_k} D_{\pi_k}^{\pi^*}(s) - (\tfrac{1}{\eta_k}+ \mu) D_{\pi_{k+1}}^{\pi^*}(s)] 
+ 2  \varsigma_k + \tfrac{ \eta_k \sigma_k^2}{2(1-\gamma)}.
\end{align*}
Noting $V^{\pi_{k+1}}(s) -  V^{\pi_k}(s)= 
V^{\pi_{k+1}}(s) - V^{\pi^*}(s) - [V^{\pi_k}(s) - V^{\pi^*}(s) ] $,
rearranging the terms in the above inequality, and using the definition of $f$ in \eqnok{eq:MDP_OPT}, we arrive at the result.
\end{proof}

We are now ready to establish the convergence rate of the SPMD method.  
We start with the case when $\mu > 0$ and state a constant stepsize rule which requires both
$\varsigma_k$ and $\sigma_k$, $k \ge 0$, to be small enough to guarantee the convergence of 
the SPMD method. 

\begin{theorem} \label{theorem:SPMD_reg}
Suppose that 
$\eta_k = \eta = \tfrac{1-\gamma}{\gamma \mu}$
in the SPMD method. 
If $\varsigma_k =2^{-(\lfloor k/l \rfloor+2)}$ and $\sigma_k^2 = 2^{-(\lfloor k/l \rfloor+2)}$ 
for any $k \ge 0$ with $l :=\left \lceil  \log_\gamma (1/4) \right \rceil
$,
then 
\begin{align}
& \bbe_{\xi_{\lceil k-1 \rceil}} [f(\pi_{k}) - f(\pi^*) 
+ \tfrac{\mu}{1-\gamma} \cD(\pi_{k},\pi^*)]\nn \\
&\le 2^{-\lfloor k/l \rfloor}\left [f(\pi_0) - f(\pi^*) +\tfrac{1}{1-\gamma}( \mu  \log |\cA|
+ \tfrac{5} {2} + \tfrac{ 5}{8 \gamma \mu})\right]. \label{eq:SPMD_reg_MDP_results_uniform_constant1}
\end{align}
\end{theorem}

\begin{proof}
By Lemma~\ref{prop:SPMD_regularized_MDPs} and the selection of $\eta$, we have
\begin{align*}
& \bbe_{\xi_{\lceil k \rceil}} [f(\pi_{k+1}) - f(\pi^*) + \tfrac{\mu}{1-\gamma}  \cD(\pi_{k+1},\pi^*)] \\
&\le  \gamma [ \bbe_{\xi_{\lceil k-1 \rceil}} [f(\pi_k) - f(\pi^*) + \tfrac{\mu}{1-\gamma} \cD(\pi_k, \pi^*) ]
+ 2 \varsigma_k + \tfrac{ \sigma_k^2}{2 \gamma \mu},
\end{align*}
which, in view of Lemma~\ref{lemma:induction_tech_sequence}
with $X_k =  \bbe_{\xi_{\lceil k-1 \rceil}} [f(\pi_k) - f(\pi^*) + \tfrac{\mu}{1-\gamma} \cD(\pi_k, \pi^*)$
and $Z_k = 2 \varsigma_k + \tfrac{ \sigma_k^2}{2 \gamma \mu}$,
then implies that
\begin{align*}
&\bbe_{\xi_{\lceil k-1 \rceil}} [f(\pi_k) - f(\pi^*) + \tfrac{\mu}{1-\gamma} \cD(\pi_k, \pi^*)\\
&\le \gamma^{\lfloor k/l \rfloor}\left [f(\pi_0) - f(\pi^*) + \tfrac{\mu  \cD(\pi_0, \pi^*)}{1-\gamma}
+ \tfrac{5} {4}( \tfrac{2}{1-\gamma}+ \tfrac{ 1}{2 \gamma(1 -\gamma) \mu})\right]\\
&\le \gamma^{\lfloor k/l \rfloor}\left [f(\pi_0) - f(\pi^*) +\tfrac{1}{1-\gamma}( \mu  \log |\cA|
+ \tfrac{5} {2} + \tfrac{ 5}{8 \gamma \mu})\right].
\end{align*}
\end{proof}

\vgap

We now turn our attention to the convergence properties of the SPMD method for the case when $\mu = 0$.

\begin{theorem} \label{the_SPMD_unreg}
Suppose that 
$\eta_k = \eta$
for any $k \ge 0$ in the SPMD method. 
 If $\varsigma_k \le \varsigma$ and $\sigma_k \le \sigma$ for any $k \ge 0$,
then we have
\begin{align}
 \bbe_{\xi_{\lceil k \rceil}, R}[ f(\pi_{R}) - f(\pi^*)]
&\le \tfrac{\gamma [f(\pi_0) - f(\pi^*)]}{(1- \gamma)k} + \tfrac{\log |\cA| }{\eta(1- \gamma)k} + \tfrac{2 \varsigma}{1- \gamma} + \tfrac{ \eta \sigma^2}{2(1-\gamma)^2}, \label{eq:SPMD_unreg}
\end{align}
where $R$ denotes a random number uniformly distributed between $1$ and $k$.
In particular, if the
number of iterations $k$ is given a priori and $\eta = (\tfrac{2(1-\gamma) \log |\cA|}{k \sigma^2})^{1/2}$,
then 
\begin{align}
 \bbe_{\xi_{\lceil k \rceil}, R}[ f(\pi_{R}) - f(\pi^*)]
&\le \tfrac{\gamma [f(\pi_0) - f(\pi^*)]}{(1- \gamma)k} 
 + \tfrac{2 \varsigma}{1- \gamma} + \tfrac{ \sigma \sqrt{2 \log |\cA|}}{(1-\gamma)^{3/2}\sqrt{k}}. \label{eq:SPMD_unreg_refined}
\end{align}
\end{theorem}

\begin{proof}
By Lemma~\ref{prop:SPMD_regularized_MDPs} and the fact that $\mu = 0$, we have
\begin{align*}
&  \bbe_{\xi_{\lceil k \rceil}} [f(\pi_{k+1}) - f(\pi^*) + \tfrac{1}{\eta} \cD(\pi_{k+1}, \pi^*)] \\
&\le   \bbe_{\xi_{\lceil k-1 \rceil}} [\gamma (f(\pi_k) - f(\pi^*))+ \tfrac{1}{\eta} \cD(\pi_k,\pi^*) ]
+ 2 \varsigma_k + \tfrac{ \eta \sigma_k^2}{2(1-\gamma)}.
\end{align*}
Taking the telescopic sum of the above relations, we have
\begin{align*}
(1- \gamma) \tsum_{i = 1}^{k} \bbe_{\xi_{\lceil k \rceil}}[ f(\pi_{i}) - f(\pi^*)]
\le  [\gamma (f(\pi_0) - f(\pi^*))+ \tfrac{1}{\eta} \cD(\pi_0,\pi^*) ]
+ 2 k \varsigma + \tfrac{ k \eta \sigma^2}{2(1-\gamma)}.
\end{align*}
Dividing both sides by $(1-\gamma) k$ and using the definition of $R$, we obtain the result
in \eqnok{eq:SPMD_unreg}.
\end{proof}

We add some remarks about the results in Theorem~\ref{the_SPMD_unreg}.
In comparison with the convergence results of SPMD for the case $\mu > 0$,
there exist some possible shortcomings for the case when $\mu = 0$. Firstly,
one needs to output a randomly selected $\pi_{R}$ from the trajectory.
Secondly, since the first term in \eqnok{eq:SPMD_unreg_refined}
converges sublinearly, one has to update $\pi_{k+1}$
at least ${\cal O}(1/\epsilon)$ times, which may also impact the gradient complexity
of computing $\nabla h^\pi$ if $\pi_{k+1}$ cannot be computed explicitly. 
We will address these issues by developing the stochastic APMD method in next subsection.

\subsection{Stochastic approximate policy mirror descent}
The stochastic approximate policy mirror descent (SAPMD) method is obtained by replacing $Q_{\tau_k}^{\pi_k}$ in \eqnok{eq:PMD_step_ada} 
with its stochastic estimator $\cQ_{\tau_k}^{\pi_k,\xi_k}$. As such, its updating formula is given by
\beq \label{eq:SPMD_step_ada}
\pi_{k+1}(\cdot|s) = \argmin_{p(\cdot|s) \in \Delta_{|\cA|}} \left\{ \eta_k[ \langle \cQ_{\tau_k}^{\pi_{k},\xi_k}(s, \cdot), p(\cdot|s) \rangle + h^{p}(s)
+ \tau_k D_{\pi_0}^\pi(s_t)] + D_{\pi_k}^p(s)\right\}.
\eeq
With a little abuse of notation, we still denote $\delta_k := \cQ_{\tau_k}^{\pi_k,\xi_k} - Q_{\tau_k}^{\pi_k}$ and assume that 
\begin{align}
 \bbe_{\xi_k} [\cQ_{\tau_k}^{\pi_k,\xi_k}] &= \bar \cQ_{\tau_k}^{\pi_k}, \label{eq:expectation_ass_ada} \\
\bbe_{\xi_k}[\|\cQ_{\tau_k}^{\pi_k,\xi_k} - Q_{\tau_k}^{\pi_k} \|_\infty^2] &\le \sigma_k^2, \label{eq:variance_ass_ada}\\
\|\bar \cQ_{\tau_k}^{\pi_k} - Q_{\tau_k}^{\pi_k}\|_\infty &\le \varsigma_k, \label{eq:bias_ass_ada}
\end{align}
for some $\sigma_k\ge$ and $\varsigma_k \ge 0$.
Similarly to \eqnok{eq:bnd_bias_SPMD} we have
\begin{align}
&\bbe_{\xi_k}[ \langle \cQ_{\tau_k}^{\pi_{k},\xi_k}(s, \cdot), \pi_{k}(\cdot|s) - \pi^*(\cdot|s) \rangle \mid \xi_{\lceil \xi_{k-1} \rceil}] \nn \\
&\ge  \langle Q_{\tau_k}^{\pi_k}(s, \cdot),  \pi_{k}(\cdot|s) - \pi^*(\cdot|s) \rangle - 2 \varsigma_k. \label{eq:bnd_bias_SPMD_ada}
\end{align}

Lemma~\ref{prop:srpmd_decrease} and Lemma~\ref{prop:SRPMD_un_reg_semi_final} below show the improvement 
for each SAPMD iteration.
 
\begin{lemma} \label{prop:srpmd_decrease}
For any $k \ge 0$, we have
\begin{align}
V_{\tau_k}^{\pi_{k+1}}(s) - V_{\tau_k}^{\pi_k}(s) &\le
\langle \cQ_{\tau_k}^{\pi_k,\xi_k}(s, \cdot) , \pi_{k+1}(\cdot | s) - \pi_k(\cdot|s) \rangle + h^{\pi_{k+1}}(s) - h^{\pi_{k}}(s) \nn\\
& \quad   +  \tau_k[D_{\pi_0}^{\pi_{k+1}}(s) - D_{\pi_0}^{\pi^{*}}(s)] +
\tfrac{1}{\eta_k} D_{\pi_k}^{\pi_{k+1}}(s)+ \tfrac{\eta_k \|\delta_k\|_\infty^2}{2(1-\gamma)}.
\label{eq:srpmd_decrease2}
\end{align}
\end{lemma}
\begin{proof}
The proof is similar to the one for Lemma~\ref{prop:SPMD_generic} except that we will apply Lemma~\ref{lemma_per_diff} 
to the perturbed value functions  $V_{\tau_k}^\pi$ instead of $V^\pi$.
\end{proof}

\begin{lemma} \label{prop:SRPMD_un_reg_semi_final}
If $1+\eta_k \tau_k = 1/\gamma$ and $\tau_k \ge \tau_{k+1}$ in the SAPMD method, then for any
$k \ge 0$,
\begin{align}
& \bbe_{s\sim \nu^*,\xi_{\lceil k \rceil}}[V_{\tau_{k+1}}^{\pi_{k+1}}(s) - V_{\tau_{k+1}}^{\pi^*}(s)  + \tfrac{\tau_k}{1 - \gamma }  D_{\pi_{k+1}}^{\pi^*}(s)]\nn\\
&\le   \bbe_{s\sim \nu^*,\xi_{\lceil k-1 \rceil}}[ \gamma[ V_{\tau_k}^{\pi_k}(s) - V_{\tau_k}^{\pi^*}(s) +
\tfrac{\tau_k}{1 - \gamma}D_{\pi_k}^{\pi^*}(s)] \nn\\
&\quad + \tfrac{\tau_{k}-\tau_{k+1}}{1-\gamma}  \log|\cA| +2\varsigma_k + \tfrac{\sigma_k^2}{2 \gamma \tau_k}.  \label{eq:srpmd_temp}
\end{align}
\end{lemma}
\begin{proof}
By Lemma~\ref{lemma:prox_optimality_ada} 
with $p = \pi^*$ and $Q_{\tau_k}^{\pi_{k}}$ replaced by $\cQ_{\tau_k}^{\pi_{k},\xi_k}$, we have
\begin{align*}
& \langle \cQ_{\tau_k}^{\pi_{k},\xi_k}(s, \cdot), 
\pi_{k+1}(\cdot|s) - \pi^*(\cdot|s) \rangle + h^{\pi_{k+1}}(s) - h^{\pi^*}(s)\\
&+\tau_k [ D_{\pi_0}^{\pi_{k+1}}(s) - D_{\pi_0}^{\pi^{*}}(s)] 
+ \tfrac{1}{\eta_k} D_{\pi_k}^{\pi_{k+1}}(s) \\
&\le \tfrac{1}{\eta_k} D_{\pi_k}^{\pi^*}(s) - (\tfrac{1}{\eta_k}+  \tau_k ) D_{\pi_{k+1}}^{\pi^*}(s),
\end{align*}
which, in view of \eqnok{eq:srpmd_decrease2}, implies that
\begin{align*}
&\langle \cQ_{\tau_k}^{\pi_{k},\xi_k}(s, \cdot) , 
\pi_{k}(\cdot|s) - \pi^*(\cdot|s) \rangle + h^{\pi_{k}}(s) - h^{\pi^*}(s) +  \tau_k [ D_{\pi_0}^{\pi_{k}}(s) - D_{\pi_0}^{\pi^{*}}(s)] \\
&+  V_{\tau_k}^{\pi_{k+1}}(s) - V_{\tau_k}^{\pi_k}(s) 
\le \tfrac{1}{\eta_k} D_{\pi_k}^{\pi^*}(s) - (\tfrac{1}{\eta_k}+  \tau_k ) D_{\pi_{k+1}}^{\pi^*}(s)+  \tfrac{\eta_k \|\delta_k\|_\infty^2}{2(1-\gamma)}.
\end{align*}
Taking expectation w.r.t. $\xi_{\lceil k \rceil}$ and $\nu^*$ on both sides of the above inequality,
and using Lemma~\ref{prop_strong_monotone_APMD} and the relation in \eqnok{eq:bnd_bias_SPMD_ada}, 
we arrive at
\begin{align*}
&\bbe_{s\sim \nu^*, \xi_{\lceil k \rceil}}[(1- \gamma)(V_{\tau_k}^{\pi_k}(s) - V_{\tau_k}^{\pi^*}(s) )]
+  \bbe_{s\sim \nu^*, \xi_{\lceil k \rceil}}[V_{\tau_k}^{\pi_{k+1}}(s) - V_{\tau_k}^{\pi_k}(s) ] \\
&
\le \bbe_{s\sim \nu^*, \xi_{\lceil k \rceil}}[\tfrac{1}{\eta_k}D_{\pi_k}^{\pi^*}(s) - (\tfrac{1}{\eta_k}+  \tau_k)D_{\pi_{k+1}}^{\pi^*}(s)]
+ 2 \varsigma_k+ \tfrac{\eta_k \sigma_k^2}{2(1-\gamma)} .
\end{align*}
Noting $V_{\tau_k}^{\pi_{k+1}}(s) -  V_{\tau_k}^{\pi_k}(s)= 
V_{\tau_k}^{\pi_{k+1}}(s) - V_{\tau_k}^{\pi^*}(s) - [V_{\tau_k}^{\pi_k}(s) - V_{\tau_k}^{\pi^*}(s) ] $
and rearranging the terms in the above inequality, we have
\begin{align}
& \bbe_{s\sim \nu^*, \xi_{\lceil k \rceil}}[V_{\tau_k}^{\pi_{k+1}}(s) - V_{\tau_k}^{\pi^*}(s) + 
(\tfrac{1}{\eta_k}+ \tau_k)D_{\pi_{k+1}}^{\pi^*}(s)]  \nn\\
&\le \gamma \bbe_{s\sim \nu^*, \xi_{\lceil k-1 \rceil}}[V_{\tau_k}^{\pi_k}(s) - V_{\tau_k}^{\pi^*}(s)] 
+ \bbe_{s\sim \nu^*, \xi_{\lceil k-1 \rceil}}[\tfrac{1}{\eta_k}D_{\pi_k}^{\pi^*}(s) ] + 2 \zeta_k+ \tfrac{\eta_k \sigma_k^2}{2(1-\gamma)}, \nn
\end{align}
which, in view of the assumption $\tau_k \ge \tau_{k+1}$ and \eqnok{eq:closeness_perturbation_direct}, then implies that
\begin{align}
V_{\tau_k}^{\pi_{k+1}}(s) - V_{\tau_k}^{\pi_k}(s) &\le
\langle \cQ_{\tau_k}^{\pi_k,\xi_k}(s, \cdot) , \pi_{k+1}(\cdot | s) - \pi_k(\cdot|s) \rangle + h^{\pi_{k+1}}(s) - h^{\pi_{k}}(s) \nn\\
& \quad   +  \tau_k[D_{\pi_0}^{\pi_{k+1}}(s) - D_{\pi_0}^{\pi^{*}}(s)] +
\tfrac{1}{\eta_k} D_{\pi_k}^{\pi_{k+1}}(s)+ \tfrac{\eta_k \|\delta_k\|_\infty^2}{2(1-\gamma)}.
\end{align}
The result then immediately follows from the assumption that $1+\eta_k \tau_k = 1/\gamma$.
\end{proof}

\vgap

We are now ready to establish the convergence of the SAPMD method.

\begin{theorem} \label{theorem_srpmd_un_reg_optimal}
Suppose that $\eta_k = \tfrac{1-\gamma}{\gamma \tau_k}$
in the SAPMD method. If $
\tau_k = \tfrac{1}{\sqrt{ \gamma \log |\cA|}} 2^{-(\lfloor k/l \rfloor + 1)}$,
$\varsigma_k = 2^{-(\lfloor k/l \rfloor + 2)}$, and 
$\sigma_k^2 = 4^{-(\lfloor k/l \rfloor +2)}$ with
$l :=\left \lceil  \log_\gamma (1/4) \right \rceil
$,
then 
\begin{align} \label{eq:srpmd_un_reg_optimal}
& \bbe_{\xi_{\lceil k-1 \rceil}} [f(\pi_{k}) - f(\pi^*)] \le 2^{-\lfloor k/l \rfloor} [f(\pi_0) - f(\pi^*) 
+ \tfrac{3\sqrt{\log|\cA|}}{(1-\gamma)\sqrt{\gamma}} +\tfrac{5}{2(1-\gamma)} ]. 
\end{align}
\end{theorem}

\begin{proof}
By Lemma~\ref{prop:SRPMD_un_reg_semi_final} and the selection of $\tau_k, \varsigma_k$ and $\sigma_k$,
we have
\begin{align}
& \bbe_{s\sim \nu^*,\xi_{\lceil k \rceil}}[V_{\tau_{k+1}}^{\pi_{k+1}}(s) - V_{\tau_{k+1}}^{\pi^*}(s)  + \tfrac{\tau_k}{1 - \gamma }  D_{\pi_{k+1}}^{\pi^*}(s)]\nn\\
&\le   \bbe_{s\sim \nu^*,\xi_{\lceil k-1 \rceil}}[ \gamma[ V_{\tau_k}^{\pi_k}(s) - V_{\tau_k}^{\pi^*}(s) +
\tfrac{\tau_k}{1 - \gamma}D_{\pi_k}^{\pi^*}(s)] ]\nn\\
&\quad + \tfrac{\tau_{k}-\tau_{k+1}}{1-\gamma}  \log|\cA| + (2 + \tfrac{\sqrt{ \log |\cA|}}{2\sqrt{\gamma}})  2^{-(\lfloor k/l \rfloor + 2)}.  
\end{align}
Using the above inequality and Lemma~\ref{lemma:induction_tech_sequence} (with $X_k = \bbe_{s\sim \nu^*,\xi_{\lceil k-1 \rceil}}[ \gamma[ V_{\tau_k}^{\pi_k}(s) - V_{\tau_k}^{\pi^*}(s) +
\tfrac{\tau_k}{1 - \gamma}D_{\pi_k}^{\pi^*}(s)] ]$, $Y_k = \tfrac{\tau_{k}}{1-\gamma} \log|\cA|$
and $Z_k =(2 + \tfrac{\sqrt{ \log |\cA|}}{2\sqrt{\gamma}})  2^{-(\lfloor k/l \rfloor + 2)}$),
we conclude
\begin{align*}
&\bbe_{s\sim \nu^*,\xi_{\lceil k-1 \rceil}}[ V_{\tau_k}^{\pi_k}(s) - V_{\tau_k}^{\pi^*}(s) +
\tfrac{\tau_k}{1 - \gamma}D_{\pi_k}^{\pi^*}(s)] \\
& \le 2^{-\lfloor k/l \rfloor} \{\bbe_{s\sim \nu^*}[ V_{\tau_0}^{\pi_k}(s) - V_{\tau_0}^{\pi^*}(s) +
\tfrac{\sqrt{\log |\cA|}}{2 (1 - \gamma) \sqrt{\gamma}}]  + \tfrac{ \sqrt{\log|\cA|}}{(1-\gamma)\sqrt{\gamma}} +\tfrac{5}{2(1-\gamma)}+ \tfrac{5\sqrt{ \log |\cA|}}{8(1-\gamma) \sqrt{\gamma}} \}\\
&= 2^{-\lfloor k/l \rfloor} \{\bbe_{s\sim \nu^*}[ V_{\tau_0}^{\pi_k}(s) - V_{\tau_0}^{\pi^*}(s)]  + \tfrac{17 \sqrt{\log|\cA|}}{8(1-\gamma)\sqrt{\gamma}} +\tfrac{5}{2(1-\gamma)} \}.
\end{align*}
Noting that  $V_{\tau_{k}}^{\pi_{k}}(s) \ge V^{\pi_{k}}(s)$,
$V_{\tau_{k}}^{\pi^*}(s) \le V^{\pi^*}(s) + \tfrac{\tau_k}{1-\gamma}\log |\cA|$, 
$V_{\tau_0}^{\pi^*}(s) \ge V^{\pi^*}(s)$ due to  \eqnok{eq:closeness_perturbation}, and that
$V_{\tau_0}^{\pi_0}(s) = V^{\pi_0}(s)$ due to $D_{\pi_0}^{\pi_0}(s) = 0$, we conclude from the previous inequality
and the definition of $\tau_k$ that
\begin{align*}
&\bbe_{s\sim \nu^*,\xi_{\lceil k-1 \rceil}}[ V^{\pi_k}(s) - V^{\pi^*}(s) ]
\le 2^{-\lfloor k/l \rfloor} \{\bbe_{s\sim \nu^*}[ V^{\pi_0}(s) - V^{\pi^*}(s) + \tfrac{3\sqrt{\log|\cA|}}{(1-\gamma)\sqrt{\gamma}}+\tfrac{5}{2(1-\gamma)} \},
\end{align*}
from which the result immediately follows.
\end{proof}


A few remarks about the convergence of the SAPMD method are in place.

First, in view of Theorem~\ref{theorem_srpmd_un_reg_optimal}, the SAPMD method
does not need to randomly output a solution as most existing nonconvex stochastic gradient descent methods did.
Instead, the linear rate of convergence in \eqnok{eq:srpmd_un_reg_optimal} has
been established for the last iterate $\pi_k$ generated by this algorithm. 
The convergence for the last iterate indicates that the SAMPD method will continuously improve the
policy deployed by the system for implementation and evaluation. This is not the case for the convergence of the average or 
random iterate, since the average iterate will not be implemented and evaluated, and the convergence of the random iterate 
does not warrant continuous improvement of the generated policies. 

Second, both Theorems~\ref{theorem:SPMD_reg} and~\ref{theorem_srpmd_un_reg_optimal}
allow us to establish some strong large-deviation properties associated with the convergence of
SPMD and SAPMD. Let us focus on the SAPMD method. For a given confidence level $\lambda \in (0,1)$ and accuracy level $\epsilon > 0$,
if the number of iterations $k$ satisfies 
\[
\lfloor k/l \rfloor \ge \log_2 \left\{\tfrac{1}{\lambda \epsilon}\left[f(\pi_0) - f(\pi^*) 
+ \tfrac{3\sqrt{\log|\cA|}}{(1-\gamma)\sqrt{\gamma}} +\tfrac{5}{2(1-\gamma)}\right] \right\},
\]
then by \eqnok{eq:srpmd_un_reg_optimal} and Markov's inequality,
we have
\begin{align*}
\prob\{f(\pi_{k}) - f(\pi^*) > \epsilon\} \le \tfrac{1}{\epsilon} 2^{-\lfloor k/l \rfloor} [f(\pi_0) - f(\pi^*) 
+ \tfrac{3\sqrt{\log|\cA|}}{(1-\gamma)\sqrt{\gamma}} +\tfrac{5}{2(1-\gamma)} ] \le \lambda. 
\end{align*} 
In other words, with probability greater than $1-\lambda$, we have $f(\pi_{k}) - f(\pi^*) \le \epsilon$.
On the other hand, it is more difficult to derive a similar large deviation result for SPMD
directly applied to unregularized problems  (c.f. Theorem~\ref{the_SPMD_unreg}).
Due to the sublinear rate of convergence and
random selection of output, we need to run the algorithm for a few times
to general several candidate solutions and apply a post-optimization procedure to choose from these candidate solutions
in order to improve the the reliability  of the algorithm (see 
Chapter 6 of \cite{LanBook2020} for more discussions).


\section{Stochastic Estimation for Action-value Functions} \label{sec:sampling_complexity}
In this section, we discuss the estimation of the action-value functions $Q^\pi$ or $Q^\pi_\tau$
through two different approaches. In Subsection~\ref{sec_mult_trajectory}, we assume the existence of
a generative model for the Markov Chains so that we can estimate value functions by generating
multiple independent trajectories starting from an arbitrary pair of state and action.
In Subsection~\ref{sec_single_trajectory}, we consider a more challenging setting where we only have access to 
a single trajectory observed when the dynamic system runs online.
In this case, we employ and enhance the conditional temporal difference (CTD) method recently
developed in \cite{KotsalisLanLi2020PartII} to estimate value functions.
Throughout the section we assume that 
\begin{align}
c(s, a) &\le \bar c, \forall (s, a) \in \cS \times \cA, \label{eq:bnd_c}\\
h^\pi(s) & \le \bar h, \forall s \in \cS, \pi \in \Delta_{|\cA|}. \label{eq:bnd_h}
\end{align}

\subsection{Multiple independent trajectories} \label{sec_mult_trajectory}

In the multiple trajectory setting, starting from state-action pair $(s,a)$ and following policy $\pi_k$, 
we can generate $M_k$ independent trajectories of length $T_k$, denoted by 
\[
\zeta_k^i \equiv \zeta_k^i(s,a) := \{(s_0^i= s, a_0^i =a); (s_1^i, a_1^i), \ldots, (s_{T_k-1}^i, a_{T_k-1}^i)\}, i = 1, \ldots, M_k.
\]
Let $\xi_k := \{\zeta_k^i(s,a), i = 1, \ldots, M_k, s\in \cS, a \in \cA\}$ denote all these random variables. 
We can estimate $Q^{\pi_k}$ in the SPMD method by
\begin{align*}
\cQ^{\pi_k,\xi_k}(s,a) &= \tfrac{1}{M_k} \tsum_{i=1}^{M_k} \tsum_{t=0}^{T_k-1} \gamma^t [c(s_t^i, a_t^i) + h^{\pi_k}(s_t^i)].
\end{align*}
We can show that $\cQ^{\pi_k,\xi_k}$ satisfy \eqnok{eq:expectation_ass}-\eqnok{eq:bias_ass} with
\begin{align}
\varsigma_k = \tfrac{(\bar c + \bar h) \gamma^{T_k}}{1 - \gamma}
\ \ \mbox{and} \ \
\sigma_k^2 
= \tfrac{(\overline{c} + \overline{h})^2}{(1-\gamma)^2}
\left[ \gamma^{2T_k} + \tfrac{\kappa (\log (\abs{\cS} \abs{\cA}) + 1)}{M_k} \right],
\label{eq:TM_SPMD}
\end{align}
for some absolute constant $\kappa > 0$ (see Proposition~\ref{prop:bound_inf_var}
in the Appendix).
By choosing $T_k$ and $M_k$ properly, we can show the convergence of the SPMD
method employed with different stepsize rules as stated in Theorems~\ref{theorem:SPMD_reg} 
and~\ref{the_SPMD_unreg}. 

\begin{proposition}
Suppose that 
$\eta_k = \tfrac{1-\gamma}{\gamma \mu}$
in the SPMD method. 
If $T_k$ and $M_k$ are chosen such that
\[
T_k \ge \tfrac{l}{2} (\lfloor k /l \rfloor + \log_2 \tfrac{\bar c + \bar h}{1-\gamma} +2 ) \ \ \mbox{and} \ \
M_k \ge  \tfrac{(\bar c+ \bar h)^2 \kappa (\log (\abs{\cS} \abs{\cA}) + 1)}{(1-\gamma)^2} 2^{\lfloor k/l\rfloor + 4}
\]
with $l :=\left \lceil  \log_\gamma (1/4) \right \rceil$, then the relation in \eqnok{eq:SPMD_reg_MDP_results_uniform_constant1}
holds. As a consequence, an $\epsilon$-solution of \eqnok{eq:MDP_OPT}, i.e., a solution $\bar \pi$ s.t.
$\bbe [f(\bar \pi) - f(\pi^*) + \tfrac{\mu}{1-\gamma} \cD(\bar \pi,\pi^*)] \le \epsilon$, can be found 
in at most ${\cal O}(\log_\gamma \epsilon)$
SPMD iterations. In addition, the total number of samples for $(s_t,a_t)$ pairs can be bounded by
\beq \label{eq:sample_SPMD}
{\cal O} (\tfrac{|\cS| |\cA|  \log |\cA|  \log (|\cS| |\cA|)  \log_\gamma (1/2) \log_\gamma \epsilon  }{\mu (1-\gamma)^3 \epsilon}).
\eeq
\end{proposition}

\begin{proof}
Using the fact that $\gamma^l \le 1/4$, we can easily check 
from \eqnok{eq:TM_SPMD} and the selection of $T_k$ and $M_k$  that
\eqnok{eq:expectation_ass}-\eqnok{eq:bias_ass}  hold with
$\varsigma_k =2^{-(\lfloor k/l \rfloor+2)}$ and $\sigma_k^2 = 2^{-(\lfloor k/l \rfloor+2)}$.
Suppose that an $\epsilon$-solution $\bar \pi$ will be found at the $\bar k$ iteration. By
\eqnok{eq:SPMD_reg_MDP_results_uniform_constant1}, we have
\[
\lfloor \bar k/l \rfloor \le \log_2\{  [f(\pi_0) - f(\pi^*) +\tfrac{1}{1-\gamma}( \mu  \log |\cA|
+ \tfrac{5} {2} + \tfrac{ 5}{8 \gamma \mu})] \epsilon^{-1}\},
\]
which implies that the number of iterations is bounded by ${\cal O}( l \lfloor \bar k/l \rfloor) = {\cal O}(\log_\gamma \epsilon)$.
Moreover by the definition of $T_k$ and $M_k$, the total number of samples is bounded by
\begin{align*}
 &|\cS| |\cA| l \tsum_{p=0}^{\lfloor \bar k/l\rfloor}  
[\tfrac{l}{2} (p + \log_2 \tfrac{\bar c + \bar h}{1-\gamma} +2 ) \tfrac{(\bar c+ \bar h)^2}{(1-\gamma)^2} 2^{p + 4}] \\
 &= {\cal O} \{|\cS| |\cA| l^2 (\lfloor \bar k/l\rfloor + \log_2 \tfrac{\bar c + \bar h}{1-\gamma}) \tfrac{(\bar c+ \bar h)^2 \kappa (\log (\abs{\cS} \abs{\cA}) + 1)}{(1-\gamma)^2} 2^{\lfloor \bar k/l\rfloor }\} 
 =  {\cal O}  (\tfrac{|\cS| |\cA|  \log |\cA|  \log (|\cS| |\cA|)  \log_\gamma(1/2) \log_\gamma \epsilon}{\mu (1-\gamma)^3 \epsilon}).
\end{align*}
\end{proof}

To the best of our knowledge, this is the first time in the literature that an ${\cal O}(\log (1/\epsilon)/\epsilon)$
sampling complexity, after disregarding all constant factors, 
has been obtained for solving RL problems with strongly convex regularizers, 
even though problem~\eqnok{eq:MDP_OPT} is still nonconvex.
The previously best-known sampling complexity for RL problems
with entropy regularizer was $\tilde {\cal O}(|\cS| |\cA|^2/\epsilon^3)$~\cite{DBLP:conf/aaai/ShaniEM20},
and the author was not aware of an $\tilde {\cal O}(1/\epsilon)$ sampling complexity results
for any RL problems.


\vgap

Below we discuss the sampling complexities of SPMD and SAPMD 
for solving RL problems with general convex regularizers. 

\begin{proposition} \label{prop:SPMD_un}
Consider the general RL problems with $\mu = 0$. Suppose that the
number of iterations $k$ is given a priori and $\eta_k = (\tfrac{2(1-\gamma) \log |\cA|}{k \sigma^2})^{1/2}$.
If $T_k \ge T \equiv \log_\gamma \tfrac{(1-\gamma) \epsilon}{3(\bar c + \bar h)}$
and $M_k = 1$, then an $\epsilon$-solution of
problem of \eqnok{eq:MDP_OPT}, i.e., a solution $\bar \pi$ s.t.
$\bbe [f(\bar \pi) - f(\pi^*) ] \le \epsilon$, can be found 
in at most ${\cal O}(\log |\cA|/ [(1-\gamma)^5 \epsilon^2])$
SPMD iterations. In addition, the total number of state-action samples can be bounded by
\beq \label{eq:sample_SPMD_general}
{\cal O} (\tfrac{|\cS| |\cA| \log |\cA| \log_\gamma \epsilon}{(1-\gamma)^5 \epsilon^2}).
\eeq
\end{proposition}

\begin{proof}
We can easily check 
from \eqnok{eq:TM_SPMD} and the selection of $T_k$ and $M_k$  that
\eqnok{eq:expectation_ass}-\eqnok{eq:bias_ass}  holds with
$\varsigma_k =\epsilon/3$ and $\sigma_k^2 = 2 (\tfrac{\epsilon^2}{3^2} + \tfrac{2(\bar c+ \bar h)^2}{(1-\gamma)^2})$.
Using these bounds in \eqnok{eq:SPMD_unreg}, we conclude that
an $\epsilon$-solution will be found in at most
\beq \label{eq:pmd_basic_iter}
\bar k = \tfrac{4 [(\epsilon/3)^2 + (\bar c + \bar h)^2/(1-\gamma)^2)]\log |\cA|}{(1-\gamma)^3 (\epsilon/3)^2} + \tfrac{\gamma [f(\pi_0) - f(\pi^*)]}{(1-\gamma)(\epsilon/3)}
\eeq
iterations.
Moreover, the total number of samples is bounded by $|\cS| |\cA| T \bar k$ and hence by \eqnok{eq:sample_SPMD_general}.
\end{proof}

\vgap

We can also establish the iteration and sampling complexities of the SAPMD method,
in which we estimate $Q_{\tau_k}^{\pi_k}$ by
\begin{align*}
\cQ_{\tau_k}^{\pi_k,\xi_k}(s,a) &= \tfrac{1}{M_k} \tsum_{i=1}^{M_k} \tsum_{t=0}^{T_k-1} \gamma^t [c(s_t^i, a_t^i) + h^{\pi_k}(s_t^i) +\tau_k D_{\pi_0}^{\pi_k}(s_t^i)].
\end{align*}
Since $\tau_0 \ge \tau_k$, similar to \eqnok{eq:TM_SPMD},
we can show that $\cQ_{\tau_k}^{\pi_k,\xi_k}$ satisfy \eqnok{eq:expectation_ass_ada}-\eqnok{eq:bias_ass_ada} 
with
\begin{align}
\varsigma_k = \tfrac{(\bar c + \bar h + \tau_0 \log |\cA| ) \gamma^{T_k}}{1 - \gamma}
\ \ \mbox{and} \ \
\sigma_k^2 
=  \tfrac{2 (\bar c + \bar h + \tau_0 \log |\cA|)^2}{(1-\gamma)^2} (\gamma^{2 T_k} + \tfrac{\kappa (\log (\abs{\cS} \abs{\cA}) + 1)}{M_k}) \label{eq:TM_SAPMD}
\end{align}
for some absolute constant $\kappa > 0$.

\begin{proposition} \label{prop:SAPMD_un}
Suppose that 
$\eta_k = \tfrac{1-\gamma}{\gamma \tau_k}$ and $
\tau_k = \tfrac{1}{\sqrt{ \gamma \log |\cA|}} 2^{-(\lfloor k/l \rfloor + 1)}$
in the SAPMD method. 
If $T_k$ and $M_k$ are chosen such that
\[
T_k \ge \tfrac{l}{2} (\lfloor k /l \rfloor + \log_2 \tfrac{\bar c + \bar h + \tau_0 \log |\cA|}{1-\gamma} +4 ) \ \ \mbox{and} \ \
M_k \ge  \tfrac{(\bar c+ \bar h + \tau_0 \log |\cA|)^2 \kappa (\log (\abs{\cS} \abs{\cA}) + 1)}{(1-\gamma)^2} 4^{\lfloor k/l\rfloor + 3}
\]
with $l :=\left \lceil  \log_\gamma (1/4) \right \rceil$, then the relation in \eqnok{eq:srpmd_un_reg_optimal}
holds. As a consequence, an $\epsilon$-solution of \eqnok{eq:MDP_OPT}, i.e., a solution $\bar \pi$ s.t.
$\bbe [f(\bar \pi) - f(\pi^*)] \le \epsilon$, can be found 
in at most ${\cal O}(\log_\gamma \epsilon)$
SAPMD iterations. In addition, the total number of samples for $(s_t,a_t)$ pairs can be bounded by
\beq \label{eq:sample_SPMD}
{\cal O} (\tfrac{|\cS| |\cA| \log^2|\cA|  \log (\abs{\cS} \abs{\cA}) \log_\gamma (1/2) \log_\gamma \epsilon}{(1-\gamma)^4 \epsilon^2}).
\eeq
\end{proposition}

\begin{proof}
Using the fact that $\gamma^l \le 1/4$, we can easily check 
from \eqnok{eq:TM_SAPMD} and the selection of $T_k$ and $M_k$  that
\eqnok{eq:expectation_ass_ada}-\eqnok{eq:bias_ass_ada}  hold with
$\varsigma_k =2^{-(\lfloor k/l \rfloor+2)}$ and $\sigma_k^2 = 4^{-(\lfloor k/l \rfloor+2)}$.
Suppose that an $\epsilon$-solution $\bar \pi$ will be found at the $\bar k$ iteration. By
\eqnok{eq:srpmd_un_reg_optimal}, we have
\[
\lfloor \bar k/l \rfloor \le \log_2\{  [f(\pi_0) - f(\pi^*) 
+ \tfrac{3\sqrt{\log|\cA|}}{(1-\gamma)\sqrt{\gamma}} +\tfrac{5}{2(1-\gamma)} ]
\epsilon^{-1}\},
\]
which implies that the number of iterations is bounded by ${\cal O}( l \lfloor \bar k/l \rfloor) = {\cal O}(\log_\gamma \epsilon)$.
Moreover by the definition of $T_k$ and $M_k$, the number of samples is bounded by
\begin{align*}
&|\cS| |\cA| l\tsum_{p=1}^{\lfloor \bar k/l\rfloor +1}  
[\tfrac{l}{2} (p + \log_2 \tfrac{\bar c + \bar h + \tau_0 \log |\cA|}{1-\gamma} +4 ) \tfrac{(\bar c+ \bar h + \tau_0 \log |\cA|)^2  \kappa (\log (\abs{\cS} \abs{\cA}) + 1)}{(1-\gamma)^2} 4^{p + 3}] \\
 &= {\cal O} \{|\cS| |\cA| l^2 (\lfloor \bar k/l\rfloor + \log_2 \tfrac{\bar c + \bar h + \tau_0 \log |\cA|}{1-\gamma}) \tfrac{(\bar c+ \bar h+ \tau_0 \log |\cA|)^2 \log (\abs{\cS} \abs{\cA})}{(1-\gamma)^2} 4^{\lfloor \bar k/l\rfloor }\} \\
& =  {\cal O}  (\tfrac{|\cS| |\cA| \log^2 |\cA| \log (\abs{\cS} \abs{\cA}) \log_\gamma (1/2) \log_\gamma \epsilon}{(1-\gamma)^4 \epsilon}).
\end{align*}
\end{proof}

To the best of our knowledge, the results in Propositions~\ref{prop:SPMD_un} and \ref{prop:SAPMD_un}
appear to be new for policy gradient type methods. The previously best-known sampling complexity for
policy gradient methods for RL problems
was $\tilde {\cal O}(|\cS| |\cA|^2 /\epsilon^4)$~(e.g., \cite{DBLP:conf/aaai/ShaniEM20}) although some
improvements have been made under certain specific settings (e.g., \cite{Xu2020ImprovingSC}).
Observe that the sampling complexity in \eqnok{eq:sample_SPMD}
is slightly better than the one in \eqnok{eq:sample_SPMD_general} in the logarithmic terms. In fact, one can possibly further improve the dependence of 
the sampling complexity on $\gamma$ in  \eqnok{eq:sample_SPMD_general} 
 by a factor of $1/(1-\gamma)$ by allowing a slightly worse iteration complexity 
 than the one in \eqnok{eq:pmd_basic_iter}.  This indicates that one needs to carefully consider the tradeoff between iteration and sampling
complexities when implementing PMD type algorithms.

\subsection{Conditional temporal difference} \label{sec_single_trajectory}
In this subsection, we enhance a recently developed temporal different (TD) type method, i.e.,
conditional temporal difference (CTD) method, and use it to estimate the action-value functions
in an online manner. We focus 
on estimating $Q^\pi$ in SPMD since the estimation of $Q^\pi_{\tau}$ in SAPMD is similar. 

For a given policy $\pi$, we  denote the Bellman operator
\beq\label{eq:def_Bellman_operator}
T^{\pi} Q(s,a) := c(s, a) + h^{\pi}(s) + \gamma \tsum_{s' \in \cS} \cP(s'|s,a) \tsum_{a' \in \cA} \pi(a'|s') Q(s', a').
\eeq
The action value function $Q^{\pi}$ corresponding to policy $\pi$ satisfies the
Bellman equation
\beq \label{eq:fixed_point_Q}
Q^{\pi}(s,a) = T^{\pi} Q^{\pi}(s,a).
\eeq

We also need to define a positive-definite weighting matrix $M^\pi \in \bbr^{n \times n}$
to define the sampling scheme to evaluate policies using TD-type methods.
A natural weighting matrix 
is the diagonal matrix $ M^\pi = \Diag(\nu(\pi)) \otimes  \Diag(\pi)$, where $\nu(\pi)$ is the
steady state distribution induced by $\pi$ and $\otimes$ denotes the Kronecker product. 

\begin{assumption} \label{ass:ctd}
We make the following assumptions about policy $\pi$:
(a) $\nu(\pi)(s) \ge \underline \nu$ for some $\underline \nu > 0$, which holds
when the Markov chain employed with policy $\pi$ has a single ergodic class with unique stationary distribution,
i.e., $ \nu(\pi) = \nu(\pi) \cP^\pi$; and (b) $\pi$ is sufficiently random, i.e.,
$\pi(s, a) \ge \underline \pi$ for some $\underline \pi > 0$, which
can be enforced, for example, by adding some corresponding constraints through $h^\pi$. 
\end{assumption}

Note that Assumption~\ref{ass:ctd}.a) is widely accepted
for evaluating policies using TD type methods in the RL literature, and that Assumption~\ref{ass:ctd}.b) requires
 that $\pi$ assigns a non-zero probability to each action.
We will discuss how to possibly relax these assumptions, especially Assumption~\ref{ass:ctd}.b) later in Remark~\ref{remark_sigular}.

\vgap

In view of Assumption~\ref{ass:ctd} we have $M^\pi \succ 0$. 
With this weighting matrix $M^\pi$, we define the operator 
$F^\pi$ as 
$$ F^\pi(\theta) := M^\pi \big( \theta -   T^\pi \theta \big),$$
where $T^\pi$ is the Bellman operator defined in \eqnok{eq:def_Bellman_operator}.
Our goal is to find the root  $\theta^* \equiv Q^\pi$ 
of $F(\theta)$, i.e., $F(\theta^*)=0$. 
We can show that $F$ is strongly monotone with strong monotonicity modulus bounded from below by
$
\Lambda_{\min} :=  (1 - \gamma) \lambda_{\min}(M^\pi).
$
Here $\lambda_{\min}(A)$ denotes the smallest eigenvalue of $A$.
It can also be easily seen that $F^\pi$ is Lipschitz continuous with Lipschitz constant bounded by
$
\Lambda_{\max} := (1-\gamma) \lambda_{\max}(M^\pi),
$
where $\lambda_{\max}(A)$
denotes the largest eigenvalue of $A$.

\vgap

At time instant $t \in \mathbb{Z}_+$, we define the stochastic operator of $F^\pi$ as
$$
\tilde{F}^\pi(\theta_t, \zeta_t ) = \left( \langle e(s_t, a_t)  ,   \theta_t \rangle -   c(s_t,a_t) 
- h^\pi(s_t)  -   \gamma   \langle  e(s_{t+1}, a_{t+1})   ,   \theta_t  \rangle  \right) ~ e(s_t,a_t), 
$$
where $\zeta_{t} = ( s_t, a_t, s_{t+1}, a_{t+1})$ denotes the state transition steps following policy $\pi$
and $e(s_t,a_t)$ denotes the unit vector. 
The CTD method uses the stochastic operator $\tilde{F}^\pi(\theta_t, \zeta_t )$
to update the parameters $\theta_t$ iteratively as shown in Algorithm~\ref{alg:CTD_skipping}.
It involves two algorithmic parameters: $\alpha \ge 0$ determines
how often $\theta_t$ is updated and $\beta_t \ge 0$ defines the learning rate. Observe that if $\alpha = 0$, then
CTD reduces to the classic TD learning method. 

\begin{algorithm}[H]  \caption{Conditional Temporal Difference (CTD) for evaluating policy $\pi$}  
	\label{alg:CTD_skipping}
	\begin{algorithmic} 
		\State {Let $\theta_1$, the nonnegative parameters $\alpha$ and $\{ \beta_t\}$ be given. }
		\For{$ t = 1, \ldots, T$}
		\State {Collect $\alpha$ state transition steps  without updating $\{\theta_t\}$, denoted as $\{\zeta_t^1, \zeta_t^2, \dots, \zeta_t^\alpha\}$.}
		\State Set
		\beq \label{eq:CTD_skipping_step}
		\theta_{t+1} = \theta_t -  \beta_t  \tilde{F}^\pi(\theta_t,\zeta_t^\alpha).
		\eeq
		\EndFor
	\end{algorithmic}
\end{algorithm} 

When applying the general convergence results of CTD to our setting, we need to
handle the following possible pitfalls. Firstly, current analysis
of TD-type methods only provides bounds on $\bbe[\|\theta_t - \theta_*\|_2^2]$,
which gives an upper bound on $\bbe[\| \theta_t - Q^\pi\|_\infty^2]$ and thus
the bound on the total expected error (c.f., \eqnok{eq:variance_ass}). One needs to develop
a tight enough bound on the bias  $\|\bbe[\theta_t] - \theta_*\|_\infty$ (c.f., \eqnok{eq:bias_ass}) to derive the overall
best rate of convergence for the SPMD method.
Secondly, the selection of $\alpha$ and $\{\beta_t\}$ 
that gives the best rate of convergence in terms of $\bbe[\|\theta_t - \theta_*\|_2^2]$
does not necessarily result in the best rate of convergence 
for SPMD, since we need to deal with the bias term explicitly.

The following result can be shown similarly to Lemma 4.1 of \cite{KotsalisLanLi2020PartII}.
\begin{lemma} \label{2_2_markovian}
Given the single ergodic class Markov chain $ \zeta_1^1,\ldots, \zeta_1^\alpha, \zeta_2^2, \ldots, \zeta_2^\alpha, \ldots$, there exists a constant $ C > 0 $ and $ \rho \in [0,1)$ such that for every 
$t, \alpha \in \mathbb{Z}_+$ with probability 1, 
$$
\| F^\pi(\theta_t) - \bbe[\tilde{F}^\pi(\theta_t, \zeta_{t}^\alpha)|\zeta_{\lceil t-1\rceil}] \|_2  \leq C\rho^\alpha \|\theta_t-\theta^*\|_2.
$$
\end{lemma}
We can also show that the variance of $\tilde F^\pi$ is bounded as follows.
\begin{align}
&\bbe[\|\tilde{F}^\pi(\theta_t, \zeta_t^\alpha ) - \bbe[\tilde{F}^\pi(\theta_t, \zeta_{t}^\alpha)|\zeta_{\lceil t-1\rceil}\|_2^2] \nn\\
&\le 2 (1 +\gamma)^2 \bbe[ \|\theta_t\|_2^2] + 2 (\bar c +\bar h)^2 \nn \\
&\le 4 (1 +\gamma)^2 \bbe[\|\theta_t -\theta^*\|_2^2]  + \|\theta^*\|_2^2 + 2 (\bar c +\bar h)^2.
\end{align}

The following result has been shown in Proposition~6.2 of \cite{KotsalisLanLi2020PartII}.

\begin{lemma} \label{step_size_FastTD_skipping}
	If the algorithmic parameters in CTD are chosen such that
	\begin{align}\label{eq:def_tau_beta_CTD}
		\alpha \ge \tfrac{\log{(1/\Lambda_{\min})}+\log(9C)}{\log{(1/\rho)}} \ \ \mbox{and} \ \ \beta_t =  \tfrac{2}{\Lambda_{\min} ( t+t_0 -1) } 
	\end{align}
	with $t_0 =8\max\{\Lambda_{\max}^2, 8 (1 +\gamma)^2\} / \Lambda_{\min}^2$,
	then
	\begin{align*}
	\bbe[\| \theta_{t+1} - \theta^*\|_2^2]
	&\leq \tfrac{2( t_0 + 1)(t_0+2) \|\theta_1 - \theta^*\|^2 }{ (t + t_0)(t+t_0+1)}  +  \tfrac{12 t \sigma_F^2}{\Lambda_{\min}^2  (t + t_0)(t+ t_0+1)},
	\end{align*}
	where
	$\sigma_F^2 := 4(1+\gamma)^2 R^2 + \|\theta^*\|_2^2 + 2(\bar c + \bar h)^2$ and $R^2:=8 \|\theta_1 -\theta_*\|_2^2 + \tfrac{3[\|\theta^*\|_2^2 + 2 (\bar c +\bar h)^2]}{4 (1 +\gamma)^2}$.
	Moreover, we have $\bbe[\|\theta_t - \theta^*\|_2^2] \le R^2$ for any $t \ge 1$.
\end{lemma}

We now enhance the above result with a bound on the bias term given by $\|\bbe[ \theta_{t+1}] - \theta^*\|_2$.
The proof of this result is put in the appendix since it is more technical.

\begin{lemma} \label{lemma:bound_on_CTD_bias}
Suppose that the algorithmic parameters in CTD are set according to Lemma~\ref{step_size_FastTD_skipping}.
Then we have
\begin{align*}
 \| \bbe[ \theta_{t+1}] - \theta^*\|_2^2
&\le \tfrac{(t_0 - 1) (t_0 -2) (t_0 -3)  \|\theta_1 - \theta^*\|_2^2}{(t+t_0-1) (t + t_0 -2) (t+ t_0 -3)} + 
 \tfrac{8 C R^2 \rho^\alpha }{3\Lambda_{\min}} + \tfrac{C^2 R^2\rho^{2\alpha}} {\Lambda_{\min}^2}.
\end{align*}
\end{lemma}

We are now ready to establish the convergence of the SMPD method by using
the CTD method to estimate the action-value functions.
We focus on the case when $\mu >0$, and the case for $\mu = 0$ 
can be shown similarly.

\begin{proposition} \label{prop:CTD_sampling_SMPD}
Suppose that 
$\eta_k = \tfrac{1-\gamma}{\gamma \mu}$
in the SPMD method. 
If the initial point of CTD is set to $\theta_1 =0$ and the number of iterations $T$ and the parameter $\alpha$ in CTD are set to
\begin{align}
T_k &=t_0 (3 \bar \theta 2^{\lfloor k/l\rfloor+2})^{2/3} + (4 t_0^2 \bar \theta^2 2^{\lfloor k/l\rfloor+2})^{1/2} +
24 \sigma_F^2 \Lambda_{\min}^{-2} 2^{\lfloor k/l\rfloor+2}, \label{eq:def_T_CTD_SPMD}\\
\alpha_k &= \max\{ 2 (\lfloor \tfrac{k}{l}\rfloor + 2) \log_\rho \tfrac{1}{2} + \log_\rho \tfrac{\Lambda_{\min}}{24 C R^2},  
(\lfloor \tfrac{k}{l}\rfloor + 2) \log_\rho \tfrac{1}{2} + \log_\rho \tfrac{\Lambda_{\min}}{3 C R^2}\},  \label{eq:def_alpha_CTD_SPMD}
\end{align}
where $l :=\left \lceil  \log_\gamma (1/4) \right \rceil$ and $\bar \theta := \sqrt{n} \tfrac{\bar c + \bar h}{1-\gamma}$, 
then the relation in \eqnok{eq:SPMD_reg_MDP_results_uniform_constant1}
holds. As a consequence, an $\epsilon$-solution of \eqnok{eq:MDP_OPT}, i.e., a solution $\bar \pi$ s.t.
$\bbe [f(\bar \pi) - f(\pi^*) + \tfrac{\mu}{1-\gamma} \cD(\bar \pi,\pi^*)] \le \epsilon$, can be found 
in at most ${\cal O}(\log_\gamma \epsilon)$
SPMD iterations. In addition, the total number of samples for $(s_t,a_t)$ pairs can be bounded by
\beq \label{eq:sample_SPMD}
{\cal O} \{(\log_\gamma  \tfrac{1}{2})
 (\log_2 \tfrac{1}{\epsilon})  (\log_\rho \tfrac{\Lambda_{\min}}{C R^2})
( \tfrac{t_0 \bar \theta^{2/3}}{  (\mu (1-\gamma)\epsilon)^{2/3}}+ \tfrac{t_0 \bar\theta}{\sqrt{\mu (1-\gamma)\epsilon}} + \tfrac{\sigma_F^2 }{\mu (1-\gamma) \Lambda_{\min}^{2}\epsilon} )
 \}.
\eeq
\end{proposition}

\begin{proof}
Using the fact that $\gamma^l \le 1/4$, we can easily check 
from Lemma~\ref{step_size_FastTD_skipping}, Lemma~\ref{lemma:bound_on_CTD_bias}, and the selection of $T$ and $\alpha$  that
\eqnok{eq:expectation_ass}-\eqnok{eq:bias_ass}  hold with
$\varsigma_k =2^{-(\lfloor k/l \rfloor+2)}$ and $\sigma_k^2 = 2^{-(\lfloor k/l \rfloor+2)}$.
Suppose that an $\epsilon$-solution $\bar \pi$ will be found at the $\bar k$ iteration. By
\eqnok{eq:SPMD_reg_MDP_results_uniform_constant1}, we have
\[
\lfloor \bar k/l \rfloor \le \log_2\{  [f(\pi_0) - f(\pi^*) +\tfrac{1}{1-\gamma}( \mu  \log |\cA|
+ \tfrac{5} {2} + \tfrac{ 5}{8 \gamma \mu})] \epsilon^{-1}\},
\]
which implies that the number of iterations is bounded by ${\cal O}( l \lfloor \bar k/l \rfloor) = {\cal O}(\log_\gamma \epsilon)$.
Moreover by the definition of $T_k$ and $\alpha_k$, the number of samples is bounded by
\begin{align*}
& \tsum_{p=0}^{\lfloor \bar k/l\rfloor}  l \alpha_k T_k 
 =  {\cal O} \{\log_\gamma \tfrac{1}{2}
 \tsum_{p=0}^{\lfloor \bar k/l\rfloor}  (p \log_\rho \tfrac{1}{2} + \log_\rho \tfrac{\Lambda_{\min}}{C R^2})
(t_0 \bar \theta^{2/3} 2^{2p/3}  + t_0 \bar\theta 2^{p/2} + \sigma_F^2 \Lambda_{\min}^{-2} 2^{p})
 \}\\
& = {\cal O} \{\log_\gamma  \tfrac{1}{2}
 (\lfloor \bar k/l\rfloor \log_\rho \tfrac{1}{2} + \log_\rho \tfrac{\Lambda_{\min}}{C R^2})
(t_0 \bar \theta^{2/3} 2^{2\lfloor \bar k/l\rfloor/3} + t_0 \bar\theta 2^{\lfloor \bar k/l\rfloor/2} + \sigma_F^2 \Lambda_{\min}^{-2} 2^{\lfloor \bar k/l\rfloor})
 \}\\
 & = {\cal O} \{\log_\gamma  \tfrac{1}{2}
 (\log_2 \tfrac{1}{\epsilon} \log_\rho \tfrac{1}{2} + \log_\rho \tfrac{\Lambda_{\min}}{C R^2})
( 
\tfrac{t_0 \bar \theta^{2/3}}{  (\mu (1-\gamma)\epsilon)^{2/3}}+
 \tfrac{t_0 \bar\theta}{\sqrt{\mu (1-\gamma)\epsilon}} + \tfrac{\sigma_F^2 }{\mu (1-\gamma) \Lambda_{\min}^{2}\epsilon} )
 \}.
\end{align*}
\end{proof}

\vgap

The following result shows the convergence properties of the SAMPD method
when the action-value function is estimated by using the CTD method.
\begin{proposition} \label{prop:CTD_sampling_SAMPD}
Suppose that 
$\eta_k = \tfrac{1-\gamma}{\gamma \tau_k}$ and $
\tau_k = \tfrac{1}{\sqrt{ \gamma \log |\cA|}} 2^{-(\lfloor k/l \rfloor + 1)}$
in the SAPMD method. 
If the initial point of CTD is set to $\theta_1 =0$,
the number of iterations $T$ is set to
\begin{align}
T_k &= t_0 (3 \bar \theta 2^{\lfloor k/l\rfloor+2})^{2/3}  + (4 t_0^2 \bar \theta^2 4^{\lfloor k/l\rfloor+2})^{1/2} +
24 \sigma_F^2 \Lambda_{\min}^{-2} 4^{\lfloor k/l\rfloor+2}, \label{eq:def_T_CTD_SAPMD}
\end{align}
and the parameter $\alpha$ in CTD is set to \eqnok{eq:def_alpha_CTD_SPMD},
where $l :=\left \lceil  \log_\gamma (1/4) \right \rceil$ and $\bar \theta := \tfrac{\bar c + \bar h + \tau_0 \log |\cA|}{1-\gamma}$, 
then the relation in \eqnok{eq:srpmd_un_reg_optimal}
holds. As a consequence, an $\epsilon$-solution of \eqnok{eq:MDP_OPT}, i.e., a solution $\bar \pi$ s.t.
$\bbe [f(\bar \pi) - f(\pi^*)] \le \epsilon$, can be found 
in at most ${\cal O}(\log_\gamma \epsilon)$
SPMD iterations. In addition, the total number of samples for $(s_t,a_t)$ pairs can be bounded by
\beq \label{eq:sample_SAPMD}
{\cal O} \{(\log_\gamma  \tfrac{1}{2})
 (\log_2 \tfrac{1}{\epsilon})  (\log_\rho \tfrac{\Lambda_{\min}}{C R^2})
(  \tfrac{t_0 \bar\theta}{ (1-\gamma)\epsilon} + \tfrac{\sigma_F^2 }{(1-\gamma)^2 \Lambda_{\min}^{2}\epsilon^2} )
 \}.
\eeq
\end{proposition}

\begin{proof}
The proof is similar to that of Proposition~\ref{prop:CTD_sampling_SMPD}
except that we will show that \eqnok{eq:expectation_ass}-\eqnok{eq:bias_ass}  hold with
$\varsigma_k =2^{-(\lfloor k/l \rfloor+2)}$ and $\sigma_k^2 = 4^{-(\lfloor k/l \rfloor+2)}$.
Moreover, we will use \eqnok{eq:srpmd_un_reg_optimal} instead of \eqnok{eq:SPMD_reg_MDP_results_uniform_constant1}
to bound the number of iterations.
\end{proof}

\vgap

To the best of our knowledge, the complexity result in \eqnok{eq:sample_SPMD} is new 
in the RL literature, while the one in \eqnok{eq:sample_SAPMD} is new for policy gradient
type methods. It seems that this bound significantly improves the previously best-known ${\cal O}(1/\epsilon^3)$ sampling complexity result 
for stochastic policy gradient methods (see \cite{Xu2020ImprovingSC}  and Appendix C of \cite{Khodadadian2021} for more explanation).


\begin{remark} \label{remark_sigular}
In this subsection we focus on the more restrictive 
assumption $M^\pi \succ 0$
in order to compare our results with the existing ones in the literature.
Here we discuss how one can possibly relax this assumption.

If $\nu(\pi)(s) \cdot \pi(s, a) = 0$ for some $(s,a) \in \cS \times \cA$, 
one may define the weighting matrix
$M^\pi = (1 -\lambda) \Diag(\nu(\pi)) \otimes  \Diag(\pi) + \tfrac{\lambda}{n} I$
for some sufficiently small $\lambda \in (0,1)$ which depends on the target accuracy
for solving the RL problem, where $n =  |\cS| \times |\cA|$.  
As a result, the algorithmic frameworks of CTD and SPMD, and their convergence analysis are 
still applicable to this more general setting.
Obviously, the selection of $\lambda$ will impact the efficiency estimate for 
policy evaluation. 

An alternative approach that can relax Assumption~\ref{ass:ctd}.b), 
would be to first
run the enhanced CTD method to the following equation
\[
V^\pi (s) = \tsum_{a} \pi(a|s) [c(s, a) + h^{\pi}(s) + \gamma \tsum_{s' \in \cS} \cP(s'|s,a) V^\pi(s')]
\]
to evaluate the state-value function $V^\pi$. Then we estimate 
the action-value function $Q^\pi$ by using \eqnok{eq:QV2}, i.e.,
\[
Q^\pi(s, a) = c(s, a) + h^\pi(s) + \gamma \tsum_{s' \in \cS} \cP(s'| s, a) V^\pi(s'). 
\]
In order to use the above identity, we need to define an estimator of $ \cP(s'| s, a)$
by using a uniform policy $\pi_0(\cdot|s) := \{1/|A|, \ldots,1/|A|\}$.
The sample size required to estimate the transition kernel from a single trajectory
is an active research topic (see~\cite{WolferKonttorovich20} and references therein).
Current research has been focused only on bounding on the total error for estimating $ \cP(s'| s, a)$
for a given sample size, and there do not exist separate and
tighter bounds on the bias for these estimators.
Therefore, it is still not evident whether the same
sampling complexity bounds in Propositions~\ref{prop:CTD_sampling_SMPD} and~\ref{prop:CTD_sampling_SAMPD}
can be maintained using this alternative approach to relax the assumption of
non-zero probability to each action. 
%

\end{remark}

\begin{remark} \label{remark_linear_app}
For problems of high dimension (i.e., $n \equiv  |\cS| \times |\cA|$ is large),
one often resorts to a parametric approximation of the value function. In this case 
it is possible to define a more general operator $ F^\pi(\theta) := \Phi^{T} M^\pi \big( \Phi \theta -   T^\pi \Phi \theta \big)$
for some feature matrix $\Phi$ to evaluate the value functions (see Section~ 4 of \cite{KotsalisLanLi2020PartII} for
a discussion about CTD with function approximation). 
Unless the column space of $\Phi$ spans the true value functions,
an additional bias term will be introduced into the computation of gradients, resulting into
an extra error term in the overall rate of convergence
of PMD methods. In other words, these methods can only be guaranteed to converge
to a neighborhood of the optimal solution. Nevertheless, the application
of function approximation will significantly reduce the dependence of gradient computation
on the problem dimension, i.e., 
from $|\cS| \times |\cA|$ to the number of columns of~$\Phi$. 

\end{remark}

\section{Efficient Solution for  General Subproblems} \label{sec:gradient_complexity}
In this section, we study the convergence properties of the PMD methods for
 the situation where we do not have 
exact solutions for prox-mapping subprobems. 
Throughout this section, we assume that $h^\pi$ is differentiable and its gradients
are Lipschitz continuous with Lipschitz constant $L$.
We will first review Nesterov's accelerated
gradient descent (AGD) method~\cite{Nest83-1}, and then discuss the overall gradient complexity of
using this method for solving prox-mapping in the PMD methods.
We will focus on the stochastic PMD methods since they cover deterministic
methods as certain special cases.

\subsection{Review of accelerated gradient descent} 
Let us denote $X \equiv \Delta_{|\cA|}$ and consider the problem of 
\beq \label{eq:AGD_problem}
\min_{x \in X} \{\Phi(x) := \phi(x) + \chi(x)\},
\eeq
where $\phi: X \to \bbr$ is a smooth convex function such that
\[
\mu_\phi D_{x'}^x \le \phi(x) - [\phi(x') + \langle \nabla \phi(x'), x - x' \rangle] \le \tfrac{L_\phi}{2} \|x - x'\|_1.
\]
Moreover, we assume that $\chi: X \to \bbr$ satisfies
\[
\chi(x) - [\chi(x') + \langle \chi'(x'), x - x'\rangle] \ge \mu_\chi D_{x'}^x
\]
for some $\mu_\chi \ge 0$. 
Given $(x_{t-1}, y_{t-1}) \in X \times X$, the accelerated gradient method performs the following updates:
\begin{align}
\underline x_{t} &= (1 - q_t) y_{t-1} + q_t x_{t-1}, \label{eq:AG1}\\
x_{t} &= \argmin_{x \in X}  \{r_t [\langle \nabla \phi(\underline x_{t}), x \rangle + \mu_\phi D_{\underline x_t}^x + \chi(x)] + D_{x_{t-1}}^x \}, \label{eq:AG2} \\
y_{t} &=  (1 - \rho_t) y_{t-1} + \rho_t x_{t}, \label{eq:AG3}
\end{align}
for some $q_t \in [0,1]$, $r_t \ge 0$, and $\rho_t \in [0,1]$ .

Below we slightly generalize the 
convergence results for the AGD method so that they depend on
the distance $D_{x_{0}}^x$ rather than $\Phi(y_0) - \Phi(x)$
for any $x \in X$. This result better fits our need to analyze the convergence
of inexact SPMD and SAPMD methods in the next two subsections.

\begin{lemma} \label{prop:AGD_convergence}
Let us denote 
$\mu_\Phi := \mu_\phi + \mu_\chi$ and
$t_0: = \lfloor 2 \sqrt{L_\phi/ \mu_\Phi} - 1\rfloor$. If 
\[
\rho_t = 
\begin{cases}
\tfrac{2}{t+1} & t \le t_0\\
\sqrt{\mu_\Phi/L_\phi} & \mbox{o.w.}
\end{cases},
q_t = 
\begin{cases}
\tfrac{2}{t+1} & t \le t_0\\
\tfrac{\sqrt{\mu_\Phi/L_\phi} -  \mu_\Phi /L_\phi}{1-  \mu_\Phi /L_\phi} & \mbox{o.w.}
\end{cases},
r_t = 
\begin{cases}
\tfrac{t}{2 L_\phi} & t \le t_0\\
\tfrac{1}{\sqrt{L_\phi \mu_\Phi} - \mu_\Phi} & \mbox{o.w.}
\end{cases},
\]
then for any $x \in X$,
\beq \label{eq:AGDstrongtheorem2}
\Phi(y_t) - \Phi(x) + \mu_\Phi D_{x_{t}}^x \le \varepsilon(t)   D_{x_{0}}^x,
\eeq
where
\beq \label{eq:def_varepsilon}
\varepsilon(t) := 2 L_\phi\min\left\{\left(1-\sqrt{ \mu_\Phi/L_\phi}\,\right)^{t-1}, \tfrac{2}{t(t+1)}\right\} .
\eeq
\end{lemma}

\begin{proof}
Using the discussions in Corollary 3.5 of \cite{LanBook2020}  (and the possible strong convexity of $\chi$),
we can check that the conclusions in Theorems~3.6 and 3.7 of \cite{LanBook2020} hold for the AGD method applied to
problem~\eqnok{eq:AGD_problem}.
It then follows from Theorem~3.6 of \cite{LanBook2020} that
\beq \label{eq:AGD_temp}
\Phi(y_t) - \Phi(x) + \tfrac{\rho_t}{r_t} D_{x_k}^{x^*} \le \tfrac{4 L_\phi}{ t(t+1)} D_{x_0}^{x}, \forall t = 1, \ldots, t_0.
\eeq
Moreover, it follows from Theorem 3.7 of \cite{LanBook2020} that for any $t \ge t_0$,
\begin{align*}
\Phi(y_t) - \Phi(x) + \mu_\Phi D_{x_{k}}^x &\le \left(1-\sqrt{\mu_\Phi/L_\phi}\right)^{t-t_0}
[\Phi(y_{t_0}) - \Phi(x) +\mu_\Phi D_{x_{t_0}}^x] \\
&\le 2 \left(1-\sqrt{\mu_\Phi/L_\phi}\right)^{t-1} L_\phi D_{x_0}^{x},
\end{align*}
where the last inequality follows from \eqnok{eq:AGD_temp} (with $t = t_0$) and
the facts that 
\begin{align*}
\tfrac{\rho_{t}}{r_{t}} \ge \tfrac{\rho_{t_0}}{r_{t_0}} \ge \mu_\Phi \ \mbox{and} \
 \tfrac{2}{ t(t+1)} = \textstyle{\prod}_{i=2}^t (1-\tfrac{2}{i+1}) \le (1 -\sqrt{ \mu_\Phi/L_\phi})^{t-1}
\end{align*}
for any $2 \le t \le t_0$.
The result then follows by combining these observations.
\end{proof}

\subsection{Convergence of inexact SPMD}
In this subsection, we study the convergence properties of the SPMD
method when its subproblems are solved inexactly by using the AGD method (see Algorithm~\ref{inexact_spmd}).
Observe that we use the same initial point $\pi_0$ whenever calling the AGD method.
To use a dynamic initial point (e.g., $v_k$) will make the analysis more complicated since we
do not have a uniform bound on the KL divergence $D_{v_k}^\pi$ for an arbitrary $v_k$.
To do so probably will require us to use other distance generating functions than the entropy function.

\begin{algorithm}[H]
\caption{The SPMD method with inexact subproblem solutions}
\begin{algorithmic}
\State {\bf Input:} initial points $\pi_0=v_0$ and stepsizes $\eta_k\ge 0$.
\For {$k =0,1,\ldots,$}
\State  Apply $T_k$ AGD iterations (with initial points $x_0 = y_0 = \pi_0$) to
\beq \label{eq:SPMD_step_app}
\pi_{k+1}(\cdot|s) = \argmin_{p(\cdot|s) \in \Delta_{|\cA|}} \left\{ \Phi_k(p):=\eta_k [\langle \cQ^{\pi_{k},\xi_k}(s, \cdot), p(\cdot|s) \rangle + h^p(s)]+ D_{v_k}^p(s)\right\}.
\eeq
\State Set $(\pi_{k+1}, v_{k+1}) = (y_{T_k+1}, x_{T_k+1})$.
\EndFor
\end{algorithmic} \label{inexact_spmd}
\end{algorithm}

In the sequel, we will denote $\varepsilon_k \equiv \varepsilon(T_k)$ to simplify notations.
The following result will take place of Lemma~\ref{lemma:prox_optimality}
in our convergence analysis.

\begin{lemma} \label{lemma:prox_optimality_app}
For any $\pi(\cdot|s) \in X$, we have
\begin{align}
&\eta_k [ \langle \cQ^{ \pi_k, \xi_k}(s,\cdot), \pi_{k+1}(\cdot|s) - \pi(\cdot|s) \rangle + h^{ \pi_{k+1}}(s) - h^{\pi}(s)] \nn\\
&+ D_{v_k}^{\pi_{k+1}}(s) + (1 + \mu \eta_k) D_{v_{k+1}}^\pi(s) \le D_{v_k}^\pi(s) + \varepsilon_k \log |\cA| .\label{eq:prox_opt_app1}
\end{align}
Moreover, we have 
\begin{align}
&\eta_k [ \langle \cQ^{\pi_k, \xi_k}(s,\cdot), \pi_{k+1}(\cdot|s) - \pi_k(\cdot|s) \rangle + h^{ \pi_{k+1}}(s) - h^{ \pi_k}(s)] \nn \\
&+ D_{v_k}^{ \pi_{k+1}}(s) + (1 + \mu \eta_k) D_{v_{k+1}}^{\pi_{k+1}}(s) \le  (\varepsilon_k + \tfrac{\varepsilon_{k-1}}{1+ \mu \eta_{k-1}})\log |\cA|.\label{eq:prox_opt_app2}
\end{align}
\end{lemma}

\begin{proof}
It follows from Lemma~\ref{prop:AGD_convergence} (with $\mu_\Phi = 1 + \mu \eta_k$ and $L_\phi = L$) that
\[
\Phi_k(\pi_{k+1}) - \Phi_k(\pi) +
(1 + \mu \eta_k  ) D_{v_{k+1}}^\pi(s) \le \epsilon_k D_{\pi_0}^\pi(s) \le  \varepsilon_k \log|\cA|.
\]
Using the definition of $\Phi_k$, we have
\begin{align*}
&\eta_k [ \langle \cQ^{\pi_k, \xi_k}(s,\cdot), \pi_{k+1}(\cdot|s) - \pi(\cdot|s) \rangle + h^{ \pi_{k+1}}(s) - h^{\pi}(s)]\\
&+ D_{v_k}^{\pi_{k+1}}(s) - D_{v_k}^{\pi}(s) + (1 + \mu \eta_k) D_{v_{k+1}}^\pi(s) \le  \varepsilon_k \log|\cA|,
\end{align*}
which proves \eqnok{eq:prox_opt_app1}.
Setting $\pi = \pi_k$ and $\pi = \pi_{k+1}$ respectively, in the above conclusion, we obtain
\begin{align*}
&\eta_k [ \langle \cQ^{ \pi_k, \xi_k}(s,\cdot),  \pi_{k+1}(\cdot|s) - \pi_k(\cdot|s) \rangle + h^{ \pi_{k+1}}(s) - h^{ \pi_k}(s)]\\
&+ D_{v_k}^{ \pi_{k+1}}(s) + (1 + \mu \eta_k) D_{v_{k+1}}^{ \pi_k}(s) \le D_{v_k}^{ \pi_k}(s) + \varepsilon_k \log |\cA| ,\\
&(1 + \mu \eta_k) D_{v_{k+1}}^{ \pi_{k+1}}(s) \le  \varepsilon_k \log |\cA|.
\end{align*}
Then \eqnok{eq:prox_opt_app2} follows by combining these two inequalities.
\end{proof}

\begin{proposition} \label{prop:SPMD_generic_app}
For any $s \in \cS$, we have
\begin{align}
V^{\pi_{k+1}}(s) - V^{\pi_k}(s) &\le
\langle \cQ^{\pi_k,\xi_k}(s, \cdot) , \pi_{k+1}(\cdot | s) - \pi_k(\cdot|s) \rangle + h^{\pi_{k+1}}(s) - h^{\pi_{k}}(s) \nn\\
& \quad   +
\tfrac{1}{\eta_k} D_{\pi_k}^{\pi_{k+1}}(s)+ \tfrac{\eta_k \|\delta_k\|_\infty^2}{2(1-\gamma)} +
 \tfrac{\gamma}{(1-\gamma)\eta_k} (\varepsilon_k + \tfrac{\varepsilon_{k-1}}{1+ \mu \eta_{k-1}})\log |\cA|\nn\\
&\quad -\tfrac{1}{1-\gamma} \bbe_{s' \sim d_s^{\pi_{k+1}}}[ \langle \delta_k, v_{k}(\cdot | s') - \pi_k(\cdot|s') \rangle].
\label{eq:spmd_decrease2_inexact}
\end{align}
\end{proposition}

\begin{proof}
Similar to \eqnok{eq:improvement_in_one_step_temp_stoch}, we have
\begin{align}
&V^{\pi_{k+1}}(s) - V^{\pi_k}(s)\nn \\
&=
\tfrac{1}{1-\gamma} \bbe_{s' \sim d_s^{\pi_{k+1}}}
\left[
\langle \cQ^{\pi_k,\xi_k}(s', \cdot), \pi_{k+1}(\cdot | s') - \pi_k(\cdot|s')\rangle + h^{\pi_{k+1}}(s') - h^{\pi_{k}}(s') \right.\nn \\
& \quad 
\left.  - \langle \delta_k, \pi_{k+1}(\cdot | s') - v_k(\cdot|s') \rangle - \langle \delta_k, v_{k}(\cdot | s') - \pi_k(\cdot|s') \rangle \right] \nn\\
&\le \tfrac{1}{1-\gamma} \bbe_{s' \sim d_s^{\pi_{k+1}}}
\left[
\langle \cQ^{\pi_k,\xi_k}(s', \cdot), \pi_{k+1}(\cdot | s') - \pi_k(\cdot|s') \rangle + h^{\pi_{k+1}}(s') - h^{\pi_{k}}(s') \right.\nn \\
& \quad
\left. +
 \tfrac{1 }{2\eta_k}\|\pi_{k+1}(\cdot | s') - v_k(\cdot|s') \|_1^2 + \tfrac{\eta_k \|\delta_k\|_\infty^2}{2} - \langle \delta_k, v_{k}(\cdot | s') - \pi_k(\cdot|s') \rangle \right] \nn\\
 &\le \tfrac{1}{1-\gamma} \bbe_{s' \sim d_s^{\pi_{k+1}}}
\left[
\langle \cQ^{\pi_k,\xi_k}(s', \cdot), \pi_{k+1}(\cdot | s') - \pi_k(\cdot|s') \rangle + h^{\pi_{k+1}}(s') - h^{\pi_{k}}(s')\right.\nn \\
& \quad
\left. + \tfrac{1}{\eta_k} D_{v_{k}}^{\pi_{k+1}}(s') + \tfrac{\eta_k \|\delta_k\|_\infty^2}{2} - \langle \delta_k, v_{k}(\cdot | s') - \pi_k(\cdot|s') \rangle\right]. \label{eq:improvement_in_one_step_temp_stoch_inexact}
\end{align}
It follows from \eqnok{eq:prox_opt_app2} that
\begin{align}
&  \langle \cQ^{\pi_k, \xi_k}(s,\cdot), \pi_{k+1}(\cdot|s) - \pi_k(\cdot|s) \rangle + h^{ \pi_{k+1}}(s) - h^{ \pi_k}(s) \nn \\
&+\tfrac{1}{\eta_k} \left[D_{v_k}^{ \pi_{k+1}}(s) + (1 + \mu \eta_k) D_{v_{k+1}}^{\pi_{k+1}}(s)- (\varepsilon_k+ \tfrac{\varepsilon_{k-1}}{1+ \mu \eta_{k-1}})\log |\cA|
\right]
\le 0, \label{eq:improvement_in_one_step_temp_opt_stoch_inexact}
\end{align}
which implies that
\begin{align*}
&\bbe_{s' \sim d_s^{\pi_{k+1}}}
\left[
\langle \cQ^{\pi_k,\xi_k}(s', \cdot), \pi_{k+1}(\cdot | s') - \pi_k(\cdot|s') \rangle 
+ h^{\pi_{k+1}}(s') - h^{\pi_{k}}(s')  \right. \\
& \quad \left. + \tfrac{1}{\eta_k} \left(D_{\pi_{k+1}}^{\pi_{k}}(s') - (\varepsilon_k + \tfrac{\varepsilon_{k-1}}{1+ \mu \eta_{k-1}})\log |\cA| \right)\right]\nn \\
& \le d_s^{\pi_{k+1}}(s) \left[
\langle \cQ^{\pi_k,\xi_k}(s, \cdot), \pi_{k+1}(\cdot | s) - \pi_k(\cdot|s) \rangle + h^{\pi_{k+1}}(s) - h^{\pi_{k}}(s) \right. \\
& \quad \left. + \tfrac{1}{\eta_k} \left( D_{\pi_k}^{\pi_{k+1}}(s)- (\varepsilon_k + \tfrac{\varepsilon_{k-1}}{1+ \mu \eta_{k-1}})\log |\cA| \right) \right]\nn \\
 &\le (1-\gamma) \left[
\langle \cQ^{\pi_k,\xi_k}(s, \cdot), \pi_{k+1}(\cdot | s) - \pi_k(\cdot|s) \rangle + h^{\pi_{k+1}}(s) - h^{\pi_{k}}(s) \right. \\
&\quad \left. + \tfrac{1}{\eta_k} \left(D_{\pi_k}^{\pi_{k+1}}(s) - (\varepsilon_k + \tfrac{\varepsilon_{k-1}}{1+ \mu \eta_{k-1}})\log |\cA| \right)  \right], 
 \end{align*}
where the last inequality follows from the fact that $d_s^{\pi_{k+1}}(s) \ge (1-\gamma)$ due to the definition of $d_s^{\pi_{k+1}}$
in \eqnok{eq:def_visitation}. The result in \eqnok{eq:spmd_decrease2_inexact} 
then follows immediately from \eqnok{eq:improvement_in_one_step_temp_stoch_inexact} and the above inequality.
\end{proof}

\vgap

We now establish an important recursion about the inexact SPMD method in Algorithm~\ref{inexact_spmd}.

\begin{lemma} \label{prop:SPMD_regularized_MDPs_inexact}
Suppose that $\eta_k = \eta = \tfrac{1-\gamma}{\gamma \mu}$  and $\varepsilon_k \le \varepsilon_{k-1}$ for any $k \ge 0$ in the inexact
SPMD method, we have
\begin{align*}
& \bbe_{\xi_{\lceil k \rceil}} [f(\pi_{k+1}) - f(\pi^*) + \tfrac{\mu}{1-\gamma}  \cD(\pi_{k+1},\pi^*)] \\
&\le  \gamma [ \bbe_{\xi_{\lceil k-1 \rceil}} [f(\pi_k) - f(\pi^*) + \tfrac{\mu}{1-\gamma} D(\pi_k, \pi^*) ]
+ \tfrac{2(2-\gamma) \varsigma_k}{1-\gamma} + \tfrac{ \sigma_k^2}{2 \gamma \mu} + \tfrac{\mu \gamma^2 ( 1 + \gamma) \log |\cA| \varepsilon_{k-1}}{(1-\gamma)^2}.
\end{align*}
\end{lemma}

\begin{proof}
By \eqnok{eq:prox_opt_app1}
(with $p = \pi^*$), we have
\begin{align*}
& \langle \cQ^{\pi_{k},\xi_k}(s, \cdot), 
\pi_{k+1}(\cdot|s) - \pi^*(\cdot|s) \rangle 
+ h^{\pi_{k+1}}(s) -  h^{\pi^*}(s) 
+ \tfrac{1}{\eta_k}D_{v_k}^{\pi_{k+1}}(s) \\
&\le \tfrac{1}{\eta_k} D_{v_k}^{\pi^*}(s) - (\tfrac{1}{\eta_k}+ \mu) D_{v_{k+1}}^{\pi^*}(s) + \tfrac{\varepsilon_k}{\eta_k} \log |\cA|,
\end{align*}
which, in view of \eqnok{eq:spmd_decrease2}, then implies that
\begin{align*}
&\langle \cQ^{\pi_{k},\xi_k}(s, \cdot) , 
\pi_{k}(\cdot|s) - \pi^*(\cdot|s) \rangle + h^{\pi_{k}}(s) -  h^{\pi^*}(s) +  V^{\pi_{k+1}}(s) - V^{\pi_k}(s) 
 \\
& 
\le \tfrac{1}{\eta_k} D_{v_k}^{\pi^*}(s) - (\tfrac{1}{\eta_k} + \mu) D_{v_{k+1}}^{\pi^*}(s) + \tfrac{\eta_k \|\delta_k\|_\infty^2}{2(1-\gamma)}
+\tfrac{\gamma}{(1-\gamma)\eta_k} (\varepsilon_k + \tfrac{\varepsilon_{k-1}}{1+ \mu \eta_{k-1}})\log |\cA|\nn\\
&\quad -\tfrac{1}{1-\gamma} \bbe_{s' \sim d_s^{\pi_{k+1}}}[ \langle \delta_k, v_{k}(\cdot | s') - \pi_k(\cdot|s') \rangle].
\end{align*}
Taking expectation w.r.t. $\xi_{\lceil k \rceil}$ and $\nu^*$ on both sides of the above inequality,
and using Lemma~\ref{prop_strong_monotone} and the relation in \eqnok{eq:bnd_bias_SPMD}, 
we arrive at
\begin{align*}
&\bbe_{s\sim \nu^*, \xi_{\lceil k \rceil}}\left[(1- \gamma)(V^{\pi_k}(s) - V^{\pi_\tau^*}(s) ) 
+  V^{\pi_{k+1}}(s) - V^{\pi_k}(s) \right] \\
&\le \bbe_{s\sim \nu^*,\xi_{\lceil k \rceil}}[\tfrac{1}{\eta_k} D_{v_k}^{\pi^*}(s) - (\tfrac{1}{\eta_k}+ \mu) D_{v_{k+1}}^{\pi^*}(s)] 
+ 2  \varsigma_k + \tfrac{ \eta_k \sigma_k^2}{2(1-\gamma)} \nn\\
&\quad + \tfrac{\gamma}{(1-\gamma)\eta_k} (\varepsilon_k + \tfrac{\varepsilon_{k-1}}{1+ \mu \eta_{k-1}})\log |\cA| +\tfrac{2}{1-\gamma}\varsigma_k.
\end{align*}
Noting $V^{\pi_{k+1}}(s) -  V^{\pi_k}(s)= 
V^{\pi_{k+1}}(s) - V^{\pi^*}(s) - [V^{\pi_k}(s) - V^{\pi^*}(s) ] $,
rearranging the terms in the above inequality, and using the definition of $f$ in \eqnok{eq:MDP_OPT}, we arrive at
\begin{align*}
& \bbe_{\xi_{\lceil k \rceil}} [f(\pi_{k+1}) - f(\pi^*) + (\tfrac{1}{\eta_k}+  \mu) \cD(v_{k+1}, \pi^*)] 
\le   \bbe_{\xi_{\lceil k-1 \rceil}} [\gamma (f(\pi_k) - f(\pi^*))+ \tfrac{1}{\eta_k} \cD(v_k,\pi^*) ]\\
&\quad \quad + \tfrac{2(2-\gamma) \varsigma_k}{1-\gamma} + \tfrac{ \eta_k \sigma_k^2}{2(1-\gamma)} 
+ \tfrac{\gamma}{(1-\gamma)\eta_k} (\varepsilon_k + \tfrac{\varepsilon_{k-1}}{1+ \mu \eta_{k-1}})\log |\cA|.
\end{align*}
The result then follows immediately by the selection of $\eta$ and the assumption $\varepsilon_k \le \varepsilon_{k-1}$.
\end{proof}

We now are now ready to state the convergence rate of the SPMD method
with inexact prox-mapping. We focus on the case when $\mu > 0$.

\begin{theorem} \label{theorem:SPMD_reg_inexact}
Suppose that 
$\eta_k = \eta = \tfrac{1-\gamma}{\gamma \mu}$
in the inexact SPMD method. 
If $\varsigma_k =(1-\gamma) 2^{-(\lfloor k/l \rfloor+2)}$, $\sigma_k^2 = 2^{-(\lfloor k/l \rfloor+2)}$
and $\varepsilon_k = (1-\gamma)^2 2^{-(\lfloor (k+1)/l \rfloor+2)}$ 
for any $k \ge 0$ with $l :=\left \lceil  \log_\gamma (1/4) \right \rceil
$,
then 
\begin{align}
& \bbe_{\xi_{\lceil k-1 \rceil}} [f(\pi_{k}) - f(\pi^*) 
+ \tfrac{\mu}{1-\gamma} \cD(\pi_{k},\pi^*)]\nn \\
&\le 2^{-\lfloor k/l \rfloor}\left [f(\pi_0) - f(\pi^*) +\tfrac{1}{1-\gamma}( \mu  \log |\cA|
+ \tfrac{5(2-\gamma)} {2} + \tfrac{ 5}{8 \gamma \mu} + \tfrac{5\mu \gamma^2(1+\gamma)\log |\cA|}{4})\right].
\end{align}
\end{theorem}

\begin{proof}
The result follows as an immediate consequence of Proposition~\ref{prop:SPMD_regularized_MDPs_inexact}
and Lemma~\ref{lemma:induction_tech_sequence}.
\end{proof}

\vgap

In view of Theorem~\ref{theorem:SPMD_reg_inexact},
the inexact solutions of the subproblems barely affect the iteration and sampling complexities of
the SPMD method as long as $\varepsilon_k \le (1-\gamma)^2 2^{-(\lfloor (k+1)/l \rfloor+2)}$.
Notice that an $\epsilon$-solution of problem~\eqnok{eq:MDP_OPT},
 i.e., a solution $\bar \pi$ s.t.
$\bbe [f(\bar \pi) - f(\pi^*)] \le \epsilon$,
can be found at the $\bar k$-th iteration with
\[
\lfloor \bar k/l \rfloor \le \log_2 \{\epsilon^{-1} [f(\pi_0) - f(\pi^*) 
+ \tfrac{1}{1-\gamma}( 4 \mu  \log |\cA|
+ 5 + \tfrac{ 5}{8 \gamma \mu})]\}.
\]
Also observe that the condition number of the subproblem is given by 
\[
\tfrac{L \eta_k}{\mu \eta_k+1} = \tfrac{L(1-\gamma)}{\mu}.
\]
Combining these observations with Lemma~\ref{prop:AGD_convergence}, we conclude that the total number of gradient computations
of $h$ can be bounded by
\begin{align*}
l \tsum_{p=0}^{\lfloor \bar k/l \rfloor } \sqrt{\tfrac{L \eta}{\mu \eta+1} } \log (L\eta/\varepsilon_k)
&=l \tsum_{p=0}^{\lfloor \bar k/l \rfloor } \sqrt{\tfrac{L (1-\gamma)}{\mu } } \log \tfrac{4 L 2^{p+1}}{\gamma (1-\gamma) \mu }\\
&={\cal O}\{ l (\lfloor \bar k/l \rfloor)^2  \sqrt{\tfrac{L (1-\gamma)}{\mu } } \log \tfrac{ L}{\gamma (1-\gamma) \mu }\}\\
&= {\cal O}\left\{ ( \log_\gamma \tfrac{1}{4}) (\log^2 \tfrac{1}{\epsilon}) \, \sqrt{\tfrac{L (1-\gamma)}{\mu } } ( \log \tfrac{ L}{\gamma (1-\gamma) \mu })\right\}.
\end{align*}

\subsection{Convergence of inexact SAPMD}
In this subsection, we study the convergence properties of the SAPMD
method when its subproblems are solved inexactly by using the AGD method (see Algorithm~\ref{inexact_srpmd}).

\begin{algorithm}[H]
\caption{The Inexact SAPMD method}
\begin{algorithmic}
\State {\bf Input:} initial points $\pi_0=v_0$, stepsizes $\eta_k\ge 0$, and regularization parameters $\tau_k \ge 0$.
\For {$k =0,1,\ldots,$}
\State  Apply $T_k$ AGD iterations (with initial points $x_0 = y_0 = \pi_0$) to
\beq \label{eq:SRPMD_step_app}
\pi_{k+1}(\cdot|s) = \argmin_{p(\cdot|s) \in \Delta_{|\cA|}} \left\{ \tilde \Phi_k(p):=\eta_k [\langle \cQ_{\tau_k}^{\pi_{k},\xi_k}(s, \cdot), p(\cdot|s) \rangle + h^p(s)
+ \tau_k D_{\pi_0}^p(s)]+ D_{v_k}^p(s)\right\}.
\eeq
\State Set $(\pi_{k+1}, v_{k+1}) = (y_{T_k+1}, x_{T_k+1})$.
\EndFor
\end{algorithmic} \label{inexact_srpmd}
\end{algorithm}

In the sequel, we will still denote $\varepsilon_k \equiv \varepsilon(T_k)$ to simplify notations.
The following result has the same role as Lemma~\ref{lemma:prox_optimality_app}
in our convergence analysis.

\begin{lemma} \label{lemma:SRPMD_prox_optimality_app}
For any $\pi(\cdot|s) \in X$, we have
\begin{align}
&\eta_k [ \langle \cQ_{\tau_k}^{ \pi_k, \xi_k}(s,\cdot), \pi_{k+1}(\cdot|s) - \pi(\cdot|s) \rangle + h^{ \pi_{k+1}}(s) - h^{\pi}(s) + \tau_k(D_{\pi_0}^{\pi_{k+1}}(s) - D_{\pi_0}^\pi(s))] \nn\\
&+ D_{v_k}^{\pi_{k+1}}(s) + (1 + \tau_k \eta_k) D_{v_{k+1}}^\pi(s) \le D_{v_k}^\pi(s) + \varepsilon_k \log |\cA|. \label{eq:prox_opt_app1_SRPMD}
\end{align}
Moreover, we have 
\begin{align}
&\eta_k [ \langle \cQ_{\tau_k}^{\pi_k, \xi_k}(s,\cdot), \pi_{k+1}(\cdot|s) - \pi_k(\cdot|s) \rangle + h^{ \pi_{k+1}}(s) - h^{ \pi_k}(s) + \tau_k(D_{\pi_0}^{\pi_{k+1}}(s) - D_{\pi_0}^{\pi_k}(s))] \nn \\
&+ D_{v_k}^{ \pi_{k+1}}(s) + (1 + \tau_k \eta_k) D_{v_{k+1}}^{\pi_{k+1}}(s) \le  (\varepsilon_k+ \tfrac{\varepsilon_{k-1}}{1+ \tau_{k-1} \eta_{k-1}})\log |\cA|.\label{eq:prox_opt_app2_SRPMD}
\end{align}
\end{lemma}

\begin{proof}
The proof is the same as that for Lemma~\ref{lemma:prox_optimality_app} except that
Lemma~\ref{prop:AGD_convergence} will be applied to problem~\eqnok{eq:SRPMD_step_app} (with $\mu_{ \Phi} = 1 + \tau_k \eta_k$ 
and $L_\phi = L$).
\end{proof}

\begin{lemma} \label{prop:SRPMD_generic_app}
For any $s \in \cS$, we have
\begin{align}
V_{\tau_k}^{\pi_{k+1}}(s) - V_{\tau_k}^{\pi_k}(s) &\le
\langle \cQ_{\tau_k}^{\pi_k,\xi_k}(s, \cdot), \pi_{k+1}(\cdot | s) - \pi_k(\cdot|s) \rangle + h^{\pi_{k+1}}(s) - h^{\pi_{k}}(s) \nn\\
&\quad + \tau_k(D_{\pi_0}^{\pi_{k+1}}(s) - D_{\pi_k}^\pi(s))\nn\\
& \quad   +
\tfrac{1}{\eta_k} D_{\pi_k}^{\pi_{k+1}}(s)+ \tfrac{\eta_k \|\delta_k\|_\infty^2}{2(1-\gamma)} 
+ \tfrac{\gamma}{(1-\gamma)\eta_k} (\varepsilon_k + \tfrac{\varepsilon_{k-1}}{1+ \tau_{k-1} \eta_{k-1}})\log |\cA|\nn\\
&\quad -\tfrac{1}{1-\gamma} \bbe_{s' \sim d_s^{\pi_{k+1}}}[ \langle \delta_k, v_{k}(\cdot | s') - \pi_k(\cdot|s') \rangle].
\label{eq:srpmd_decrease2_inexact}
\end{align}
\end{lemma}

\begin{proof}
The proof is similar to that for Lemma~\ref{prop:SPMD_generic_app} with the following two exceptions:
(a) we will apply Lemma~\ref{lemma_per_diff} (i.e., the performance difference lemma)
to the perturbed value functions  $V_{\tau_k}^\pi$ instead of $V^\pi$ to obtain a result similar to \eqnok{eq:improvement_in_one_step_temp_stoch_inexact};
and (b) we will use use \eqnok{eq:prox_opt_app2_SRPMD} in place of \eqnok{eq:prox_opt_app2}
to derive a bound similar to \eqnok{eq:improvement_in_one_step_temp_opt_stoch_inexact}.
\end{proof}

\begin{lemma} \label{prop:SRPMD_un_reg_inexact_semi_final}
Suppose that $1+\eta_k \tau_k = 1/\gamma$ and $\varepsilon_k \le \varepsilon_{k-1}$ in the SAPMD method. Then for any
$k \ge 0$, we have
\begin{align}
& \bbe_{s\sim \nu^*,\xi_{\lceil k \rceil}}[V_{\tau_{k+1}}^{\pi_{k+1}}(s) - V_{\tau_{k+1}}^{\pi^*}(s)  + \tfrac{\tau_k}{1 - \gamma }  D_{\pi_{k+1}}^{\pi^*}(s)]\nn\\
&\le   \bbe_{s\sim \nu^*,\xi_{\lceil k-1 \rceil}}[ \gamma[ V_{\tau_k}^{\pi_k}(s) - V_{\tau_k}^{\pi^*}(s) +
\tfrac{\tau_k}{1 - \gamma}D_{\pi_k}^{\pi^*}(s)] \nn\\
&\quad + \tfrac{(\tau_{k}-\tau_{k+1})}{1-\gamma}  \log|\cA| +\tfrac{2(2-\gamma)\varsigma_k}{1-\gamma} + \tfrac{\sigma_k^2}{2 \gamma \tau_k}\nn\\
& \quad + \tfrac{\gamma^2(1+\gamma) \varepsilon_{k-1} \tau_k }{(1-\gamma)^2}\log |\cA|.  
\end{align}
\end{lemma}

\begin{proof}
By \eqnok{eq:prox_opt_app1_SRPMD} (with $p = \pi^*$), we have
\begin{align*}
& \langle Q_{\tau_k}^{\pi_{k},\xi_k}(s, \cdot), 
\pi_{k+1}(\cdot|s) - \pi^*(\cdot|s) \rangle + h^{\pi_{k+1}}(s) - h^{\pi^*}(s)\\
&+\tau_k [ D_{\pi_0}^{\pi_{k+1}}(s_t) - D_{\pi_0}^{\pi^{*}}(s_t)] 
+ \tfrac{1}{\eta_k} D_{\pi_k}^{\pi_{k+1}}(s) \\
&\le \tfrac{1}{\eta_k} D_{v_k}^{\pi^*}(s) - (\tfrac{1}{\eta_k}+  \tau_k ) D_{v_{k+1}}^{\pi^*}(s) + \tfrac{\varepsilon_k}{\eta_k} \log |\cA|,
\end{align*}
which, in view of \eqnok{eq:srpmd_decrease2_inexact}, implies that
\begin{align*}
&\langle Q_{\tau_k}^{\pi_{k},\xi_k}(s, \cdot) , 
\pi_{k}(\cdot|s) - \pi^*(\cdot|s) \rangle + h^{\pi_{k}}(s) - h^{\pi^*}(s) +  \tau_k [ D_{\pi_0}^{\pi_{k}}(s_t) - D_{\pi_0}^{\pi^{*}}(s_t)] \\
&+  V_{\tau_k}^{\pi_{k+1}}(s) - V_{\tau_k}^{\pi_k}(s) \\
&\le \tfrac{1}{\eta_k} D_{v_k}^{\pi^*}(s) - (\tfrac{1}{\eta_k}+  \tau_k ) D_{v_{k+1}}^{\pi^*}(s)+  \tfrac{\eta_k \|\delta_k\|_\infty^2}{2(1-\gamma)}
+ \tfrac{\gamma}{(1-\gamma)\eta_k} (\varepsilon_k + \tfrac{\varepsilon_{k-1}}{1+ \tau_{k-1} \eta_{k-1}})\log |\cA|\nn\\
&\quad -\tfrac{1}{1-\gamma} \bbe_{s' \sim d_s^{\pi_{k+1}}}[ \langle \delta_k, v_{k}(\cdot | s') - \pi_k(\cdot|s') \rangle].
\end{align*}
Taking expectation w.r.t. $\xi_{\lceil k \rceil}$ and $\nu^*$ on both sides of the above inequality,
and using Lemma~\ref{prop_strong_monotone} (with $h^\pi$ replaced by $h^\pi + \tau_kD_{\pi_0}^{\pi}(s_t) $
and $Q^\pi$ replaced by $Q_\tau^\pi$) and the relation in \eqnok{eq:bnd_bias_SPMD}, 
we arrive at
\begin{align*}
&\bbe_{s\sim \nu^*, \xi_{\lceil k \rceil}}[(1- \gamma)(V_{\tau_k}^{\pi_k}(s) - V_{\tau_k}^{\pi^*}(s) )]
+  \bbe_{s\sim \nu^*, \xi_{\lceil k \rceil}}[V_{\tau_k}^{\pi_{k+1}}(s) - V_{\tau_k}^{\pi_k}(s) ] \\
&
\le \bbe_{s\sim \nu^*, \xi_{\lceil k \rceil}}[\tfrac{1}{\eta_k}D_{\pi_k}^{\pi^*}(s) - (\tfrac{1}{\eta_k}+  \tau_k)D_{\pi_{k+1}}^{\pi^*}(s)]
+ 2 \varsigma_k+ \tfrac{\eta_k \sigma_k^2}{2(1-\gamma)} \\
&\quad + \tfrac{\gamma}{(1-\gamma)\eta_k} (\varepsilon_k + \tfrac{\varepsilon_{k-1}}{1+ \tau_{k-1} \eta_{k-1}})\log |\cA| +\tfrac{2 \varsigma_k}{1-\gamma}  .
\end{align*}
Noting $V_{\tau_k}^{\pi_{k+1}}(s) -  V_{\tau_k}^{\pi_k}(s)= 
V_{\tau_k}^{\pi_{k+1}}(s) - V_{\tau_k}^{\pi^*}(s) - [V_{\tau_k}^{\pi_k}(s) - V_{\tau_k}^{\pi^*}(s) ] $
and rearranging the terms in the above inequality, we have
\begin{align}
& \bbe_{s\sim \nu^*, \xi_{\lceil k \rceil}}[V_{\tau_k}^{\pi_{k+1}}(s) - V_{\tau_k}^{\pi^*}(s) + 
(\tfrac{1}{\eta_k}+ \tau_k)D_{v_{k+1}}^{\pi^*}(s)]  \nn\\
&\le \gamma \bbe_{s\sim \nu^*, \xi_{\lceil k-1 \rceil}}[V_{\tau_k}^{\pi_k}(s) - V_{\tau_k}^{\pi^*}(s)] 
+ \bbe_{s\sim \nu^*, \xi_{\lceil k-1 \rceil}}[\tfrac{1}{\eta_k}D_{v_k}^{\pi^*}(s) ] + \tfrac{2(2-\gamma) \zeta_k}{1-\gamma}+ \tfrac{\eta_k \sigma_k^2}{2(1-\gamma)}\nn\\
& \quad +  \tfrac{\gamma}{(1-\gamma)\eta_k} (\varepsilon_k + \tfrac{\varepsilon_{k-1}}{1+ \tau_{k-1} \eta_{k-1}})\log |\cA|.
\end{align}
The result then follows from $1+\eta_k \tau_k = 1/\gamma$, the assumptions $\tau_k \ge \tau_{k+1}$, $\varepsilon_k \le \varepsilon_{k-1}$ and \eqnok{eq:closeness_perturbation_direct}.
\end{proof}

\begin{theorem} \label{theorem_srpmd_un_reg_optimal_inexact}
Suppose that $\eta_k = \tfrac{1-\gamma}{\gamma \tau_k}$
in the SAPMD method. If $
\tau_k = \tfrac{1}{\sqrt{ \gamma \log |\cA|}} 2^{-(\lfloor k/l \rfloor + 1)}$,
$\varsigma_k = 2^{-(\lfloor k/l \rfloor + 2)}$,  
$\sigma_k^2 = 4^{-(\lfloor k/l \rfloor +2)}$, and
$\varepsilon_k =\tfrac {(1-\gamma)^2} {2 \gamma^2 (1+\gamma)}$  with
$l :=\left \lceil  \log_\gamma (1/4) \right \rceil
$,
then 
\begin{align}
& \bbe_{\xi_{\lceil k-1 \rceil}} [f(\pi_{k}) - f(\pi^*)] \nn \\
&\le 2^{-\lfloor k/l \rfloor} [f(\pi_0) - f(\pi^*) 
+ \tfrac{1}{1-\gamma}(\tfrac{3\sqrt{\log|\cA|}}{\sqrt{\gamma}} +\tfrac{5(2-\gamma)}{2} + \tfrac{5\sqrt{\log|\cA|}}{4\sqrt{\gamma}}) ].  
\end{align}
\end{theorem}

\begin{proof}
The result follows as an immediate consequence of Lemma~\ref{prop:SRPMD_un_reg_inexact_semi_final}, Lemma~\ref{lemma:induction_tech_sequence},
and an argument similar to the one to prove Theorem~\ref{theorem_srpmd_un_reg_optimal}.
\end{proof}

\vgap

In view of Theorem~\ref{theorem_srpmd_un_reg_optimal_inexact},
the inexact solution of the subproblem barely affect the iteration and sampling complexities of
the SAPMD method as long as $\varepsilon_k \le \tfrac {(1-\gamma)^2} {2 \gamma^2 (1+\gamma)}$.
Notice that an $\epsilon$-solution of problem~\eqnok{eq:MDP_OPT},
 i.e., a solution $\bar \pi$ s.t.
$\bbe [f(\bar \pi) - f(\pi^*)] \le \epsilon$,
can be found at the $\bar k$-th iteration with
\[
\lfloor \bar k/l \rfloor \le \log_2 \{\epsilon^{-1} [f(\pi_0) - f(\pi^*) 
+ \tfrac{5}{1-\gamma}(\tfrac{\sqrt{\log|\cA|}}{\sqrt{\gamma}} +1)]\}.
\]
Also observe that the condition number of the subproblem is given by 
\[
\tfrac{L \eta_k}{\tau_k \eta_k+1} = \tfrac{L(1-\gamma)}{\tau_k} = (1-\gamma) L \sqrt{\gamma \log |\cA|} 2^{\lfloor k/l \rfloor+1}.
\]
Combining these observations with Lemma~\ref{prop:AGD_convergence}, we conclude that the total number of gradient computations
of $h$ can be bounded by
\begin{align*}
l \tsum_{p=0}^{\lfloor \bar k/l \rfloor } \sqrt{\tfrac{L \eta_k}{\tau_k \eta_k+1} } \log (L\eta_k/\varepsilon_k)
&=l \tsum_{p=0}^{\lfloor \bar k/l \rfloor } \sqrt{\tfrac{L \eta_k}{\tau_k \eta_k+1} } \log \tfrac{L (1-\gamma)^3}{2\gamma^2 (1+\gamma)\tau_k}\\
&={\cal O}\{ l \lfloor \bar k/l \rfloor [(1-\gamma) L]^{1/2} 2^{\lfloor \bar k/l \rfloor/2}  \log \tfrac{L(1-\gamma)}{\gamma}\}\\
&= {\cal O}\left\{  (\log_\gamma\epsilon) \, \sqrt{\tfrac{L}{\epsilon} }   (\log \tfrac{L(1-\gamma) }{\gamma})\right\}.
\end{align*}

\section{Concluding Remarks} \label{sec_conclusion}
In this paper, we present the policy mirror descent (PMD) method and show that it can
achieve the linear and sublinear rate of convergence for RL problems with 
strongly convex or general convex regularizers, respectively. We then present
a more general form of the PMD method, referred to as
the approximate policy mirror descent (APMD) method, obtained by adding adaptive
perturbations to the action-value functions and show that it can achieve
the linear convergence rate for RL problems with general convex regularizers.
We develop the stochastic PMD and APMD methods
and derive general conditions on the bias and overall expected error
to guarantee the convergence of these methods. Using these conditions,
we establish new sampling complexity bounds of RL problems
by using two different sampling schemes, i.e., 
either using a straightforward generative model or a more involved
conditional temporal different method. The latter setting requires
us to establish a bound on the bias for estimating action-value functions,
which might be of independent interest. Finally, 
we establish the conditions on the accuracy required for the prox-mapping subproblems in these PMD type methods, as well as
the overall complexity of computing the gradients of the regularizers.
In the future, it will be interesting to study how to incorporate exploration into policy mirror descent to handle rarely
visited states and actions. Moreover, since this paper focuses on the
theoretical studies, it will be also rewarding to derive simplified PMD algorithms and
conduct numerical experiments 
to demonstrate possible advantages of the proposed algorithms.

\vgap

\noindent {\bf Acknowledgement:}
The author appreciates very much Caleb Ju, Sajad Khodaddadian, Tianjiao Li, Yan Li and
two anonymous reviewers for their  careful reading and a few suggested corrections for earlier versions of this paper.

\renewcommand\refname{Reference}
\bibliographystyle{plain}
\bibliography{GeorgeLan}

\section*{Appendix A: Concentration Bounds for $l_\infty$-bounded Noise}

We first show how to bound the expectation of the maximum for a finite number of sub-exponential variables. 

\begin{lemma}\label{lemma_sup_bound_exp}
Let $\norm{X}_{\psi_1}:= \inf \{t > 0: \exp(|X|/t) \le \exp(2) \}$ denote the sub-exponential norm of $X$.
For a given sequence of sub-exponential variables $\{X_i\}_{i=1}^n$ with $\bbe[X_i] \leq v$ and $\norm{X_i}_{\psi_1} \leq \sigma$, we have 
\begin{align*}
\bbe [\max_i X_i]  \leq C \sigma (\log n + 1) + v,
\end{align*}
where $C$ denotes an absolute constant.
\end{lemma}

\begin{proof}
By the property of sub-exponential random variables (Section 2.7 of \cite{vershynin2018high}), we know that $Y_i = X_i - \EE \sbr{X_i}$ is also sub-exponential with
$\norm{Y_i}_{\psi_1} \leq C_1 \norm{X_i}_{\psi_1} \le C_1 \sigma$ for some absolute constant $C_1 > 0$.
Hence by Proposition 2.7.1 of \cite{vershynin2018high}, there exists an absolute constant $C > 0$ such that
$
\EE [\exp (\lambda Y_i)] \leq \exp (C^2 \sigma^2 \lambda^2) , ~ \forall \abs{\lambda} \leq 1/( C \sigma ).
$
Using the previous observation,  we have
\begin{align*}
 \exp ( \EE[\lambda \max_{i} Y_i]) \leq \EE[ \exp (\lambda \max_i Y_i)] \leq \EE [\tsum_{i=1}^n \exp(\lambda Y_i) ]
 \leq n \exp(C^2 \sigma^2 \lambda^2), ~ \forall \abs{\lambda} \leq \frac{1}{C \sigma},
\end{align*}
which implies
$
\EE [\max_i Y_i] \leq \log n / \lambda + C^2 \sigma^2 \lambda, ~ \forall \abs{\lambda} \leq 1/(C \sigma).
$
Choosing $\lambda = 1/(C \sigma)$, we obtain  
$
\EE \sbr{\max_i Y_i}  \leq C \sigma (\log n + 1 ) .
$
By combining this relation with the definition of $Y_i$,  we conclude that 
$
\EE [\max_i X_i]  \leq \EE [\max_i Y_i ]+ v \leq C \sigma (\log n + 1) + v.
$

\end{proof}


\begin{proposition} \label{prop:bound_inf_var}
For $\delta^{k} := Q^{\pi_k, \xi_k} - Q^{\pi_k} \in \RR^{\abs{\cS} \times \abs{\cA}}$, we have 
\begin{align*}
\EE_{\xi_k}[ \|\delta^k\|_{\infty}^2 ] \leq  \tfrac{(\overline{c} + \overline{h})^2}{(1-\gamma)^2}
\left[ \gamma^{2T_k} + \tfrac{\kappa}{M_k} (\log (\abs{\cS} \abs{\cA}) + 1)\right], 
\end{align*}
where $\kappa>0$ denotes an absolute constant. 
\end{proposition}

\begin{proof}
To proceed, we denote $\delta^{k}_{s,a} := Q^{\pi_k, \xi_k}(s,a)  - Q^{\pi_k}(s,a) $, and hence 
\[\EE_{\xi_k} \|Q^{\pi_k, \xi_k} - Q^{\pi_k}\|_{\infty}^2 = 
\EE_{\xi_k} [\max_{s \in \cS, a \in \cA} (\delta^k_{s,a})^2].
\]
Note that by definition, for each $(s,a)$ pair, we have $M_k$ independent trajectories of length $T_k$ starting from $(s,a)$.
Let us denote $Z_i :=  \sum_{t = 0}^{T_k - 1} \gamma^t \sbr{c(s_t^i, a_t^i) + h^{\pi_k}(s_t^i) }$, $i = 1, \ldots, M_k$. Hence, 
\begin{align}
Q^{\pi_k, \xi_k} (s,a) &= \frac{1}{M_k} \tsum_{i=1}^{M_k} \tsum_{t = 0}^{T_k - 1} \gamma^t \sbr{c(s_t^i, a_t^i) + h^{\pi_k}(s_t^i) } = \tfrac{1}{M_k} \sum_{i=1}^{M_k} Z_i,  \nonumber \\
 \delta^{k}_{s,a} & = \frac{1}{M_k} \tsum_{i=1}^{M_k} (Z_i  - Q^{\pi_k}(s,a)), ~~ Z_i  - Q^{\pi_k}(s,a) \in [-\tfrac{\overline{c} + \overline{h}}{1-\gamma},  \tfrac{\overline{c} + \overline{h}}{1-\gamma}].\nn 
\end{align}
Since each $Z_i  - Q^{\pi_k}(s,a)$ is independent of each other, 
 it is immediate to see that 
$Y_{s,a} := (\delta^k_{s,a})^2$ is a sub-exponential with $\norm{Y_{s,a}}_{\psi_1} 
\leq \frac{(\overline{c} + \overline{h})^2}{(1-\gamma)^2 M_k}$.
Also note that 
$$\EE_{\xi_k} [Y_{s,a}] = \EE_{\xi_k} [(\delta^k_{s,a})^2]  = \mathrm{Var}(\delta^k_{s,a}) + (\EE \delta^k_{s,a})^2
\leq \tfrac{(\overline{c} + \overline{h})^2}{(1-\gamma)^2 M_k} + \tfrac{(\overline{c} + \overline{h})^2}{(1-\gamma)^2} \gamma^{2T_k}.$$
Thus in view of Lemma \ref{lemma_sup_bound_exp}, with $\sigma = \tfrac{ (\overline{c} + \overline{h})^2}{(1-\gamma)^2 M_k}$, 
and $v = \tfrac{(\overline{c} + \overline{h})^2}{(1-\gamma)^2 M_k} + \frac{(\overline{c} + \overline{h})^2}{(1-\gamma)^2} \gamma^{2T_k}$, we conclude that 
\begin{align*}
\EE[ \|\delta^k\|_{\infty}^2] &= \EE[ \max_{s\in \cS, a\in \cA} (\delta^k_{s,a})^2] 
 = 
\EE [ \max_{s\in \cS, a\in \cA} Y_{s,a}] \\
& \leq \tfrac{C (\overline{c} + \overline{h})^2}{(1-\gamma)^2 M_k} (\log (\abs{\cS} \abs{\cA}) + 1) +  \tfrac{(\overline{c} + \overline{h})^2}{(1-\gamma)^2 M_k}  +  \tfrac{(\overline{c} + \overline{h})^2}{(1-\gamma)^2} \gamma^{2T_k}.
\end{align*}

\end{proof}

\section*{Appendix B: Bias for Conditional Temporal Difference Methods}

\noindent {\bf Proof of Lemma~\ref{lemma:bound_on_CTD_bias}.} 
\begin{proof}
For simplicity, let us denote $\bar \theta_t \equiv \bbe[\theta_t]$,
$\zeta_t \equiv (\zeta_t^1, \ldots, \zeta_t^\alpha)$ and
$\zeta_{\lceil t\rceil} = (\zeta_1, \ldots, \zeta_t)$. Also let us denote
$\delta^F_t := F^\pi(\theta_t) - \bbe[\tilde{F}^\pi(\theta_t,\zeta_t^\alpha)|\zeta_{\lceil t-1\rceil}]$
and $\bar \delta^F_t = \bbe_{\zeta_{\lceil t-1\rceil}}[\delta^F_t]$.
It follows from Jensen's ienquality and Lemma~\ref{step_size_FastTD_skipping} that
\beq \label{eq:bnd_CTD_avg_dist}
\|\bar \theta_t - \theta^*\|_2 = \|\bbe_{\zeta_{\lceil t-1\rceil}}[ \theta_t] - \theta^*\|_2 
\le \bbe_{\zeta_{\lceil t-1\rceil}}[\|\theta_t -  \theta^*\|_2] \le R.
\eeq
Also by Jensen's inequality, Lemma~\ref{2_2_markovian} and Lemma~\ref{step_size_FastTD_skipping}, we have
\begin{align}
\|\bar \delta^F_t\|_2 &= \|\bbe_{\zeta_{\lceil t-1\rceil}}[ \delta^F_t]\|_2
\le \bbe_{\zeta_{\lceil t-1\rceil}}[\|\delta^F_t\|_2]\nn\\
&\le C \rho^\alpha \bbe_{\zeta_{\lceil t-1\rceil}}[\|\theta_t -\theta^*\|_2]\le C R\rho^\alpha.  \label{eq:bnd_CTD_avg_operator}
\end{align}
Notice that
\begin{align*}
\theta_{t+1} &= \theta_t -  \beta_t  \tilde{F}^\pi(\theta_t,\zeta_t^\alpha)\\
&= \theta_t - \beta_t F^\pi(\theta_t) + \beta_t [F^\pi(\theta_t) - \tilde{F}^\pi(\theta_t,\zeta_t^\alpha)].
\end{align*}
Now conditional on $\zeta_{\lceil t-1\rceil}$, taking expectation w.r.t.  $\zeta_t$ on \eqnok{eq:CTD_skipping_step},
we have
$
\bbe[\theta_{t+1}|\zeta_{\lceil t-1\rceil}]  
= \theta_t - \beta_t F^\pi(\theta_t) + \beta_t \delta^F_t.
$
Taking further expectation w.r.t. $\zeta_{\lceil t-1\rceil}$ and using the linearity of $F$, we have
$
\bar \theta_{t+1} = \bar \theta_t - \beta_t F^\pi(\bar \theta_t) + \beta_t \bar \delta^F_t,
$
which implies
\begin{align*}
\| \bar \theta_{t+1} - \theta^*\|_2^2 &= \|\bar \theta_t - \theta^* - \beta_t F^\pi(\bar \theta_t) + \beta_t \bar \delta^F_t\|_2^2\\
&= \|\bar \theta_t - \theta^*\|_2^2 - 2 \beta_t \langle F^\pi(\bar \theta_t) -  \bar \delta^F_t, \bar \theta_t - \theta^*\rangle
+ \beta_t^2 \|F^\pi(\bar \theta_t) -  \bar \delta^F_t\|_2^2\\
&\le \|\bar \theta_t - \theta^*\|_2^2 - 2 \beta_t \langle F^\pi(\bar \theta_t) -  \bar \delta^F_t, \bar \theta_t - \theta^*\rangle
+ 2 \beta_t^2 [ \|F^\pi(\bar \theta_t)\|_2^2 + \| \bar \delta^F_t\|_2^2].
\end{align*}
The above inequality, together with \eqnok{eq:bnd_CTD_avg_dist}, \eqnok{eq:bnd_CTD_avg_operator}
and the facts that
\begin{align*}
\langle F^\pi(\bar \theta_t), \bar \theta_t - \theta^*\rangle
= \langle F^\pi(\bar \theta_t) - F^\pi(\theta^*), \bar \theta_t - \theta^*\rangle \ge \Lambda_{\min} \|\bar \theta_t - \theta^*\|_2^2\\
\|F^\pi(\bar \theta_t)\|_2 = \|F^\pi(\bar \theta_t) - F^\pi(\theta^*)\|_2 \le \Lambda_{\max}\|\bar \theta_t - \theta^*\|_2,
\end{align*}
then imply that
\begin{align}
\| \bar \theta_{t+1} - \theta^*\|_2^2 &\le (1 - 2 \beta_t \Lambda_{\min} + 2 \beta_t^2  \Lambda_{\max}^2)
\|\bar \theta_t - \theta^*\|_2^2 + 2\beta_t C R^2 \rho^\alpha + 2 \beta_t^2 C^2 R^2\rho^{2\alpha} \nn\\
&\le (1 - \tfrac{3}{t+ t_0 -1}) \|\bar \theta_t - \theta^*\|_2^2 + 2\beta_t C R^2 \rho^\alpha + 2 \beta_t^2 C^2 R^2\rho^{2\alpha}, \label{eq:CTD_recursion}
\end{align}
where the last inequality follows from
\begin{align*}
2(\beta_t \Lambda_{\min} - \beta_t^2  \Lambda_{\max}^2) &= 2 \beta_t ( \Lambda_{\min} - \beta_t\Lambda_{\max}^2 )
= 2 \beta_t (\Lambda_{\min} - \tfrac{2 \Lambda_{\max}^2}{\Lambda_{\min} (t+ t_0 -1)})\\
&\ge 2 \beta_t (\Lambda_{\min} - \tfrac{2 \Lambda_{\max}^2}{\Lambda_{\min}  t_0 }) \ge \tfrac{3}{2} \beta_t \Lambda_{\min} = \tfrac{3}{t+ t_0 -1}
\end{align*}
due to the selection of $\beta_t$ in \eqnok{eq:def_tau_beta_CTD}.
Now let us denote
$
\Gamma_t :=
\begin{cases}
1 & t =0,\\
(1 - \tfrac{3}{t+ t_0 -1})\Gamma_{t-1} & t \ge 1,
\end{cases}
$
or equivalently, $\Gamma_t := \tfrac{(t_0 - 1) (t_0 -2) (t_0 -3)}{(t+t_0-1) (t + t_0 -2) (t+ t_0 -3))}$.
Dividing both sides of \eqnok{eq:CTD_recursion} by $\Gamma_t$ and taking the telescopic sum, we have
\begin{align*}
\tfrac{1}{\Gamma_t} \| \bar \theta_{t+1} - \theta^*\|_2^2
&\le  \|\bar \theta_1 - \theta^*\|_2^2 + 2 C R^2 \rho^\alpha \tsum_{i=1}^t \tfrac{\beta_i}{\Gamma_i} + 2  C^2 R^2\rho^{2\alpha} \tsum_{i=1}^t \tfrac{\beta_i^2}{\Gamma_i}.
\end{align*}
Noting that
\begin{align*}
\tsum_{i=1}^t \tfrac{\beta_i}{\Gamma_i} &= \tfrac{2}{\Lambda_{\min}} \tsum_{i=1}^t \tfrac{(i+t_0 - 2)(i+t_0-3)}{(t_0-1)(t_0-2)(t_0-3)}
\le \tfrac{2  \tsum_{i=1}^t (i+t_0 - 2)^2}{\Lambda_{\min}(t_0-1)(t_0-2)(t_0-3)}\\
&\le  \tfrac{2 (t+t_0-1)^3}{3\Lambda_{\min}(t_0-1)(t_0-2)(t_0-3)},\\
\tsum_{i=1}^t \tfrac{\beta_i^2}{\Gamma_i} &\le \tfrac{4 \tsum_{i=1}^t (i+t_0-3)}{\Lambda_{\min}^2(t_0-1)(t_0-2)(t_0-3)} 
\le \tfrac{2  (t+t_0-2)^2}{\Lambda_{\min}^2(t_0-1)(t_0-2)(t_0-3)},
\end{align*}
we conclude 
\begin{align*}
 \| \bar \theta_{t+1} - \theta^*\|_2^2
&\le \tfrac{(t_0 - 1) (t_0 -2) (t_0 -3)}{(t+t_0-1) (t + t_0 -2) (t+ t_0 -3)} \|\bar \theta_1 - \theta^*\|_2^2 + 
2 C R^2 \rho^\alpha \tfrac{2 (t+t_0-1)^2}{3\Lambda_{\min} (t + t_0 -2) (t+ t_0 -3)} \\
&\quad + 2  C^2 R^2\rho^{2\alpha}  \tfrac{2  (t+t_0-2)}{\Lambda_{\min}^2(t+t_0-1) (t+ t_0 -3)}\\
&\le \tfrac{(t_0 - 1) (t_0 -2) (t_0 -3)}{(t+t_0-1) (t + t_0 -2) (t+ t_0 -3)} \|\bar \theta_1 - \theta^*\|_2^2 + 
 \tfrac{8 C R^2 \rho^\alpha }{3\Lambda_{\min}} + \tfrac{C^2 R^2\rho^{2\alpha}} {\Lambda_{\min}^2},
\end{align*}
from which the result holds since $\bar \theta_1 = \theta_1$.
\end{proof}

\end{document}